\documentclass{article}

\usepackage[utf8]{inputenc} 
\usepackage[T1]{fontenc}    
\usepackage{url}            
\usepackage{booktabs}       
\usepackage{amsfonts}       
\usepackage{nicefrac}       
\usepackage{microtype}      
\usepackage{comment}
\usepackage{mathtools}
\usepackage{amsmath}
\usepackage{physics}
\usepackage{array,makecell}
\usepackage{booktabs}
\usepackage{tabularx}
\usepackage{bbm}
\usepackage[toc,page]{appendix}
\usepackage{setspace}
\usepackage{relsize}
\usepackage{xcolor}         
\usepackage{listings}
\usepackage{wrapfig}
\usepackage{algorithmicx}
\usepackage[ruled,vlined]{algorithm2e}
\usepackage{lastpage}
\usepackage{mathtools}  
\mathtoolsset{showonlyrefs} 
\usepackage[preprint]{jmlr2e}
\usepackage{caption}
\usepackage{xr}
\hypersetup{ hidelinks }

\usepackage{mathtools}
\DeclarePairedDelimiter\ceil{\lceil}{\rceil}
\DeclarePairedDelimiter\floor{\lfloor}{\rfloor}



\makeatletter
\newcommand*{\addFileDependency}[1]{
  \typeout{(#1)}
  \@addtofilelist{#1}
  \IfFileExists{#1}{}{\typeout{No file #1.}}
}
\makeatother

\begin{document}

\jmlrheading{23}{2022}{1-\pageref{LastPage}}{10/21; Revised
9/22}{12/22}{21-1280}{Shai Feldman, Stephen Bates, Yaniv Romano}
\ShortHeadings{Flexible Multiple-Output Quantile Regression}{Feldman, Bates, and Romano}

\title{Calibrated Multiple-Output Quantile Regression with Representation Learning} 

\author{\name Shai Feldman \email shai.feldman@cs.technion.ac.il \\
       \addr Department of Computer Science\\
       Technion---Israel Institute of Technology\\
       Technion City, Haifa 32000, Israel
       \AND
       \name Stephen Bates \email stephenbates@cs.berkeley.edu \\
       \addr Departments of Electrical Engineering and Computer Science and of Statistics\\
            University of California, Berkeley\\
            Berkeley, CA 94720, USA
        \AND
       \name Yaniv Romano \email yromano@technion.ac.il \\
       \addr Departments of Electrical and Computer Engineering and of Computer Science\\
       Technion---Israel Institute of Technology\\
       Technion City, Haifa 32000, Israel}
       
\editor{Samuel Kaski}

\maketitle

\begin{abstract}%
We develop a method to generate predictive regions that cover a multivariate response variable with a user-specified probability. Our work is composed of two components. First, we use a deep generative model to learn a representation of the response that has a unimodal distribution. Existing multiple-output quantile regression approaches are effective in such cases, so we apply them on the learned representation, and then transform the solution to the original space of the response. This process results in a flexible and informative region that can have an arbitrary shape, a property that existing methods lack. Second, we propose an extension of conformal prediction to the multivariate response setting that modifies any method to return sets with a pre-specified coverage level. The desired coverage is theoretically guaranteed in the finite-sample case for any distribution. Experiments conducted on both real and synthetic data show that our method constructs regions that are significantly smaller compared to existing techniques.

\end{abstract}







\section{Introduction}\label{sec:intro}

In real-world applications, it is often required to estimate more than one response variable. Consider, for example, estimating the effects and side effects of a drug given the patient's demographic information and medical measurements. These two responses may be correlated in the way that when the drug is effective the side effects are more severe \citep{schuell2005side}, and this relation might not be linear. In such high-stakes settings, giving point predictions for the drug’s effects and side effects is insufficient; the decision-maker must know the plausible effects for an individual. The plausible effects can be represented as a region in the multidimensional space that covers a pre-specified proportion (e.g., 90\%) of the drug's possible outcomes. In the one-dimensional case, the region reduces to an interval, determined by lower and upper bounds for the response variable. The problem of constructing such a prediction interval is extensively investigated in the literature \citep{QR, flexible_pred_bands, quantile_regression_forests, guan2019conformal, nested_conformal}. This approach can be na\"ively extended to the multivariate case by estimating a prediction interval for each response separately. However, this process will result in a rectangle-shaped region, whereas the shape of the true distribution of the response variables can be arbitrary, not a rectangle, and even not convex. 
In that case, the predicted region is likely to be over-conservative: it would not reflect the true underlying uncertainty. A better approach is to predict both outcome variables (drug effect and side effect) jointly. This strategy encourages the model to exclude unlikely combinations of the two from the predicted region. In this work, we show how to construct a region that reflects the true distribution of the response variables while attaining the pre-specified coverage level.

\subsection{Problem Formulation}

This paper studies the problem of constructing reliable uncertainty estimates in multivariate regression problems. Suppose we are given $n$ training samples $\{(X_i,Y_i)\}_{i=1}^n$, where $X\in\mathbb{R}^p$ is a feature vector, and $Y\in\mathbb{R}^d$ is a response vector. Given a new test point $X_{n+1}$, our goal is to construct a region of values in which the unknown test response $Y_{n+1}$ falls with high probability. Formally, we seek to build a \emph{marginal distribution-free quantile region} $\hat{R}(X_{n+1}) \subseteq \mathbb{R}^d$ that is likely to contain the response $Y_{n+1}$ with a user-specified coverage probability $1-\alpha$:
\begin{equation}\label{eq:cov_statement}
\mathbb{P}[Y_{n+1}\in \hat{R}(X_{n+1})]\geq 1-\alpha,
\end{equation}
for any joint distribution $P_{XY}$ and any sample size $n$. This property is called \emph{marginal coverage}, and in order to guarantee it, we assume that all samples $\{(X_i, Y_i)\}_{i=1}^{n+1}$ are drawn exchangeably. That is, we assume that the training samples and test samples follow the same distribution. In addition, we aim to construct quantile regions that are as small as possible, reliably estimating the conditional distribution of $Y \mid X$. When $d=1$, the quantile region reduces to a one-dimensional prediction interval, determined by lower and upper bounds, within which the response is expected to lie with probability at least $1-\alpha$.

One of the methods that addresses this problem is \emph{directional quantile regression} (\texttt{DQR}) \citep{kong2012quantile, paindaveine2011directional, bovcek2017directional}. The main idea is to estimate conditional quantiles in different directions, where each defines a half-space, and the quantile region is defined as the intersection of all half-spaces. This method is simple and fast, compared to competitive methods that require approximating the entire distribution $P_{XY}$ \citep{carlier2016vector, carlier2017vector, carlier2020vector}. However, being an intersection of half-spaces, the quantile region is convex, so it might be unnecessarily large, as demonstrated in Section \ref{sec:synthetic_example}. Additionally, the empirical coverage of such a quantile region is lower than the nominal one, which forces the user to estimate extremal conditional quantiles, as explained in Section \ref{sec:need_conf}. Furthermore, since the conditional distribution of $Y \mid X$ is unknown, it is difficult estimating what empirical coverage \texttt{DQR} will achieve given a certain nominal level. This raises the problem of choosing the correct nominal level for which \texttt{DQR} achieves the desired coverage rate. In sum, the \texttt{DQR} method can work effectively only in specific cases, i.e., when the distribution of $Y \mid X$ has level sets of the density that are convex. However, even in those cases, the user is required to estimate extremal quantiles, a process that is known to be impractical, as shown by \cite{extremal_qr}. Furthermore, as demonstrated in \cite[][Section 4.1]{liu2019fast}, \texttt{DQR} is not guaranteed to capture the zone with the highest density even for convex density level sets.

In this work we develop a novel scheme to construct statistically efficient quantile regions, relying on the \texttt{DQR} approach. The core idea is to learn a representation of the response variable for which the \texttt{DQR} method is effective. Once obtaining a quantile region for that representation, we transform it to the original representation of the response variable $Y$. In this work, we use a \emph{conditional variational auto-encoder} (CVAE) \citep{cvae} model to learn a mapping between these two representations, described in detail in Appendix~\ref{sec:cvae_and_kl_lambda}; see also \citep{thickstun2017learning, xu2021representation, bengio2013representation, oord2017neural, kobyzev2020normalizing} for other representation learning techniques. In striking contrast to the \texttt{DQR}, this scheme is non-parametric and can produce non-convex regions, which are therefore smaller and more informative. 
In Section \ref{sec:synthetic_example}, and in the experiments, provided in Section \ref{sec:experiments}, we show that among the four methods we examine in this work, our method consistently tracks the conditional distribution better, tested on six real data sets, and various synthetic examples. This phenomenon is supported by a theoretical guarantee of the quantile region covering only possible responses, i.e., responses in the support of $Y\mid X$.

Secondly, we extend ideas from conformal prediction to the multidimensional case and propose a calibration procedure that guarantees the coverage requirement~\eqref{eq:cov_statement} in the finite-sample case for any distribution. 
Conformal inference \citep{vovk2005algorithmic} is a framework commonly used in the one-dimensional case ($d=1$) \citep{CQR,flexible_pred_bands,sesia2020comparison,adaptive_intervals,dist_conformal_pred,nested_conformal,conformal_regions, guan2019conformal} that provides a generic methodology for building prediction intervals that provably attain valid marginal coverage~\eqref{eq:cov_statement}. See~\citep{angelopoulos-gentle} for a recent overview of this subject. 
The procedure we propose is generic and can be applied to any multiple-output quantile regression method, including those we discuss in this work. In Section \ref{sec:need_conf} we show that this procedure is vital to the \texttt{DQR} method, since its empirical coverage is significantly lower than the nominal level.  

\section{Background and Related Work}\label{sec:related_work}
In this section, we describe existing and related methods, but first, begin with a small synthetic experiment that demonstrates the challenges of constructing informative conditional quantile regions.

\subsection{A Synthetic Example}\label{sec:synthetic_example}
The example provided hereafter illustrates how the existing methods perform on this synthetic data, revealing their strengths and weaknesses, which will be further discussed in this section. We generate a v-shaped 2-dimensional response whose structure varies with the feature vector $X\in \mathbb{R}$. This data is visualized in Figure \ref{fig:syn_data}, presenting the marginal and conditional distribution of the data. Observe that as $X$ increases, the response is shifted downwards and the slope of the valley becomes steeper. For illustration purposes, we choose to work with a one-dimensional feature vector here, however, in section \ref{sec:experiments} we also examine data sets with responses and features of higher dimensions. The full description of the distribution of $Y \mid X$ is given in Appendix~\ref{sec:syn_datasets_details}.

Figure \ref{fig:syn_data_results} presents the conditional distribution of two test points: $x=1.5$ and $x=2.5$, and the quantile regions constructed for each of them. This figure shows that the regions constructed by our method reflect the true conditional distribution. This stands in striking contrast with the competitive techniques, which are described hereafter.

\newcommand\imagewidth{0.4}
\begin{figure}[htbp]
\setstretch{1.1}
  \centering
\scalebox{1}{
    \begin{tabular}{cc}
    {\centering{\includegraphics[width=\imagewidth\linewidth]{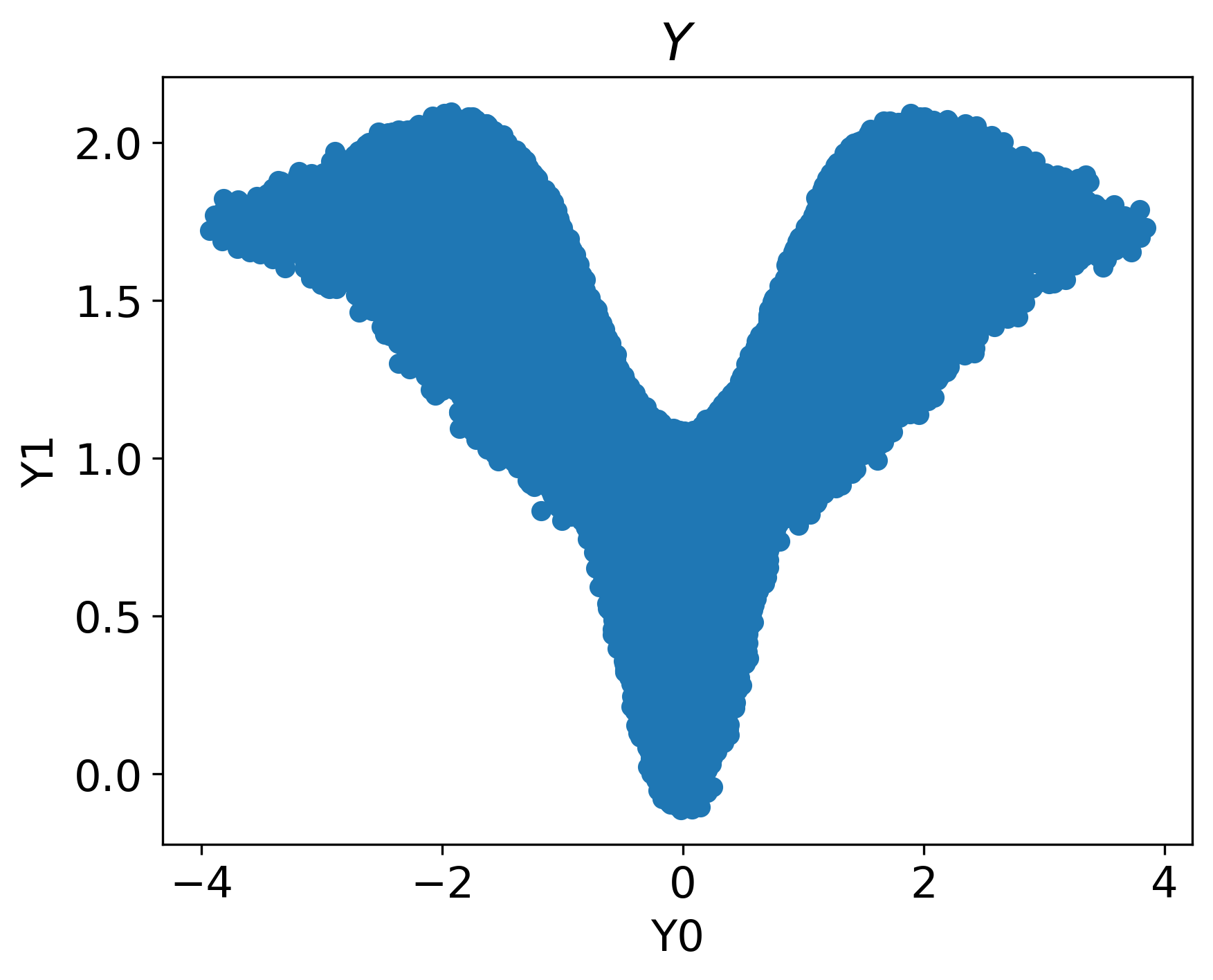}}} &
     {\centering{\includegraphics[width=\imagewidth\linewidth]{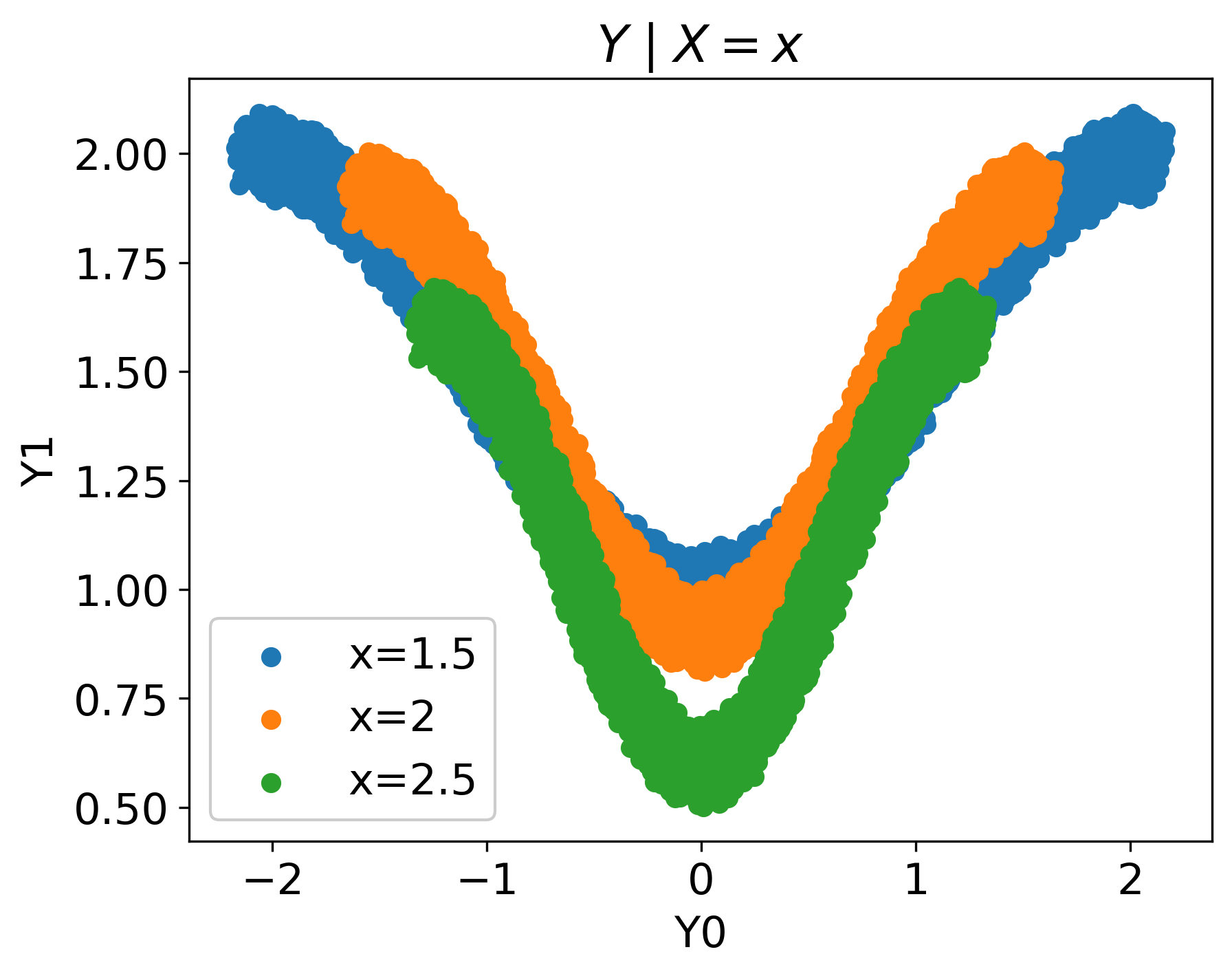}}}
    \end{tabular}%
    }
    \captionsetup{format=hang} \caption{Synthetic data visualization. Left: scatter plot of the marginal distribution of $Y$. Right: scatter plot of the conditional distribution of $Y \mid X=x$, for $x \in \{1.5, 2, 2.5\}$.}
\label{fig:syn_data}%
\end{figure}%

\renewcommand\imagewidth{0.38}
\newcommand\texthspace{2.15cm}

\newcommand\rowincludegraphics[2][]{\raisebox{-0.45\height}{\includegraphics[#1]{#2}}}
\begin{figure}[htbp]
\setstretch{1.1}
  \centering

    \scalebox{1.}{
    \begin{tabular}{ccc}
    \multicolumn{1}{c}{\textbf{Method}} & \multicolumn{2}{c}{\textbf{Quantile region}}\\ 
    
    {} & {\parbox{\imagewidth\linewidth} {\quad \hspace{\texthspace} $x=1.5$}}  & {\parbox{\imagewidth\linewidth} {\quad \hspace{\texthspace} $x=2.5$}}    \\ \\
    
    {\centering \texttt{Na\"ive QR}} & {\centering{\rowincludegraphics[width=\imagewidth\linewidth]{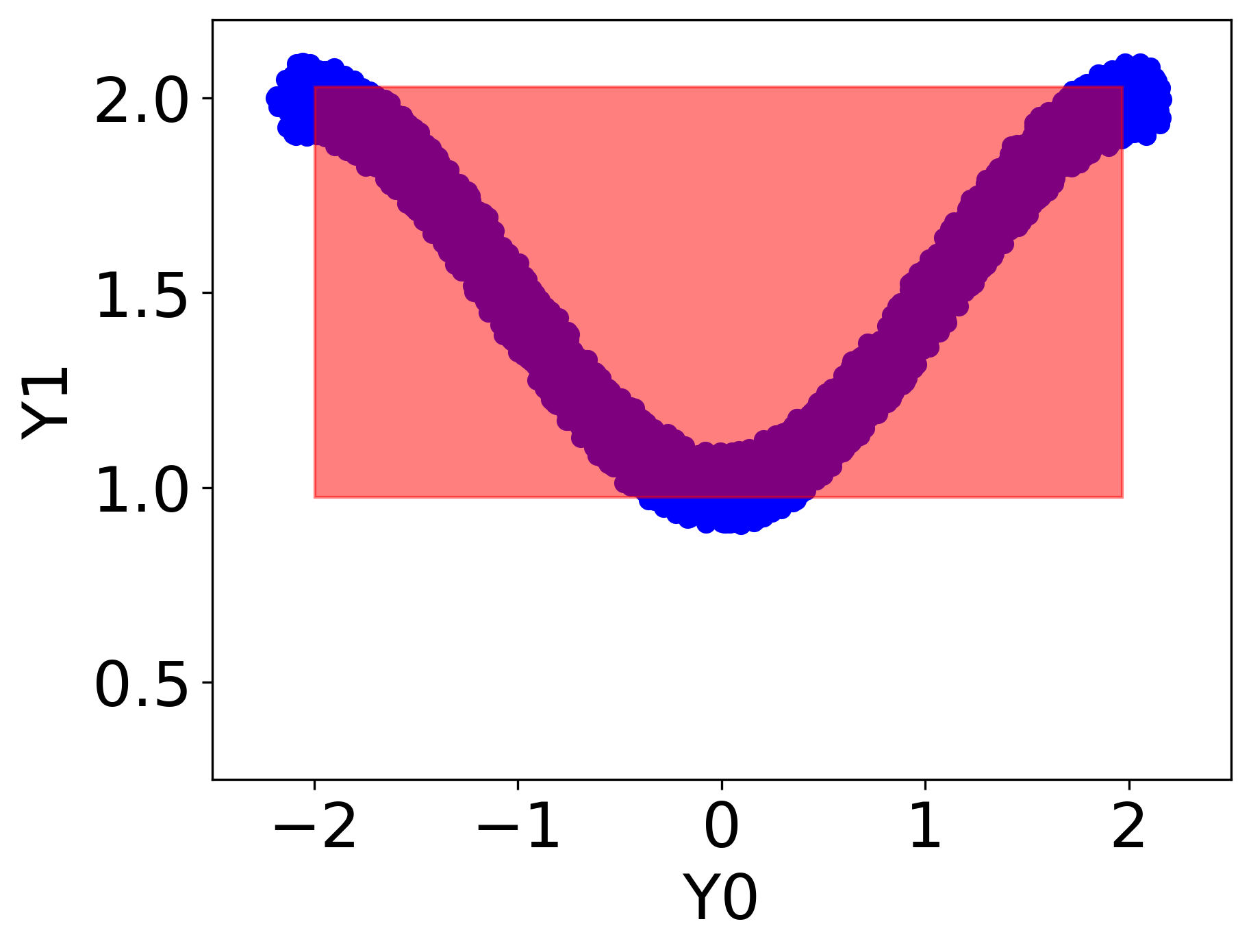}}} & {\centering{\rowincludegraphics[width=\imagewidth\linewidth]{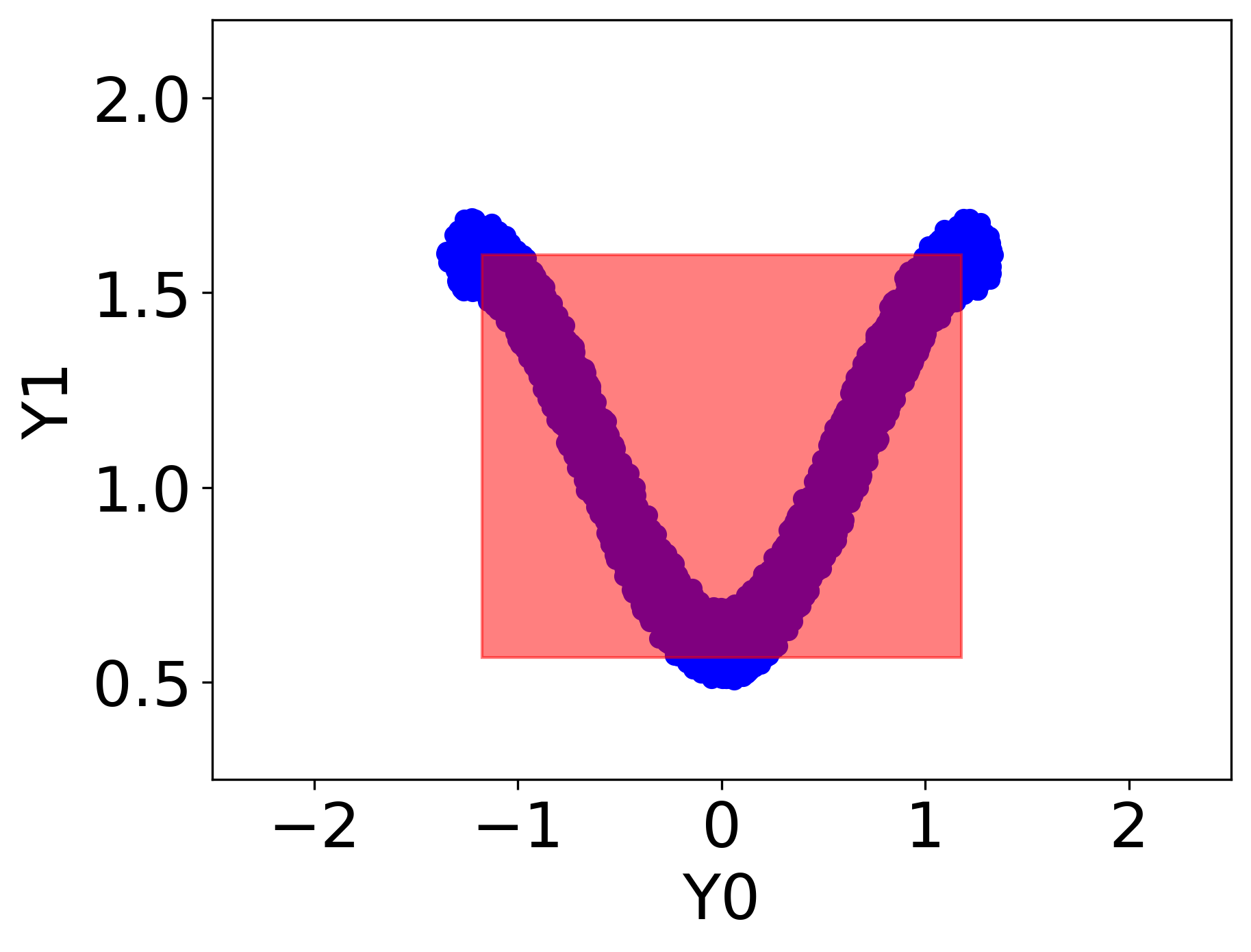}}} \\

    \texttt{NPDQR} & {\centering{\rowincludegraphics[width=\imagewidth\linewidth]{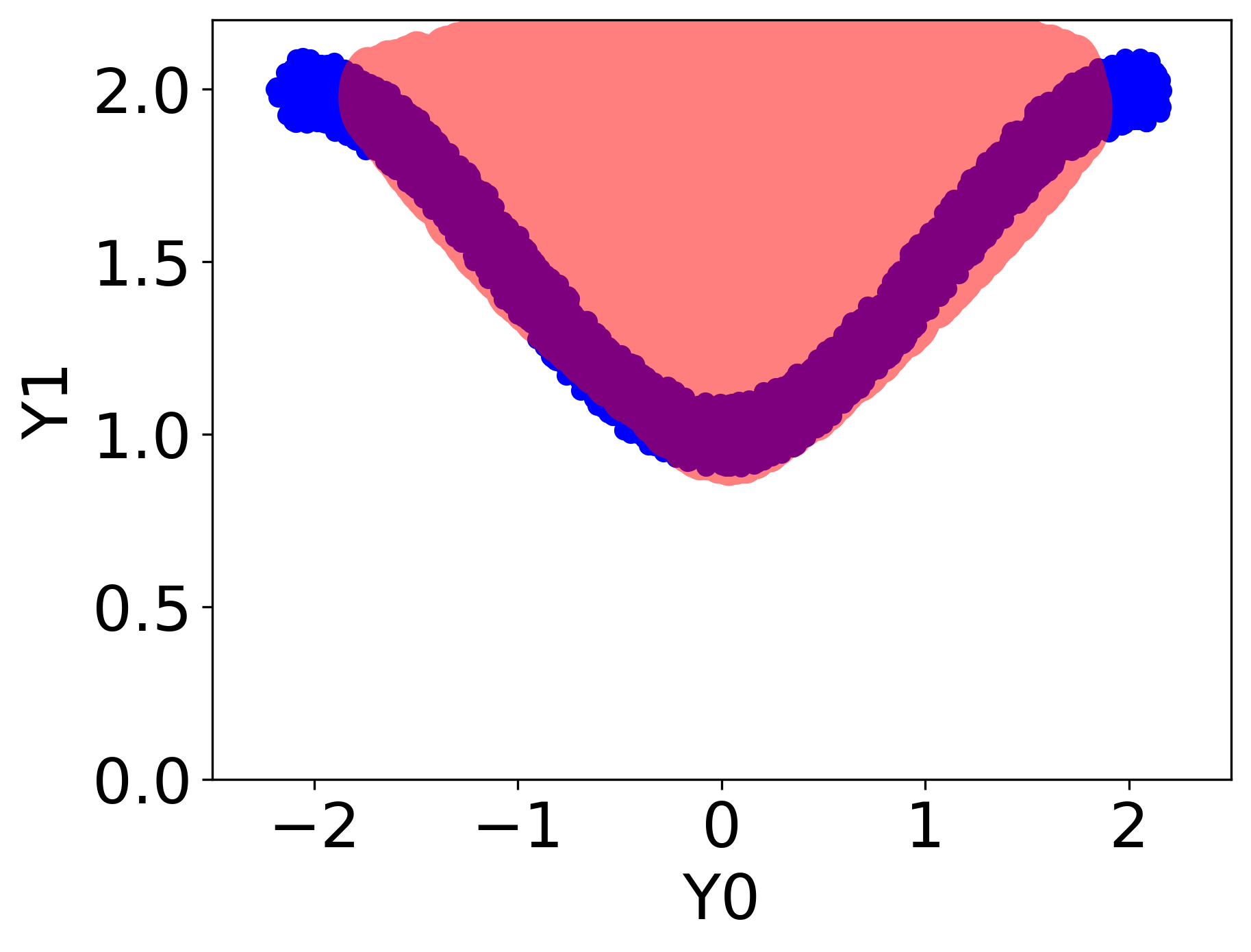}}} & {\centering{\rowincludegraphics[width=\imagewidth\linewidth]{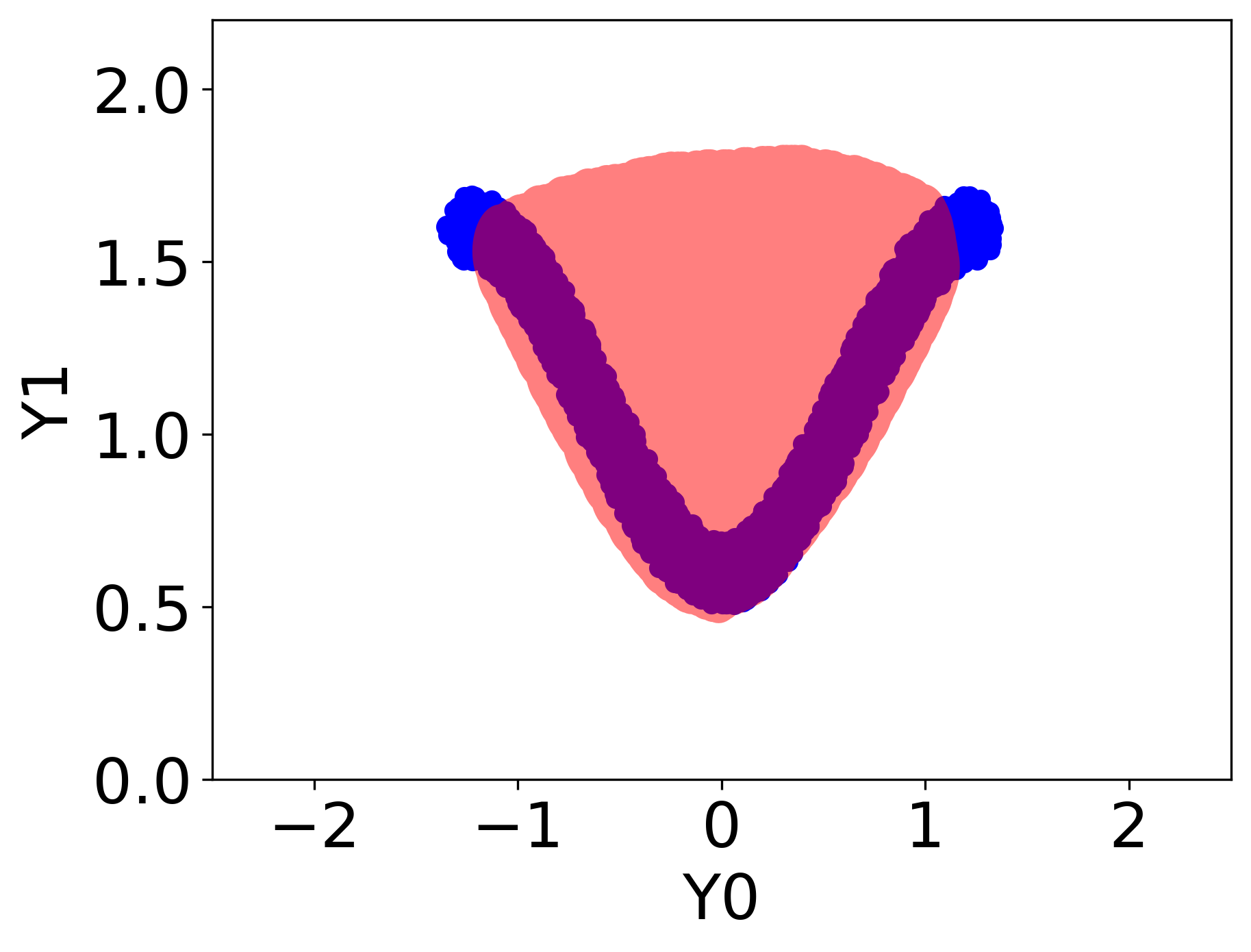}}} \\
    
    \texttt{VQR} &  {\centering{\rowincludegraphics[width=\imagewidth\linewidth]{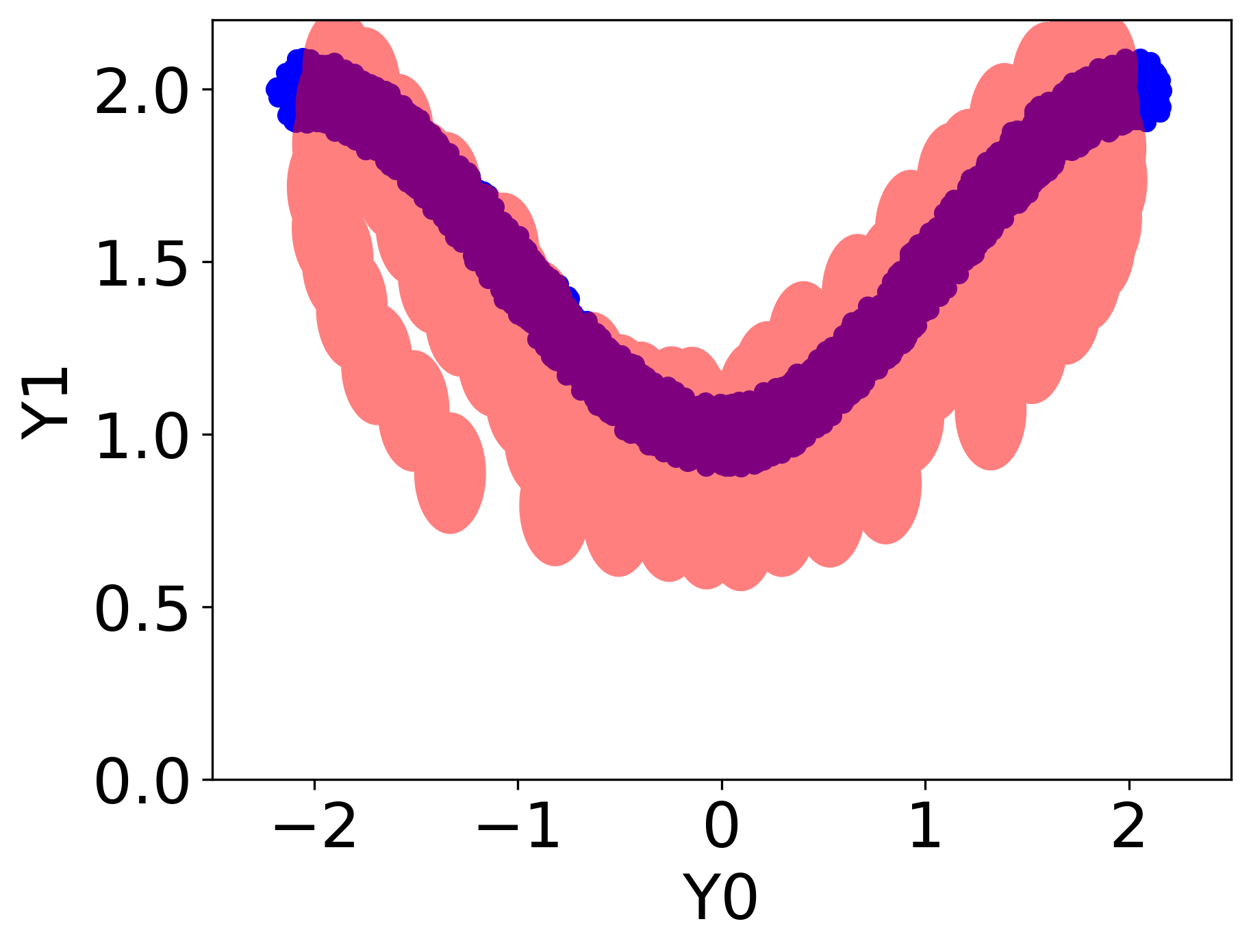}}} &
    {\centering{\rowincludegraphics[width=\imagewidth\linewidth]{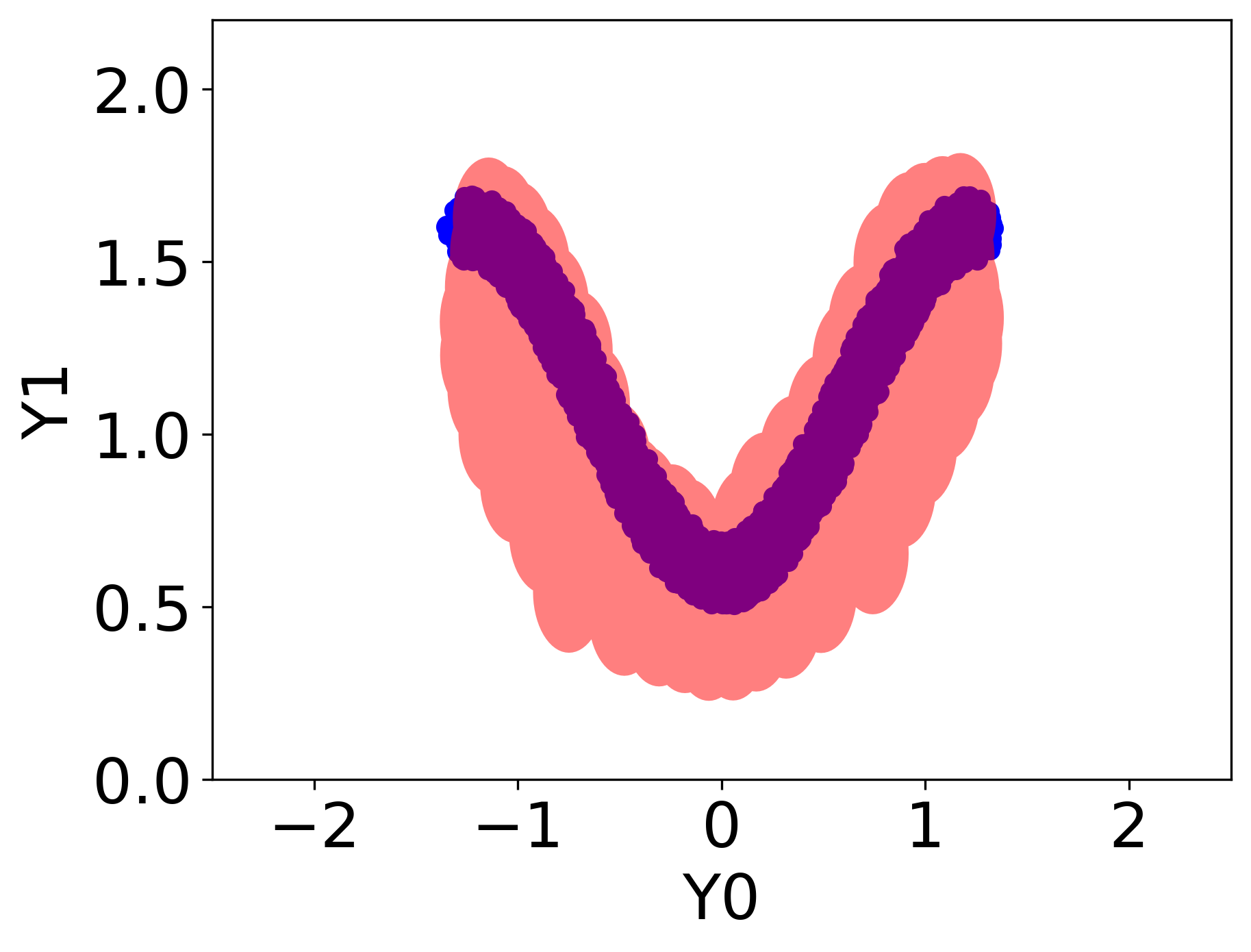}}} \\
    
    Our method & {\centering{\rowincludegraphics[width=\imagewidth\linewidth]{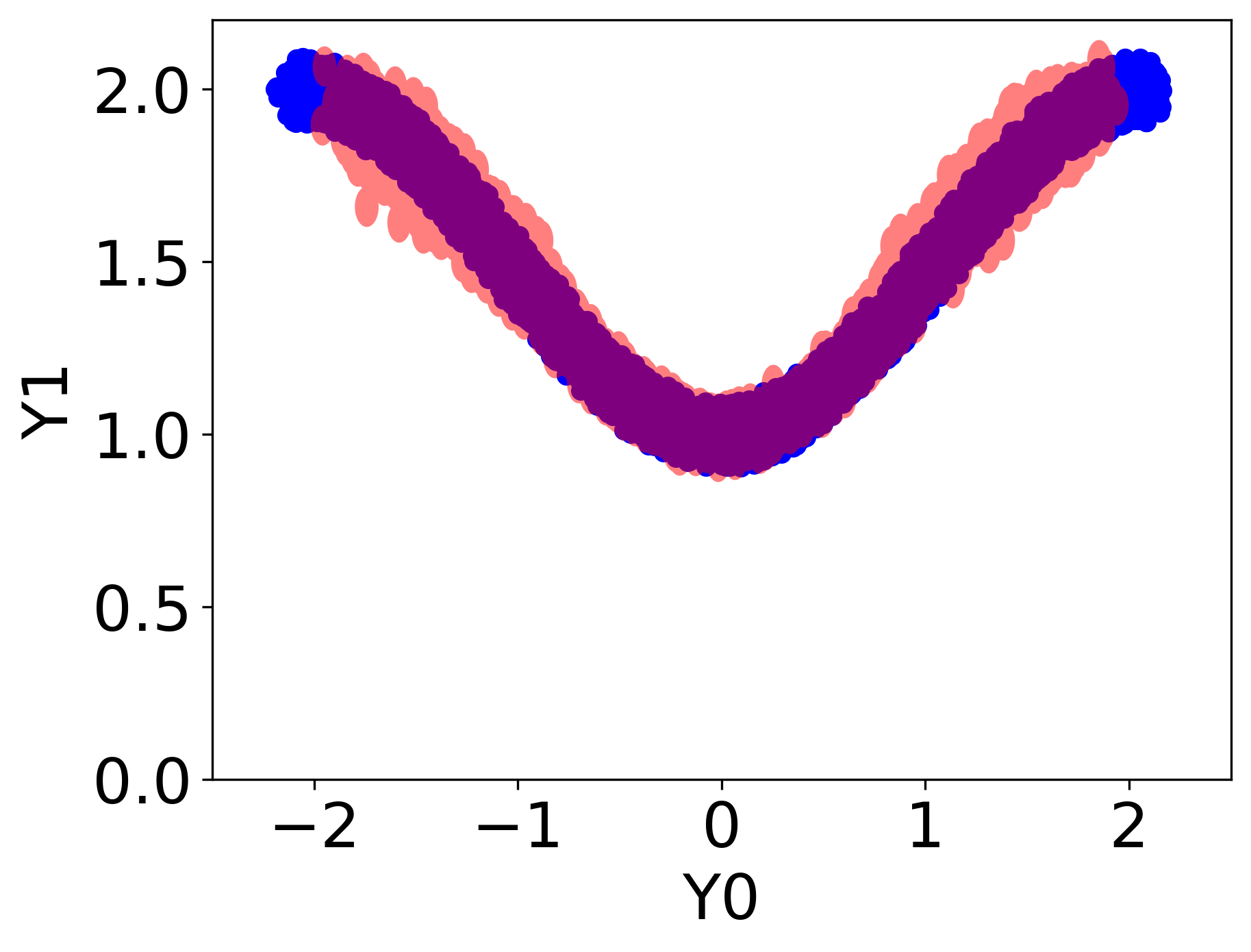}}} & 
     {\centering{\rowincludegraphics[width=\imagewidth\linewidth]{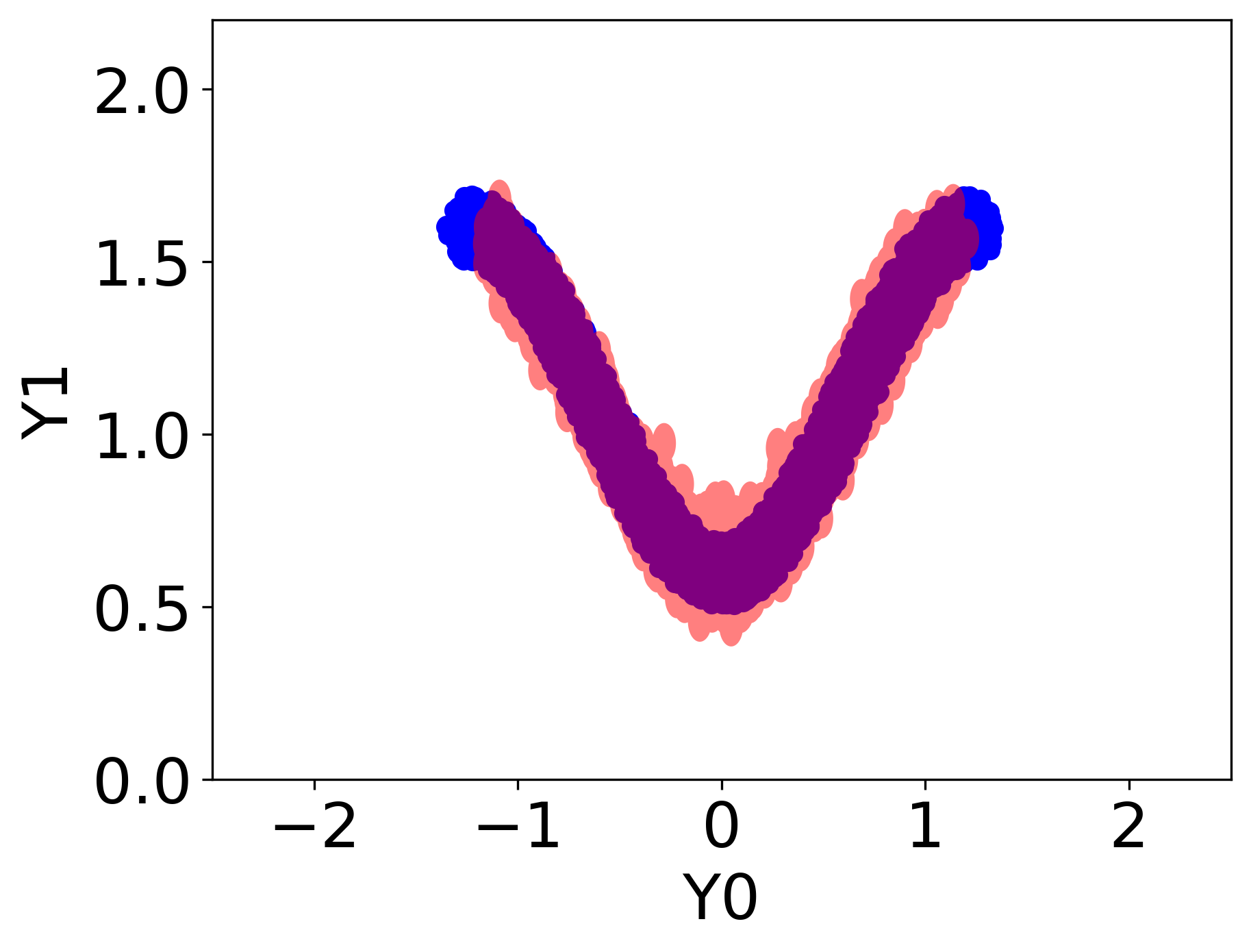}}}
    \end{tabular}%
    }
    \captionsetup{format=hang} \caption{Quantile region obtained by each of the methods: \texttt{Na\"ive QR}, \texttt{NPDQR}, \texttt{VQR}, and our method. See more details about the synthetic data in Appendix~\ref{sec:syn_datasets_details}.}
\label{fig:syn_data_results}%
\end{figure}%

\subsection{One-dimensional Quantile Regression}\label{sec:quantile_regression}
Conditional quantile regression is a commonly used method to estimate a certain quantile, such as the median, of a $Y$ conditional on $X$, given a sample $\{(X_i, Y_i)\}_{i=1}^n$ drawn from a distribution $P_{XY}$. The $\alpha$-th quantile function of $Y$ is defined as:
\begin{equation}
    q_\alpha(x) := \inf\{y\in\mathbb{R}: F(y \mid X = x)\geq \alpha\},
\end{equation}
where $F$ is the CDF of $Y \mid X=x$.
One application of conditional quantiles is to obtain a prediction interval for a one-dimensional response, as presented next. Denote by $\alpha_{\textrm{lo}}=\alpha/2$, $\alpha_{\textrm{hi}}=1-\alpha/2$ the lower and upper quantile levels, respectively. Given the lower and upper quantiles $q_{\alpha_{\textrm{lo}}}(x), q_{\alpha_{\textrm{hi}}}(x)$, the prediction interval for $Y$ given $X=x$ is defined as:
\begin{equation}
C(x) = [q_{\alpha_{\textrm{lo}}}(x), q_{\alpha_{\textrm{hi}}}(x)].
\end{equation}
By construction, the interval satisfies the requirement in \eqref{eq:cov_statement}. While the true conditional quantiles are unknown, they can be estimated empirically by solving an optimization problem, e.g., by minimizing the pinball loss \citep{QR, quantile_regression, estimating_cond_quantiles_pinball}. This process is known to yield estimations that are asymptotically consistent under some regularity conditions \citep{quantile_regression_forests, non_parametric_qr, estimating_cond_quantiles_pinball}. Even though the estimated quantiles are not perfectly accurate, they have been shown to be adaptive to local variability \citep{qr_mm, qr_conditional_density, quantile_regression, quantile_regression_forests, non_parametric_qr, estimating_cond_quantiles_pinball}.

\subsection{Na\"ive Multivariate Quantile Regression}\label{sec:naive_qr}
The idea presented in the previous section can be extended to build a quantile region for a multivariate response. The na\"ive approach regresses to the upper and lower quantile for each dimension separately. The nominal coverage level for each dimension is set to be $1-\beta$, where $\beta=\alpha/d$. This process results in a prediction interval $C^j$ for each feature in the response vector that attains the right coverage rate in the population level:
\begin{equation}
\mathbb{P}[Y_{n+1}\in C^j(x)] = 1-\beta.
\end{equation}
The prediction intervals are used to construct the quantile region in the following way:
\begin{equation}
R(x) = C^1(x) \cross  C^2(x) \cross ... \cross  C^d(x).
\end{equation}
Notice that the resulted quantile region is a rectangle, for any distribution of $Y\mid X$. Furthermore, in the ideal infinite-samples case, the produced quantile region satisfies the coverage requirement~\eqref{eq:cov_statement}, as proved in Appendix~\ref{sec:naive_qr_proof}.
Even though this method converges to the desired coverage level with infinite data, the quantile regions it produces are not flexible, and too conservative. This problem is illustrated in Figure \ref{fig:syn_data_results}, where we see that the regions produced by this na\"ive approach do not reflect the true distribution of the response, as opposed to the other methods: the true distribution is v-shaped, whereas the estimated regions are rectangles. Moreover, further experiments (Section \ref{sec:experiments}) reveal that this method produces the largest quantile regions among all existing methods, indicating poor statistical efficiency. 

\subsection{Directional Quantile Regression}\label{sec:dqr}
The next approach, which we refer to as \emph{directional quantile regression} (\texttt{DQR}) \citep{kong2012quantile, paindaveine2011directional, bovcek2017directional}, is not restricted to produce rectangle-shaped quantile regions, in contrast to the na\"ive approach, and thus can improve the statistical efficiency. However, \texttt{DQR} is also limited as it can only produce \emph{convex} quantile regions, as it coincides with Tukey depth \citep{tukey1975, kong2012quantile}. Observe that the na\"ive method from Section~\ref{sec:naive_qr} estimates the boundaries of the quantile region in four directions $u \in \{(0,1), (1,0), (0,-1), (-1,0)\}$. The \texttt{DQR} method extends this procedure, and estimates the boundaries in all directions, as displayed in Figure \ref{fig:dqr_halfspaces}.
Formally, \texttt{DQR} first projects $Y\in \mathbb{R}^d$ using a direction $u\in \mathbb{S}^{d-1}$, where $\mathbb{S}^{d-1} := \{u\in\mathbb{R}^d: \norm{u}_2=1\}$ is the unit sphere of $\mathbb{R}^d$. The quantile region boundaries are defined as the following hyperplanes. The order-$\alpha$ quantile of $Y$ given $X=x$ in a direction $u\in \mathbb{S}^{d-1}$ is any element of the collection of hyperplanes:
\begin{equation}
\pi_{\alpha u, x} := \{(x,y)\in \mathbb{R}^p \cross \mathbb{R}^d : u^Ty= f_{\theta}(x, u) \},
\end{equation}
with
\begin{equation}
\theta =\operatorname*{argmin}_{{\theta'}} \ \ \ {\frac{1}{n}{\sum_{i=1}^{n} {\rho_\alpha(u^TY_i , f_{\theta'}(X_i,u))}}},
\end{equation}
where $\theta$ are the parameters of the regression model, $f_\theta(x, u): \mathbb{R}^{p+d} \rightarrow \mathbb{R}$ is the regression function and $\rho_\alpha$ is the pinball loss, expressed as:
\begin{equation}
\rho_\alpha(y,\hat{y}) = 
\begin{cases} 
      \alpha(y-\hat{y}) & y-\hat{y} > 0, \\
      (1-\alpha)(\hat{y}-y) & \textrm{otherwise}.
   \end{cases}
\end{equation}
\cite{paindaveine2011directional} defined $f_\theta(x, u)$ as a linear function, i.e., $f_\theta(x, u) = \theta(u)^Tx$, where the coefficients $\theta(u) \in \mathbb{R}^p$ are a function of the direction $u$. A solution for each direction defines the following half-space:
\begin{equation}
H^+_u(x)=\{y\in\mathbb{R}^d:  u^Ty \geq f_{\theta}(x, u)\}.
\end{equation}
Figure \ref{fig:dqr_halfspaces} illustrates the half-spaces obtained from different directions. As the figure implies, the quantile region is defined as the intersection of all half-spaces obtained from all directions:
\begin{equation}\label{eq:qr_def_dqr}
R(x) = \cap_{u\in \mathbb{S}^{d-1}} H^+_u(x).
\end{equation}
As shown by \cite{bovcek2017directional}, the conditional quantile regions are closed, convex and nested, i.e., for any $x\in\mathcal{X}$, $R_{\alpha_{1}}(x) \subseteq R_{\alpha_{2}}(x)$ for $\alpha_{1} \geq \alpha_{2}$. The convexity of the constructed regions is also illustrated in Figure \ref{fig:syn_data_results}. However, the quantile region achieves coverage lower than the nominal level in the population level. See Section \ref{sec:need_conf} for more details. We note that there are more recent methods to compute the same quantile regions described in this section \citep{hallin2010multivariate, hallin2015local, charlier2020multiple}. These methods require estimating only a finite set of half-spaces, but they are all limited to construct convex regions. In fact, the method of \cite{hallin2010multivariate} is the common way to compute these directional regions, although in this work we focus on the version proposed by \cite{paindaveine2011directional}, for simplicity. 

\renewcommand\imagewidth{0.4}
\begin{figure}[htbp]
\setstretch{1}
  \centering

    \scalebox{1}{

    \begin{tabular}{cc}
    
     {\centering{\includegraphics[width=\imagewidth\linewidth]{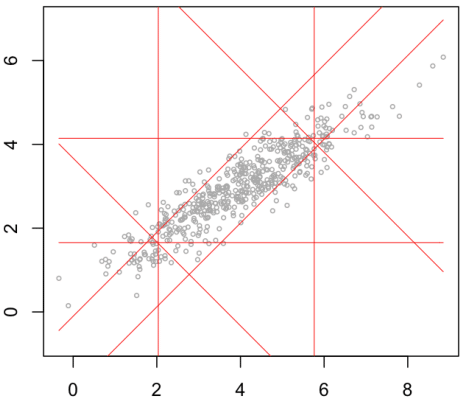}}} &
     
    {\centering{\includegraphics[width=\imagewidth\linewidth]{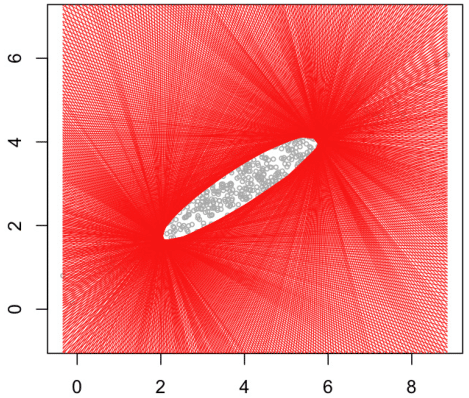}}}

    \\

    \end{tabular}%
    }
    \captionsetup{format=hang} \caption{Half-spaces obtained from \texttt{DQR} on unconditional data. The left and right panels contain 8 and 512 different half-spaces, respectively. Figure credit: \cite{Hallina_possibleapplications}.}
    
\label{fig:dqr_halfspaces}%
\end{figure}%

\subsection{Additional Related Work}\label{sec:additional_related}
The problem of estimating the quantile region for $Y \mid X=x$ was also tackled by \cite{hallin2015local, charlier2020multiple}. However, their proposed methods can only construct convex regions and are based on kernel functions or on a quantization grid, which are infeasible for high-dimensional regressors. A similar technique is the one by \cite{liu2019fast}, which is a fast algorithm to construct half-space regions. A different approach to multivariate quantiles, called \emph{vector quantile regression} (\texttt{VQR}) \citep{carlier2016vector, carlier2017vector, carlier2020vector}, addresses the estimation of the conditional distribution of $Y \mid X=x$, and can produce non-convex quantile regions. Nevertheless, it assumes that $Y$ depends linearly on $X$, and Figure \ref{fig:syn_data_results} shows that this method fails on the synthetic data, in which this assumption is not satisfied. This approach relates to the one proposed by \cite{chernozhukov2017monge} that is based on statistical depth. More recent approaches to construct distribution-free quantile regions, based on geometric tools, are proposed by \cite{hallin2017distribution, hallin2021distribution}. However, these methods can only handle data without covariates. 

\subsection{Our Contribution}\label{sec:our_contr}
We state three features of our method, where each addresses a different limitation of the vanilla \texttt{DQR} method.

\paragraph{Flexible quantile regions.} The \texttt{DQR} method can only construct convex regions, whereas the true distribution of the response might not be convex. As illustrated in Figure \ref{fig:syn_data_results}, the quantile regions produced by this method are too conservative and uninformative.
In contrast, as indicated by this figure, the regions produced by our algorithm are non-convex, reflecting the true distribution of the response.

\paragraph{Feasible for high-dimensional responses.} Due to the curse of dimensionality, the quantile regression problem becomes more difficult when increasing the dimension of the response. In addition, the discrepancy between the nominal coverage level and the empirical coverage rate achieved by \texttt{DQR} worsens as the dimension of the response increases. See Section \ref{sec:need_conf} for more details. Moreover, the time complexity of the methods proposed by \cite{carlier2016vector, carlier2017vector, carlier2020vector} grows exponentially as the dimension of the response increases.
Our method overcomes these limitations by computing the directional quantiles in space $\mathcal{Z}$, whose dimension can be determined regardless of the dimension of the response. As a result, our method is feasible for higher-dimensional response settings.
\paragraph{Guaranteed coverage rate.} The quantile regions constructed by \texttt{DQR} and \texttt{VQR} are not guaranteed to achieve the desired coverage rate. This problem is more severe with the \texttt{DQR}, whose coverage rate is significantly lower than the nominal level, even with infinite data. See Section \ref{sec:need_conf} for additional details. 
To overcome this limitation, we develop a calibration scheme that guarantees the coverage requirement~\eqref{eq:cov_statement}. This process is generic and can also be applied to our proposed method, \texttt{DQR}, \texttt{VQR}, and other methods. Our numerical experiments in Section \ref{sec:experiments} show that, after calibration, all methods achieve the right marginal coverage.

\section{Proposed Method}\label{sec:method}
In this section, we introduce the proposed algorithm, but first, extend \texttt{DQR} beyond linear settings.

\subsection{Non-parametric Directional Quantile Regression}\label{sec:npdqr}
Before describing our contribution, we pause to extend the formulation of conditional directional quantiles as defined by \cite{kong2012quantile, paindaveine2011directional, bovcek2017directional} beyond linear models. This extension of \texttt{DQR} will be used as a subroutine in our main algorithm. To this end, we follow the original \texttt{DQR} from Section~\ref{sec:dqr}, however, use a \emph{non-parametric function class} for $f_\theta$, formulated as neural networks in this work. In such a case, the quantile region for $Y$ conditional on $X=x$ is given by
\begin{equation} \label{eq:npdqr_qregion}
  R(x) = \{y\in\mathbb{R}^d : u^Ty \geq f_{\theta}(x, u), \forall u\in \mathbb{S}^{d-1}\}. 
\end{equation}
The latter stands in contrast with the methods proposed by \cite{paindaveine2011directional, bovcek2017directional} that allow $R(x)$ in \eqref{eq:qr_def_dqr} to depend only linearly on $x$. We refer to this new method as non-parametric \texttt{DQR} (\texttt{NPDQR}) throughout this work.

\subsection{Our Method: Going Beyond Convex Quantile Regions}\label{sec:vae_dqr}
In this section, we present a general approach to construct quantile regions of an arbitrary shape, overcoming the convexity restriction of \texttt{NPDQR}.
Our method relies on the following observation. When the distribution of $Y \mid X$ has level sets of the density that are convex, \texttt{NPDQR} (which must create convex regions) is still appropriate. This motivates us to transform an arbitrary response into a space where it has level sets of the density that are convex. Then, we will apply \texttt{NPDQR} and construct a convex quantile region in that space. Lastly, we will transform it back to the original space of the response, using the inverse of the mapping. By applying a non-linear mapping, this process will result in a quantile region that is not restricted to have a convex shape, having an arbitrary structure.

We now describe this procedure in detail. We start by learning a mapping that transforms a general distribution $Y \mid X=x$ into a latent distribution $Z_x$ whose level sets are convex. In this work, we focus on mapping $Y\mid X$ to a $r$-dimensional standard normal distribution, which is not only spherical, but also has convex level sets. To learn such a mapping, we fit a \emph{conditional variational auto-encoder} (CVAE) \citep{cvae} on the training set $\{(X_i, Y_i)\}_{i\in \mathcal{I}_1}$, and obtain the non-linear transformation between space $\mathcal{Y}$ to space $\mathcal{Z}$. For technical details regarding CVAE, see Appendix~\ref{sec:cvae_and_kl_lambda}. For our purposes, an ideal 
CVAE $\left( \mathcal{E}(y;x), \mathcal{D}(z;x) \right)$ should satisfy the following:
\begin{equation}
Z_x=\mathcal{E}(Y;X=x) \sim \mathcal{N}(0,1)^r, \quad \mathcal{D}(Z_x ; X=x)=Y.
\end{equation}
Since $Z_x$ is spherically distributed, the conditional distribution $Z_x \mid X_{n+1}=x$ for a new test point $X_{n+1}$ has a convex level sets. Figure \ref{fig:cvae_dqr_scheme} illustrates this process. The top panel visualizes the non-linear mapping $Y \mid X \rightarrow Z_x \sim \mathcal{N}(0,1)^3$ obtained by the CVAE model. Observe how the distribution $Z_x \mid X=x$ has approximately a spherical shape. Observe also that the inverse transformation is fairly accurate, so it can map samples from space $\mathcal{Z}$ back to space $\mathcal{Y}$.

Since the distribution of $Z \mid X$ is approximately spherical, \texttt{NPDQR} can estimate effectively its quantile region. We therefore fit \texttt{NPDQR} in space $\mathcal{Z}$. First, we map the response vectors of the training set to space $\mathcal{Z}$, and obtain the transformed training set $\{\left(X_i, \mathcal{E}(Y_i; X_i)\right)\}_{i\in\mathcal{I}_1}$. Next, we fit a \texttt{NPDQR} model on the transformed training samples, as described in Section \ref{sec:dqr}. This process results in a model that can construct a quantile region $R_{\mathcal{Z}}(x) \subseteq \mathcal{Z}$, for any given feature vector $x$. Notice that even though the constructed regions are convex, they are appropriate, since the distribution of $Z \mid X$ is approximately spherical. That is, \texttt{NPDQR} is applied in a space for which it is well-suited. The procedure of fitting the \texttt{NPDQR} model is summarized in the bottom panel of Figure \ref{fig:cvae_dqr_scheme}, displaying (in red) the quantile region constructed in space $\mathcal{Z}$ during training, for a specific feature vector $x$. 
\renewcommand\imagewidth{0.5}
\begin{figure}[htbp]
\setstretch{1.2}
    \scalebox{1.}{
    \begin{tabular}{cc}

    {\parbox{0.1\linewidth} {CVAE training}} & {\parbox{\imagewidth\linewidth}{\centering{\includegraphics[width=1.1\linewidth]{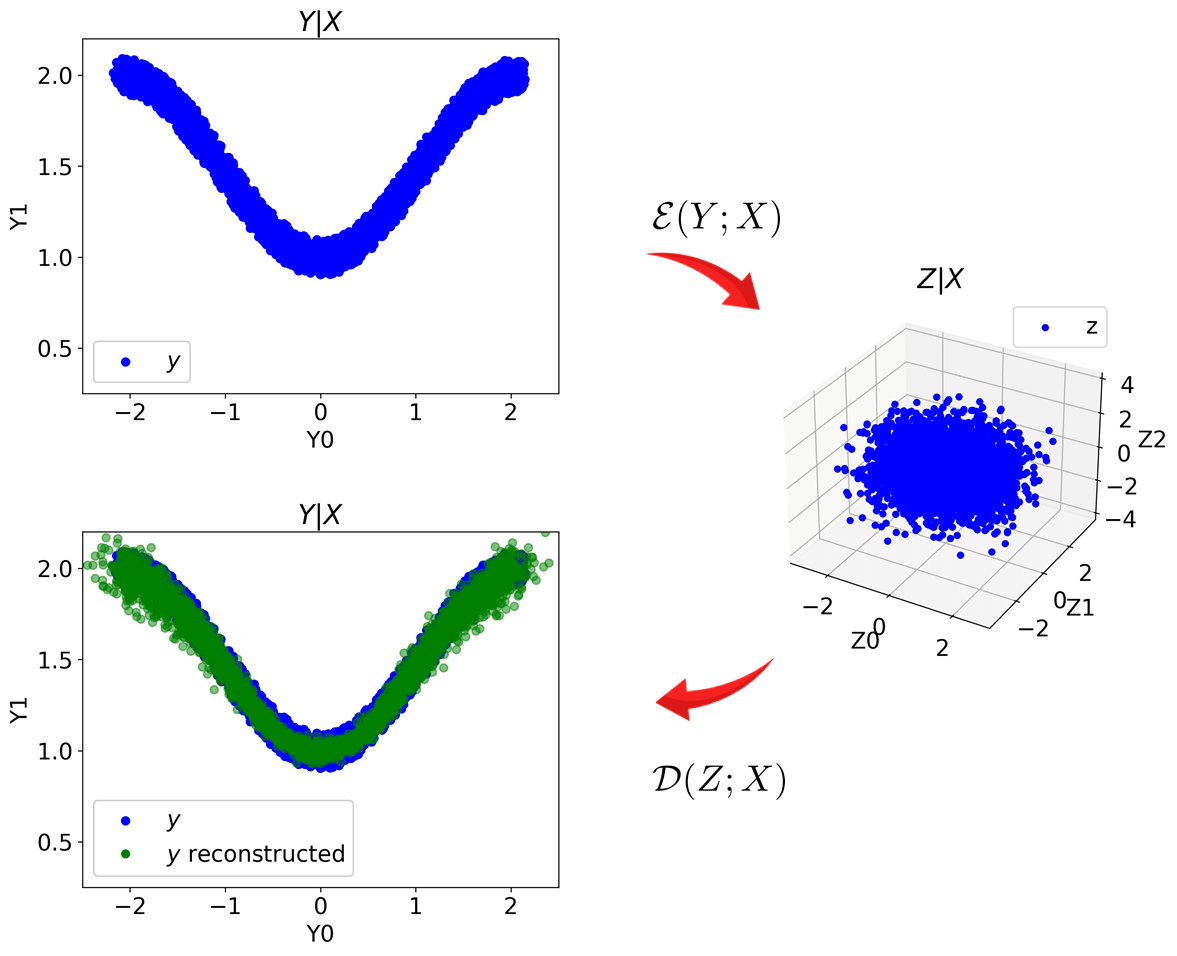}}}} \\ \\
  
    {\parbox{0.1\linewidth}{\texttt{NPDQR} training}} & {\parbox{\imagewidth\linewidth}{\centering{\includegraphics[width=1.7\linewidth]{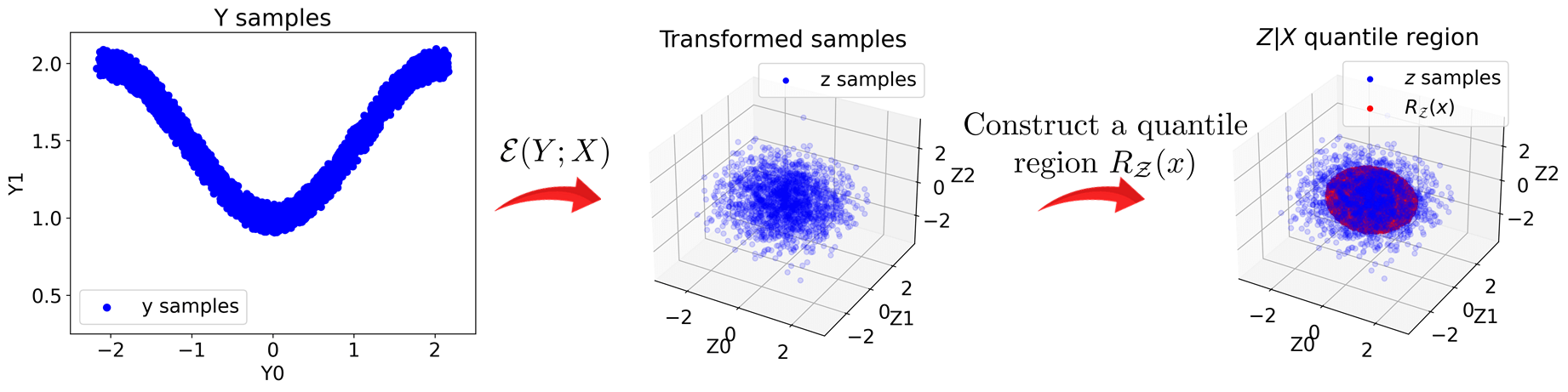}}}} \\

    \end{tabular}
    }
    \captionsetup{format=hang} \caption{CVAE and \texttt{NPDQR} training schemes on the synthetic data. For further details regarding the synthetic data, see Appendix~\ref{sec:syn_datasets_details}.}
    
\label{fig:cvae_dqr_scheme}
\end{figure}%

At test time, given a test point $X_{n+1}$, we (i) construct a quantile region $R_\mathcal{Z}(X_{n+1}) \subseteq \mathcal{Z}$ by applying the fitted \texttt{NPDQR} model; and (ii) transform the estimated region to the original space $\mathcal{Y}$, forming the desired quantile region: 
\begin{equation}\label{eq:vae_dqr_r_y}
R_\mathcal{Y}(X_{n+1}):=\mathcal{D}(R_\mathcal{Z}(X_{n+1});X_{n+1})\subseteq{\mathcal{Y}}.
\end{equation}
Observe that $R_\mathcal{Y}(X_{n+1})$ is the quantile region of $X_{n+1}$ in $\mathcal{Y}$. From a practical point of view, the function $\mathcal{D}$ can only map a discrete set of points from $R_\mathcal{Z}(X_{n+1})$, and therefore the resulting set $R_\mathcal{Y}(X_{n+1})$ is a discretization of the quantile region. We address this important issue in Section \ref{sec:intuition} and show how to construct a continuous region from the discretized $R_\mathcal{Y}(X_{n+1})$.
The test procedure is illustrated in Figure~\ref{fig:inference_scheme}, in which, we can see that while the quantile region in space $\mathcal{Z}$ has a convex shape, the transformed one (in space $\mathcal{Y}$) has the desired non-convex structure. 
The whole procedure is summarized in Algorithm~\ref{alg:vae_dqr}, which we refer to as \emph{Spherically Transformed} \texttt{DQR} (\texttt{ST-DQR}).


\begin{figure}[htbp]
\setstretch{1.1}
  \centering
     {\centering{\includegraphics[width=1\linewidth]{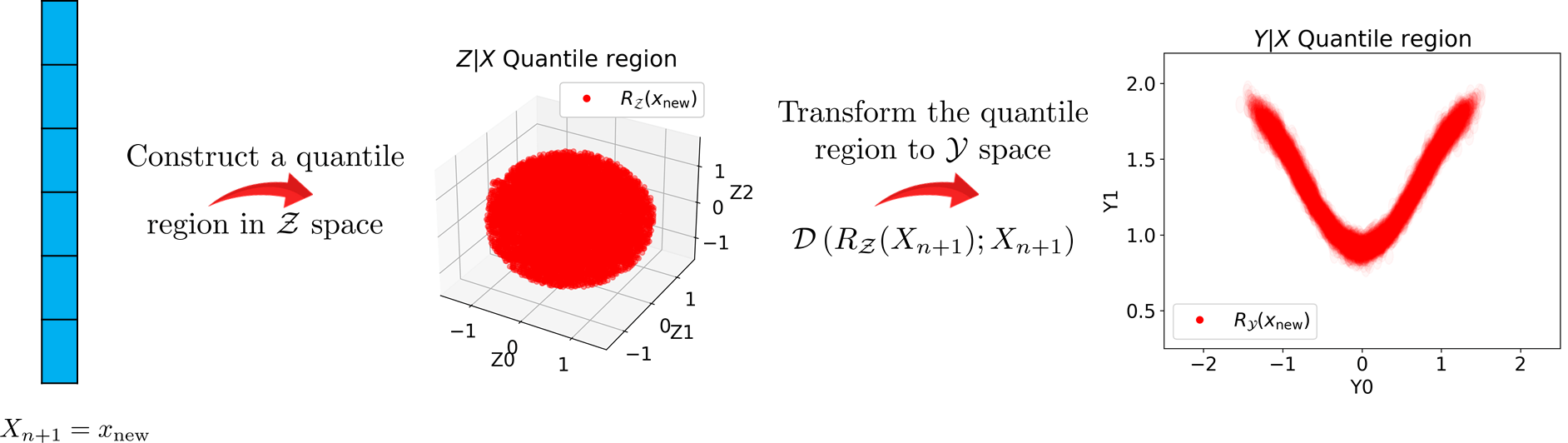}}}
    \captionsetup{format=hang} \caption{The test procedure on the synthetic data, given a new test point $X_{n+1}=x_\textrm{new}$.}
\label{fig:inference_scheme}%
\end{figure}%

\begin{algorithm}[t]
	\captionsetup{format=hang} \caption{Spherically transformed \texttt{DQR} (\texttt{ST-DQR})}
	\label{alg:vae_dqr}
	
	\textbf{Input:}
	\begin{algorithmic}
		\State Data $(X_i, Y_i) \in \mathbb{R}^p \cross \mathbb{R}^d, i \in \mathcal{I}_1$.
		\State Miscoverage level $\alpha \in (0,1)$.
		\State Directional quantile regression algorithm, e.g., \texttt{NPDQR} from Section \ref{sec:npdqr}.
		\State Conditional variational auto-encoder algorithm $(\mathcal{E}(y;x), \mathcal{D}(z;x))$; see Section \ref{sec:cvae_and_kl_lambda}.
		\State A test point $X_{n+1}=x$.
	\end{algorithmic}
	
	\textbf{Training time:}
	\begin{algorithmic}
        \State Fit a CVAE model on the data $\{(X_i, Y_i)\}_{i\in\mathcal{I}_{1}}$. See \citep{cvae}.
		\State Transform the response values $Y_i$ to the space $\mathcal{Z}$: $Z_{i} = \mathcal{E}(Y_i; X_i), i \in \mathcal{I}_1$.
		\State Fit a directional quantile regression model on the training set in space $\mathcal{Z}$ $\{(X_i, Z_i) : i\in \mathcal{I}_1 \}$ to obtain a method to construct quantile regions in $\mathcal{Z}$, denoted by $R_\mathcal{Z}(x)$.
	\end{algorithmic}
	\textbf{Test time:}
	\begin{algorithmic}
		\State Construct the quantile region in $\mathcal{Z}$ $R_\mathcal{Z}(X_{n+1}=x)$.
		\State Transform the quantile region to $\mathcal{Y}$: $R_\mathcal{Y}(X_{n+1}=x)=\mathcal{D}\left(R_\mathcal{Z}(X_{n+1}=x);X_{n+1}=x\right)$.
	\end{algorithmic}
	
	\textbf{Output:}
	\begin{algorithmic}
		\State A quantile region $R_\mathcal{Y}(X_{n+1}=x)$ for the unseen input $X_{n+1}={x}$.
	\end{algorithmic}
\end{algorithm}

We pause here to highlight several features of the proposed algorithm. First, we use a CVAE model which is nonlinear and non-convex, so we can obtain arbitrary quantile regions, unlike previous approaches. Second, since we apply \texttt{NPDQR} in a latent $r$-dimensional space, where $r$ is a hyper-parameter, our method can effectively treat high-dimensional response variables by choosing $r < d$. We demonstrate this in Section \ref{sec:syn_data_results}, in which we display the results obtained by different methods on synthetic data sets with higher-dimensional responses.

\subsection{Theoretical Results}\label{sec:vae_dqr_theorems}
We explain a formal property satisfied by our proposed algorithm that supports the behavior we observed in Figure~\ref{fig:syn_data_results}.
We would like our quantile region to reflect the true distribution of the response variable. For example, in the synthetic data from Figure \ref{fig:syn_data_results}, the quantile region should only cover blue points, i.e., areas where the response can be present. Formally, we ask that the quantile region will be contained in the support of $Y \mid X=x$. We now show that a quantile region constructed by our method satisfies this property.
\begin{theorem}\label{thm:qr_in_y}
Suppose $Y\mid X=x$ has a continuous distribution for all $x$. Suppose $\left( \mathcal{E}(y;x), \mathcal{D}(z;x) \right)$ is a CVAE model that satisfies:
\begin{equation}
\forall x\in \textrm{supp}(X): Z_x=\mathcal{E}(Y;X=x) \in \mathbb{R}^r, \quad \mathcal{D}(Z_x ; X=x) \stackrel{d}{=} Y \mid X=x,
\end{equation}
where $\mathcal{E}$ and $\mathcal{D}$ are continuous functions.
Suppose $R_{\mathcal{Z}}(x)$ is a quantile region in space $\mathcal{Z}$. Define the quantile region in space $\mathcal{Y}$ as: $R_{\mathcal{Y}}(x) = \mathcal{D}(R_{\mathcal{Z}}(x);x).$
Then the quantile region $R_{\mathcal{Y}}(x)$ satisfies:
\begin{equation}
\label{eq:correct_support}
R_{\mathcal{Y}}(x) \subseteq \textrm{supp}(Y\mid X=x).
\end{equation}
\end{theorem}
All proofs are given in Appendix~\ref{sec:theo_results}. Even though the requirement in~\eqref{eq:correct_support} is a modest bar, we see in Figure~\ref{fig:syn_data_results} that, unlike the proposed method, the other methods do not satisfy this property.
In conclusion, we have shown that a quantile region in $\mathcal{Y}$ constructed by our method does not contain spurious portions, since it does not cover areas outside the support of $Y \mid X=x$. As a complementary result, we also give a lower bound for the coverage rate of a quantile region constructed by our method in Appendix~\ref{sec:thm_3_proof}. To achieve the exact nominal coverage level, we propose a calibration procedure, described in Section \ref{sec:conformalization}.

\section{Calibration}\label{sec:conformalization}
In this section, we introduce a procedure to calibrate quantile regions to exactly achieve $1-\alpha$ coverage. The procedure is modular and can be used with any quantile region algorithm, such as \texttt{DQR}, \texttt{VQR}, or our proposed method from the previous section. At a technical level, the calibration scheme instantiates split conformal prediction~\citep{vovk2005algorithmic} in a way that is compatible with multi-dimensional quantile regions.

\subsection{\texttt{DQR} Requires Estimating Extreme Quantiles}\label{sec:need_conf}
To motivate our calibration scheme, we first point out that the parameter $\alpha$ in \texttt{DQR} does not correspond to the coverage level. This phenomenon is known in the literature \citep{zuo2000general,tukey1975}, and in this section, we provide an intuitive explanation and an example illustrating this problem. As a result, the \texttt{DQR} regions have a coverage level unknown to the user without further calibration, such as the one described in this section.
This problem arises from the definition of the \texttt{DQR} quantile region as an intersection of infinite half-spaces, where each covers~$1-\alpha$ of the distribution; see Figure~\ref{fig:dqr_halfspaces}. As a result, their intersection, i.e., the quantile region output by \texttt{DQR}, covers strictly less than $1-\alpha$ of the distribution. To make this precise, we now analyze the coverage rate of a quantile region constructed with the \texttt{DQR} method, in the setting in which $Y \mid X = x \sim \mathcal{N}(0,1)^r$ (see Appendix~\ref{sec:DQR_cov} for the full calculation). The left panel of Figure \ref{fig:dqr_cov} displays the coverage rate of a quantile region constructed by \texttt{DQR} as a function of the dimension $r$, when the nominal coverage level is set to 90\%. The right panel in that figure presents the coverage of a \texttt{DQR} quantile region as a function of the directional quantile level $1-\alpha$ for $r=3$. We see that the achieved coverage is far below the nominal rate. For example, to construct regions that truly have coverage $90\%$ in a three-dimensional response setting, one would need the $99.38\%$ directional quantiles. Unfortunately, such extreme quantiles are impractical to estimate, as shown by~\cite{extremal_qr}.
In summary, the \texttt{DQR} regions do not achieve the nominal coverage rate, even for reasonable quantile levels. The coverage level is the scaling of interest to the user, so we turn to formulate a calibration scheme that guarantees the desired coverage level.

\renewcommand\imagewidth{0.4}
\begin{figure}[htbp]
\setstretch{1.1}
  \centering

    \scalebox{1}{

    \begin{tabular}{cc}
    
    {\centering{\includegraphics[width=\imagewidth\linewidth]{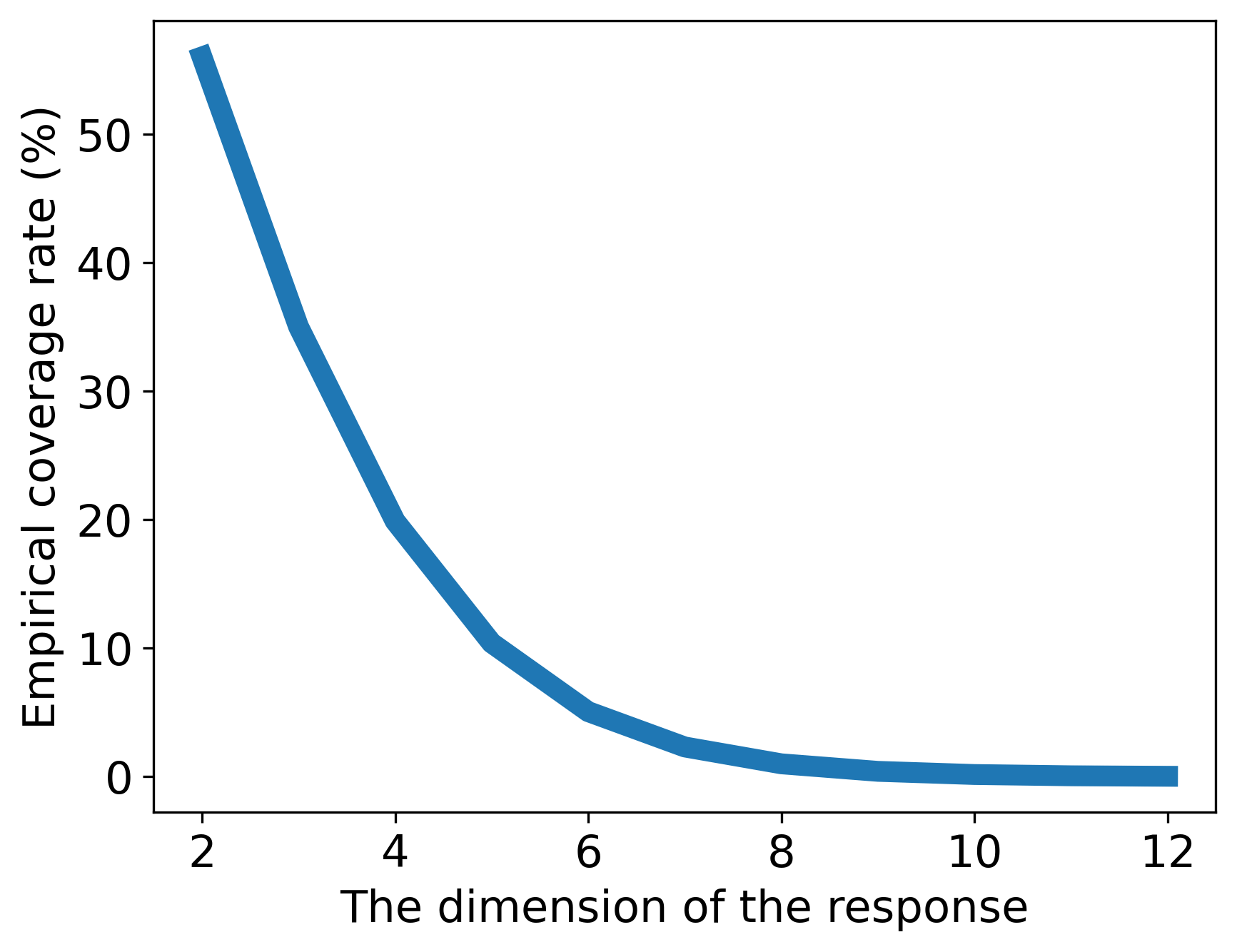}}} &
     
    {\centering{\includegraphics[width=\imagewidth\linewidth]{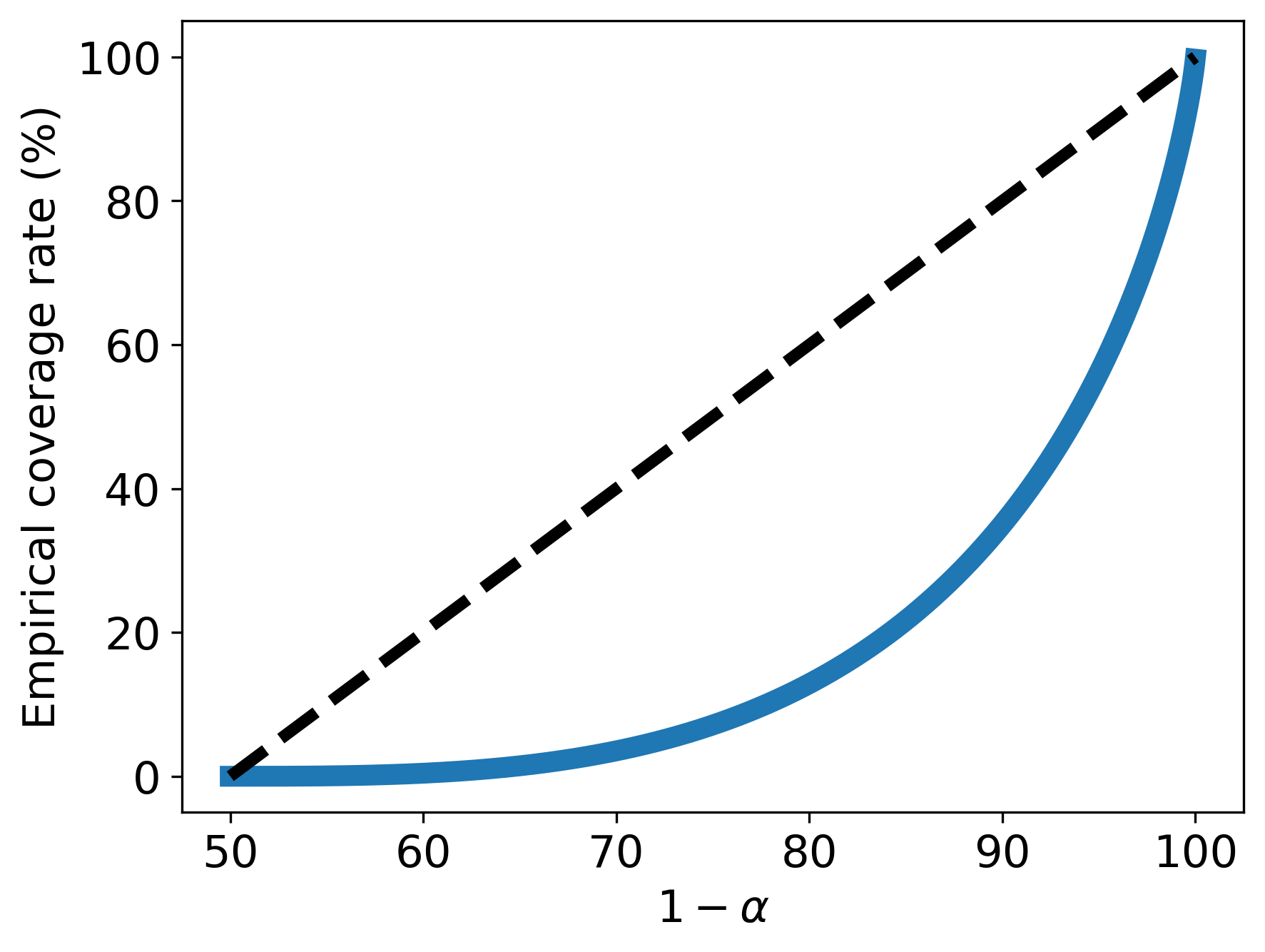}}}

    \end{tabular}%
    }
    \captionsetup{format=hang} \caption{Coverage rate of a quantile region constructed by \texttt{DQR}. Left panel: marginal coverage as a function of $r$, when the desired level is 90\%. Right panel: marginal coverage as a function of $1-\alpha$, for $r=3$.}
    
\label{fig:dqr_cov}%
\end{figure}%

\subsection{Calibration Preliminaries}\label{sec:intuition}

Recall that \texttt{ST-DQR} produces a discretization of the quantile region for a given test point $X_{n+1}$, denoted by $R_\mathcal{Y}(X_{n+1})\subseteq \mathcal{Y}$; see \eqref{eq:vae_dqr_r_y}. We now show how to extend this discrete set to a continuous quantile region that contains infinitely many points. In more detail, we introduce a family of continuous quantile regions, parameterized by a single number. Then, in Section \ref{sec:conf_scheme}, we explain how to choose this parameter to achieve the desired coverage level. The method we develop can also generate valid predictive regions for \texttt{NPDQR} (or any quantile method), by discretizing its quantile regions.

We begin by defining a base region $S^\gamma(x)$, which we will later expand or contract.
We define that a point $y \in \mathcal{Y}$ is inside the base region of $X_{n+1}=x$ if it is close to a point in $R_\mathcal{Y}(x)$. Formally, the base region is given by
\begin{equation}\label{eq:base_qr}
S^\gamma(x) = \left\{ y\in\mathbb{R}^d : \min_{a\in R_\mathcal{Y}(x)}{d(a,y)} \leq \gamma  \right\},
\end{equation}
where $d$ denotes $L_2$ distance, and $\gamma$ is a distance threshold. We initialize $\gamma$ to be $\gamma_\textrm{init}$, which is the 90-th quantile of the distance between two neighbor points in $R_\mathcal{Y}(X_{n+1})$; see more details in Appendix~\ref{sec:gamma_init}. We find that this initialization performs well in the sense that it tends to transform the discrete set into a continuous region, although other options are possible. 

Notice that since $\gamma$ is not tuned, the coverage achieved by this method might be far from the nominal level. To tune this parameter, we first split the data into a training set, indexed by $\mathcal{I}_1$, and a calibration set, indexed by $\mathcal{I}_2$. Denote the coverage rate of the base regions by \begin{equation}
    c_\textrm{init} = \frac{1}{\mid \mathcal{I}_2 \mid} \left| \{Y_i: Y_i \in S^{\gamma_\textrm{init}}(X_i), i\in\mathcal{I}_2 \} \right|,
\end{equation} 
where $|\cdot|$ is the set size.
Depending on $c_\textrm{init}$, we grow or shrink the base region $S^{\gamma_\textrm{init}}$ to the extent required to achieve the desired $1-\alpha$ coverage. We describe these two cases (grow/shrink) separately next. See Appendix~\ref{sec:cal_shrink_figs} for an explanation of why it is important to handle the two cases separately.

\paragraph{Case 1: Too low coverage.}\label{par:increase_qr_case} In this setting, $c_\textrm{init} \leq 1-\alpha$ and therefore we need to enlarge the base region by increasing $\gamma$. Figure \ref{fig:qr_vs_gamma} shows the effect of $\gamma$ on the quantile region and its coverage rate. By inflating $\gamma$, we enlarge the quantile region and, as a result, increase the coverage rate. In Section \ref{sec:conf_scheme} we show how to exploit the calibration set to compute $\gamma_\textrm{cal}$ that rigorously achieves this nominal rate. Given $\gamma_\textrm{cal}$, the calibrated quantile region in this case is formulated as
\begin{equation}\label{eq:small_qr}
S^{\gamma_\textrm{cal}}(x) = \left\{ y\in\mathbb{R}^d : \min_{a\in R_\mathcal{Y}(x)}{d(a,y)} \leq \gamma_\textrm{cal}  \right\}.
\end{equation}
In practice, \texttt{DQR} tends to generate regions with coverage rate below the nominal level (recall Figure~\ref{fig:dqr_cov}), therefore, this regime, where $c_\textrm{init} \leq 1-\alpha$, is most likely to happen in practice.

\renewcommand\imagewidth{0.31}
\begin{figure}[htbp]
\setstretch{1.1}
  \centering
    
    \scalebox{1.}{
    \begin{tabular}{ccc}
     {\centering{\rowincludegraphics[width=\imagewidth\linewidth]{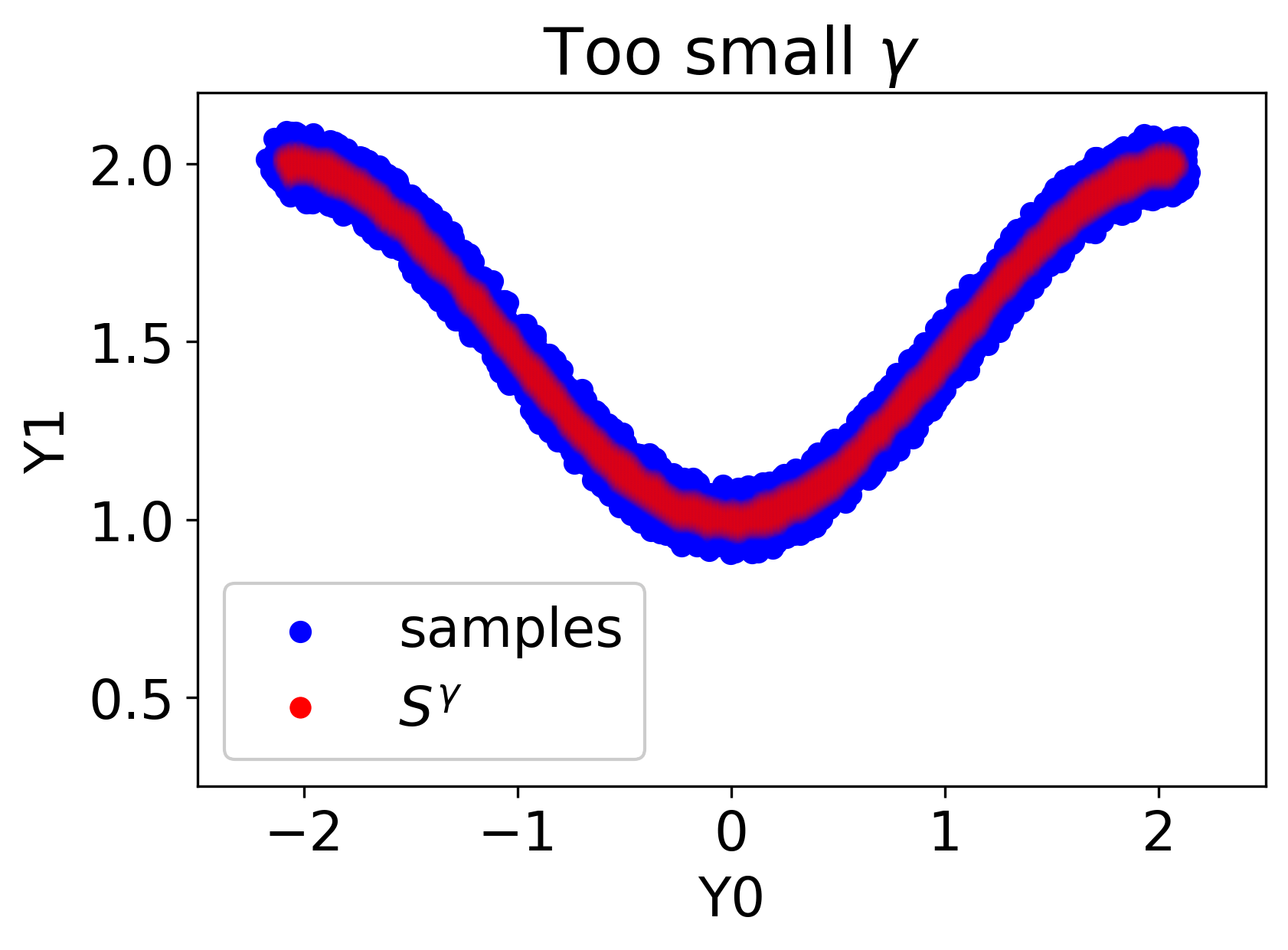}}} & {\centering{\rowincludegraphics[width=\imagewidth\linewidth]{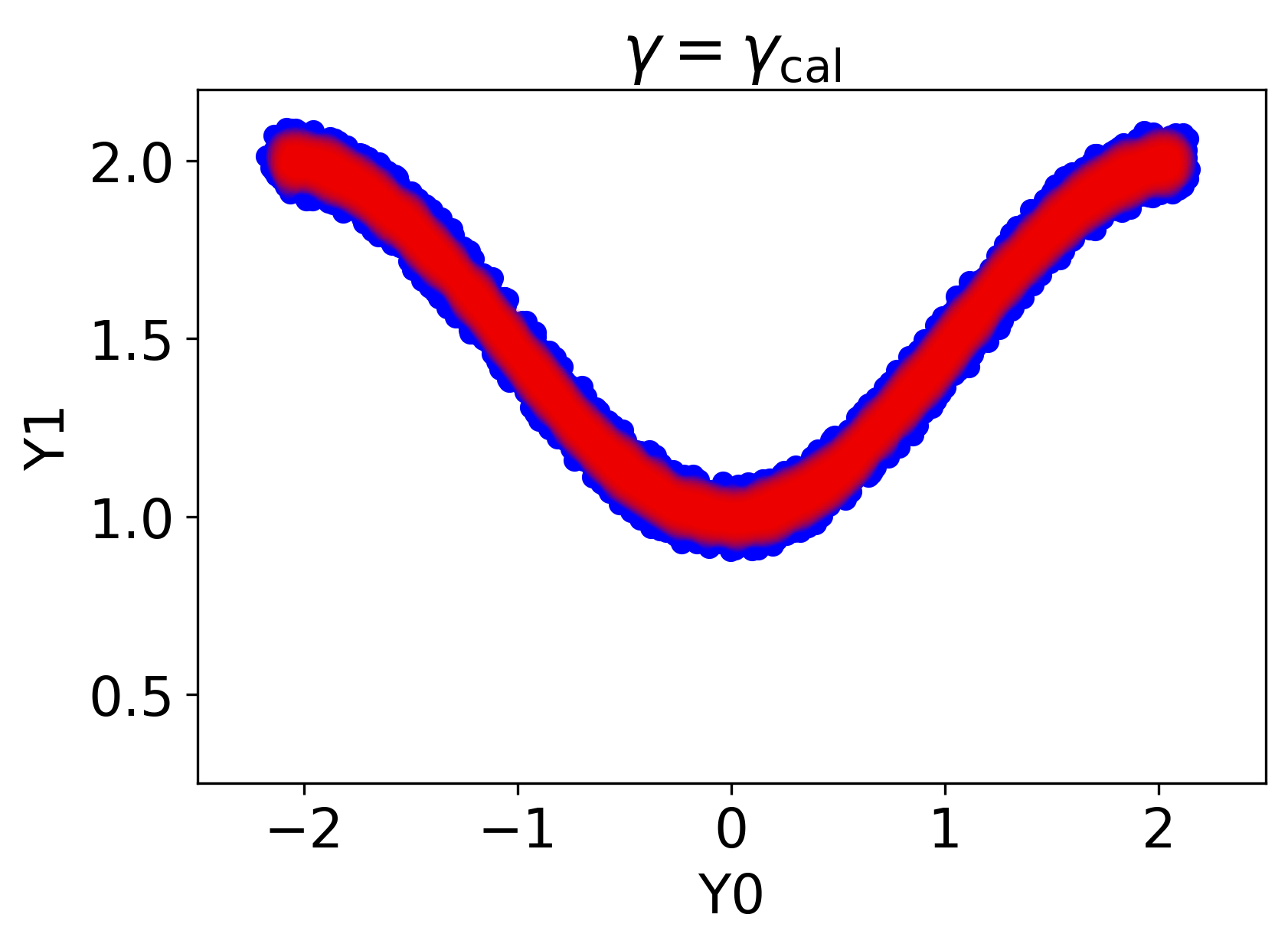}}} & {\centering{\rowincludegraphics[width=\imagewidth\linewidth]{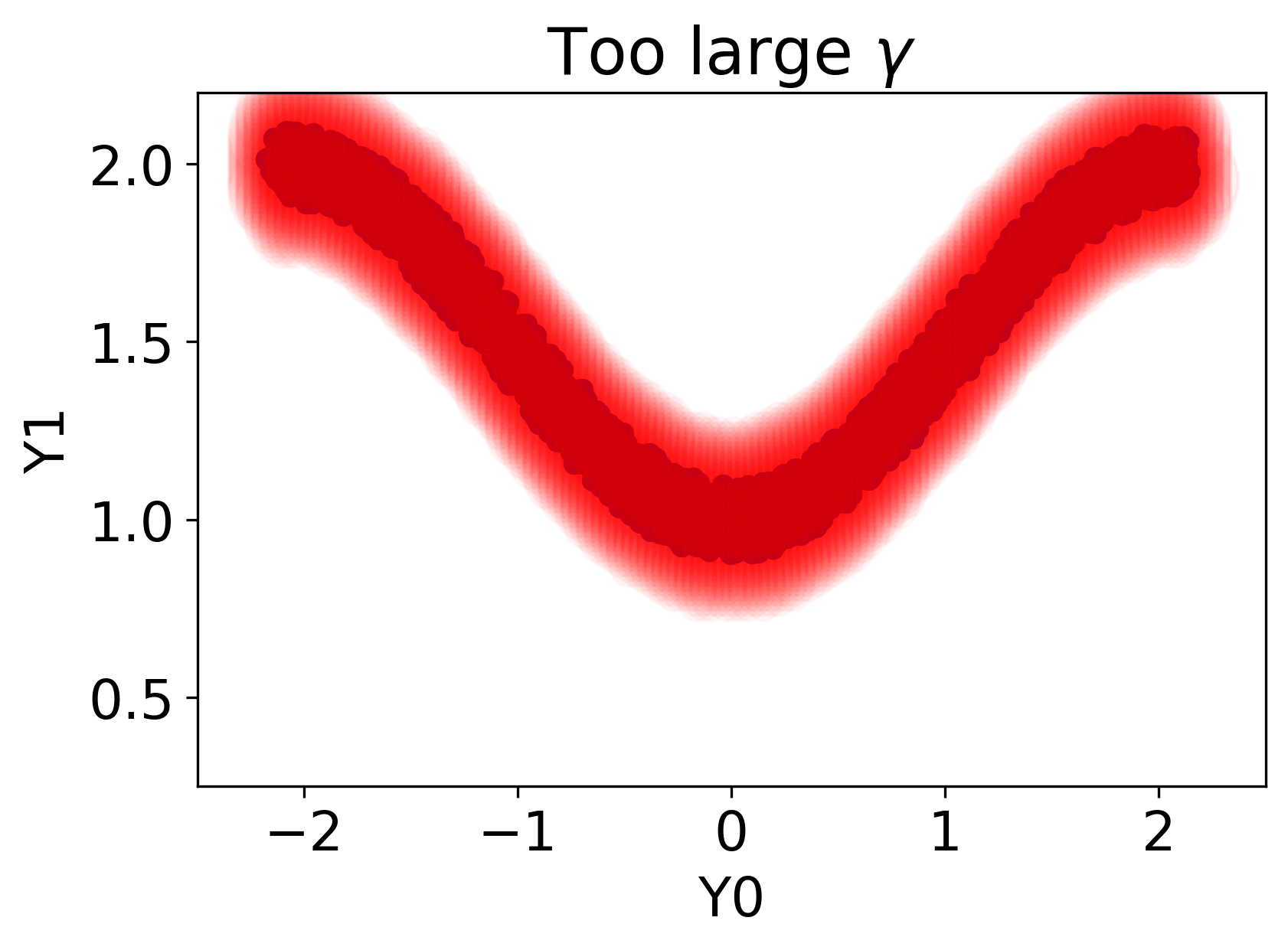}}}

    \end{tabular}%
    }
    \captionsetup{format=hang} \caption{Demonstration of the quantile region under Case~1 (i.e., $\gamma_\textrm{init}$ yields regions of a low coverage rate) for different values of $\gamma$.}

\label{fig:qr_vs_gamma}%
\end{figure}%

\paragraph{Case 2: Too high coverage.}\label{par:shrink_qr_case} This scenario treats the case where $c_\textrm{init} > 1-\alpha$, which is less likely to occur in practice. In this setting, analogously to Case~1, one could decrease $\gamma$ to reduce the coverage rate. This strategy, however, may result in a new region that is composed of many disjoint sub-regions, as explained in Appendix~\ref{sec:cal_shrink_figs}. This construction is undesired and hard to interpret for continuous distributions, such as the one presented in Figure~\ref{fig:syn_data_results}. We therefore alter the scheme in Case~1, and shrink the base region in a different manner. We begin by taking a set of points outside the quantile region, denoted by $R^c_\mathcal{Y}(x)$:
\begin{equation}\label{eq:qr_complement}
R^c_\mathcal{Y}(x) = \left\{ y \in\mathbb{R}^d : \min_{a\in R_\mathcal{Y}(x)}{d(a,y)} > \gamma_\textrm{init} \right\}.
\end{equation}
Next, we say that a point $y$ is inside the quantile region if it is far from its boundaries. Formally, the calibrated quantile region is given by
\begin{equation}\label{eq:large_qr}
S^{\gamma_\textrm{cal}}(x) = \left\{ y\in\mathbb{R}^d : \min_{a\in {R^c_\mathcal{Y}(x)}}{d(a,y)} \geq \gamma_\textrm{cal} \right\},
\end{equation}
where the calibrated threshold parameter, $\gamma_\textrm{cal}$, is defined hereafter. 

\subsection{The Calibration Scheme}\label{sec:conf_scheme}

We now turn to describe how to choose the distance threshold $\gamma$ in a way that guarantees the coverage requirement~\eqref{eq:cov_statement}, by borrowing ideas from conformal prediction. Following the discussion from the previous subsection, we divide the calibration scheme into two cases, depending on the value of $c_\textrm{init}$ which is evaluated on the calibration set.
For $c_\textrm{init} \leq 1-\alpha$ (Case~1), we grow the base quantile region by computing $\gamma_\textrm{cal} > \gamma_\textrm{init}$ as follows:
\begin{align}\label{eq:small_qr_conformity_score}
& E^+_i = \min_{a\in R_\mathcal{Y}(X_i)}{d(a,Y_i)}, \forall i\in \mathcal{I}_2, &\\
& \gamma_{\textrm{cal}} := \ceil*{(n_2+1)(1-\alpha)}\textrm{-th smallest value of } \{E^+_i : i\in \mathcal{I}_2\}, 
\end{align}
where $n_2 = |\mathcal{I}_2|$. The effect of $\gamma$ on the coverage rate is visualized in Figure \ref{fig:cov_vs_gamma}. The figure shows that $\gamma_\textrm{cal}$ is the value for which the empirical marginal coverage rate is equal to the desired level (up to a small correction).
\renewcommand\imagewidth{0.4}
\begin{figure}[htbp]
\setstretch{1.1}
  \centering
  
    \includegraphics[width=\imagewidth\linewidth]{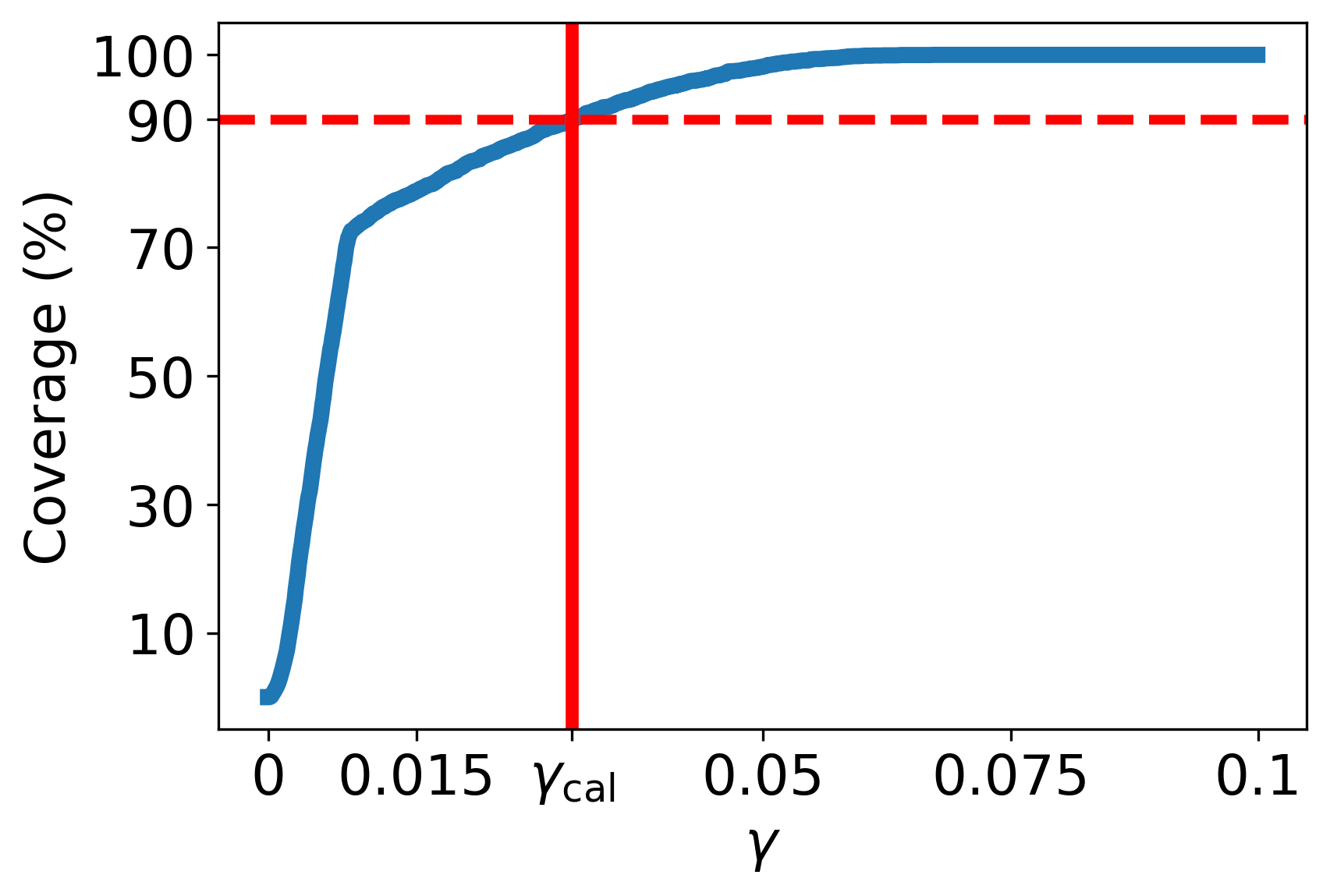}
    
    \captionsetup{format=hang} \caption{Quantile region coverage rate under Case~1 for different values of $\gamma$. The value for which the 90\% marginal coverage rate is attained is $\gamma=\gamma_\textrm{cal}$.}

\label{fig:cov_vs_gamma}%
\end{figure}%
In Case~2, where $c_\textrm{init} > 1-\alpha$, we instead grow the \emph{complement} of the base quantile region by computing $\gamma_{\textrm{cal}}$ as follows:
\begin{align}\label{eq:large_qr_conformity_score}
& E^-_i = \min_{a\in {R^c_\mathcal{Y}(X_i)}}{d(a,Y_i)}, \forall i\in \mathcal{I}_2, &\\
& \gamma_{\textrm{cal}} := {\floor*{(n_2+1)\alpha}}\textrm{-th smallest value of } \{E^-_i : i\in \mathcal{I}_2\}. &
\end{align}
In words, under Case~1 (Case~2), the quantity $E^+_i$ ($E^-_i$) is the distance of $Y_i$ from its closest point \emph{inside} (\emph{outside}) the base quantile region. From a computational perspective, Case~1 is more efficient since usually $| R_\mathcal{Y}(x) | < |R^c_\mathcal{Y}(x)|$. As a result, computing $E_i^+ = \min_{a\in R_\mathcal{Y}(x)}{d(a,Y_i)}$ requires less operations compared to $E_i^-=\min_{a\in {R^c}_\mathcal{Y}(x)}{d(a,Y_i)}$.
\begin{algorithm}[t]
	\captionsetup{format=hang} \caption{Calibrating Multivariate Quantile Regression}
	\label{alg:cal}
	
	\textbf{Input:}
	\begin{algorithmic}
		\State Data $(X_i, Y_i) \in \mathbb{R}^p \cross \mathbb{R}^d, 1 \leq i \leq n$.
		\State Miscoverage level $\alpha \in (0,1)$.
		\State Multivariate quantile regression algorithm, e.g., \texttt{NPDQR} from Section \ref{sec:npdqr}.
		\State An unseen input $X_{n+1}=x$.
	\end{algorithmic}
	
	\textbf{Training time:}
	\begin{algorithmic}
	    \State Randomly split $\{1,...,n\}$ into two disjoint sets $\mathcal{I}_1, \mathcal{I}_2$ of sizes $n_1$ and $n_2 = n-n_1$, respectively.
	    \State Fit the multivariate quantile regression algorithm on the training set $\{(X_i, Y_i) : i \in \mathcal{I}_1\}$.
	\end{algorithmic}
	
	\begin{algorithmic}
		\State Compute the coverage rate of the uncalibrated quantile regions: $c_\textrm{init} \gets \frac{1}{n_2} \left| \{Y_i: Y_i \in S^{\gamma_\textrm{init}}(X_i), i\in\mathcal{I}_2 \} \right|$.\\
		\If{$c_\textrm{init} \leq 1-\alpha$}{
    		\State Compute $E^+_i$ for each $i \in \mathcal{I}_2$, according to Equation~\eqref{eq:small_qr_conformity_score}.
            \State Compute $\gamma_\textrm{cal}$ the $\ceil*{(n_2 + 1)(1-\alpha)}$-th  smallest value of $\{E^+_i\}_{i \in \mathcal{I}_2}$.

        }
        \State \Else{
            \State Compute $E^-_i$ for each $i \in \mathcal{I}_2$, according to Equation~\eqref{eq:large_qr_conformity_score}.
            \State Compute $\gamma_\textrm{cal}$ the $\floor*{(n_2 + 1)\alpha}$-th  smallest value of $\{E^-_i\}_{i \in \mathcal{I}_2}$.

        }
            
	\end{algorithmic}
	
	\textbf{Test time:}
	\begin{algorithmic}
	\State Obtain a base quantile region $R_\mathcal{Y}(X_{n+1}=x)$ using the multivariate quantile regression algorithm.
	\If{$c_\textrm{init} \leq 1-\alpha$}{
        \State Construct the calibrated quantile region $S^{\gamma_\textrm{cal}}$ according to Equation~\eqref{eq:small_qr}.
    }
    \State \Else{

        \State Construct the calibrated quantile region $S^{\gamma_\textrm{cal}}$ according to Equation~\eqref{eq:large_qr}.
    }
	
	\end{algorithmic}
	
	\textbf{Output:}
	\begin{algorithmic}
		\State A quantile region $S^{\gamma_\textrm{cal}}(x)$ for the unseen test point $X_{n+1}=x$.
	\end{algorithmic}
\end{algorithm}
We now state that a quantile region constructed by the above procedure, summarized in Algorithm~\ref{alg:cal}, satisfies the marginal, distribution-free coverage guarantee~\eqref{eq:cov_statement}. The proof is given in Section \ref{sec:thm_2_proof}.
\begin{theorem}\label{thm:cov_guarantee}
If $(X_i, Y_i), i=1,...,n+1$ are exchangeable, then the quantile region $S^{\gamma_\textrm{cal}}(X_{n+1})$ constructed by Algorithm~\ref{alg:cal} satisfies:
\begin{equation*}
\mathbb{P}(Y_{n+1} \in S^{\gamma_\textrm{cal}}(X_{n+1})) \geq 1-\alpha.
\end{equation*}
Moreover, if the distances $E^+_i, E^-_i$ are almost surely distinct, then the quantile region is almost perfectly calibrated:
\begin{equation*}
\mathbb{P}(Y_{n+1} \in S^{\gamma_\textrm{cal}}(X_{n+1})) \leq 1-\alpha + \frac{1}{n_2 +1}.
\end{equation*}
\end{theorem}
We pause here to explain the significance of Theorem~\ref{thm:cov_guarantee}. First, the coverage guarantee of the calibration procedure applies for any sample size, and dimension of $(X,Y)$. In addition, once applying this procedure, we guarantee that the calibrated version of any base multivariate quantile regression method (\texttt{DQR}/\texttt{VQR}/\texttt{ST-DQR}) would attain the desired $1-\alpha$ coverage. Therefore, the calibrated methods would differ only in their statistical efficiency, i.e., the area of the constructed quantile region.

\section{Experiments}\label{sec:experiments}

Herein, we systematically quantify the effectiveness of our proposed method (\texttt{ST-DQR}) and compare its performance to existing techniques (\texttt{Na\"ive QR}, \texttt{NPDQR}, and \texttt{VQR}). Turning to the details of our setup, for all methods except for \texttt{VQR}, we apply a deep neural network as a base model for constructing quantile regions with $1-\alpha = 0.9$ coverage level. The \texttt{VQR} method cannot incorporate neural networks in the same way, so we applied the procedure exactly as proposed by \citet{carlier2016vector}; see Appendix~\ref{sec:vqr} for details. We split the data sets (both real and synthetic) into a training set (38.4\%), calibration (25.6\%), validation set (16\%) used for early stopping, and a test set (20\%) to evaluate performance. Then, we normalize the feature vectors and response variables to have a zero mean and unit variance. Appendix~\ref{sec:exp_settings} gives the details about the network architecture, training strategy, and more information about this experimental protocol. The performance metrics (coverage and area, as described below) are averaged over 20 random splits of the data. For our method, we set the dimension of the latent space to $r=3$; see \ref{sec:vae_r} for other choices of this hyper-parameter. In all experiments, we report only the performance of the calibrated quantile regions, since this puts all methods on the same scale. Specifically, \texttt{ST-DQR}, \texttt{NPDQR}, and \texttt{VQR} are calibrated according to Algorithm~\ref{alg:cal} and \texttt{Na\"ive QR} is calibrated as described in Appendix~\ref{sec:naive_setup}. 
Software implementing the proposed method and reproducing our experiments can be found at \url{https://github.com/Shai128/mqr}

We report the following two metrics, evaluated on test data:
\begin{itemize}
    \item \textbf{Coverage}: The percentage of samples that are covered by the quantile region. The coverage of a point is determined as described in Section~\ref{sec:intuition}.
    \item \textbf{Area}: The area of the generated quantile region. To evaluate this metric, we take a grid in space $\mathcal{Y}$, and define the area to be the number of cells that fall inside the quantile region. See more details in Appendix~\ref{sec:grid_size}.
\end{itemize}

\subsection{Synthetic Data Results}\label{sec:syn_data_results}
We return to the synthetic v-shaped data from Section \ref{sec:synthetic_example} and extend it to higher-dimensional settings. In Appendix~\ref{sec:syn_datasets_details} we describe how we generate such data for increased dimensions of $X$ and $Y$.
Furthermore, we explore two settings: in the first, the relationship between the response variables and the covariates is \textbf{linear}, whereas in the second this relationship follows a \textbf{non-linear} model. In both cases the relationship between the elements in the response vector is non-linear, however, the relationship between $Y$ and $X$ can be either linear or non-linear; see Figure~\ref{fig:syn_data}.
We evaluate the four methods described in this paper (\texttt{ST-DQR}, \texttt{Na\"ive QR}, \texttt{NPDQR}, and \texttt{VQR}) using the synthetic data sets, and examine their robustness to non-linearity, a high regressor dimension, and a high dimensional response vector. We find that the \texttt{VQR} method is feasible only for data sets of small dimensions, so we report the results only for those data sets; see more details in Table \ref{tab:vqr_time_memory} in the Appendix~that summarizes \texttt{VQR}'s runtime and memory footprint. For \texttt{NPDQR} and \texttt{ST-DQR}, we estimate a (pre-calibration) directional quantiles of a level higher than the nominal 90\% rate (see Table \ref{tab:syn_coverage_rates_info}), due to the under-coverage problem presented in Section \ref{sec:need_conf}. For \texttt{Na\"ive QR} and \texttt{VQR}, we set the quantile level to be equal to the target $1-\alpha=90\%$ rate.

Table \ref{tab:syn_data_results} displays the coverage rates and areas of the constructed quantile regions. Observe that all methods attain the nominal coverage level, a consequence of applying our proposed calibration procedure from Algorithm~\ref{alg:cal}. However, the regions constructed by different methods are different in size, as presented in the same table. Here, our proposed method \texttt{ST-DQR} constructs quantile regions that are substantially smaller compared to all other techniques. These results are anticipated, since \texttt{Na\"ive QR} and \texttt{NPDQR} are restricted to produce convex quantile regions, forcing the two to cover irrelevant areas, whereas our method does not have this limitation. In addition, since the linearity assumption of \texttt{VQR} is not satisfied in the non-linear setting, the quantile regions it produces are unnecessarily large. Finally, following Figure \ref{fig:syn_data_results}, we can see the advantages of our non-parametric method: it produces a quantile region of an arbitrary shape, estimating well the conditional distribution of $Y \mid X$ in contrast to the competitive techniques. Table~\ref{tab:syn_data_results} also reports the performance metrics that correspond to data sets with high-dimensional features. These results indicate that \texttt{VQR} is infeasible when the dimension of the feature vector is not small enough, while other methods (including ours) are robust to high-dimensional regressors; see Table \ref{tab:vqr_time_memory} in the Appendix~for more details.

In the case where the response is a four-dimensional vector, the differences between our method and \texttt{Na\"ive QR}/\texttt{NPDQR} become more significant. (Once again, \texttt{VQR} is infeasible in this setting.) Here, \texttt{Na\"ive QR} and \texttt{NPDQR} produce quantile regions with an area larger by a factor of 180-440 than regions constructed by our method. This limitation of the standard methods to handle a high dimensional response is also visualized in Figure \ref{fig:area_vs_d}. One explanation for the substantial improvement that our method achieves is this: while the standard methods work in a four-dimensional space (the dimension of $Y$), our method works in a lower-dimensional space (in this case, the dimension of $Z$ is three), so it can achieve a higher coverage rate for the same directional quantile level; recall Figure \ref{fig:dqr_cov}. Therefore, the calibration applied to our method is milder and does not affect much the base quantile region \eqref{eq:base_qr}. By contrast, the uncalibrated \texttt{NPDQR} has an extremely low coverage rate and therefore requires an aggressive calibration. That is, it must smooth out the original fit, making it more like a round ball and less adaptive to the test point. 

\begin{table}[!htb]
    \centering
    \setstretch{1.4}

    \begin{minipage}{.5\linewidth}
        \captionsetup{format=hang} \caption*{Coverage rate}
        \centering
  \scalebox{0.77}{
      \centering

    \begin{tabular}{ccc|cccc}

    \toprule[1.1pt]
    \textbf{Setting} &  $\boldsymbol{d}$ & $\boldsymbol{p}$ & \textbf{\texttt{ST-DQR}} &  \textbf{\texttt{Na\"ive QR}} & \textbf{\texttt{NPDQR}} & \textbf{\texttt{VQR}} \\
    \midrule
linear    &          2 &          1 &               89.943 &         90.059 &         90.041 &       89.755 \\
linear    &          2 &         10 &               89.926 &         89.789 &         90.131 &       90.065 \\
linear    &          2 &         50 &                89.91 &          89.99 &          89.96 &            - \\
linear    &          2 &        100 &               89.963 &         89.993 &         90.003 &            - \\
nonlinear &          2 &          1 &               90.126 &         90.078 &         90.165 &        90.13 \\
nonlinear &          3 &          1 &               90.165 &         90.114 &         90.021 &       90.156 \\
nonlinear &          3 &         10 &               89.991 &         89.881 &         90.051 &            - \\
nonlinear &          4 &          1 &               90.031 &         90.175 &         89.955 &            - \\
nonlinear &          4 &         10 &               89.792 &         89.841 &         89.956 &            - \\
    \bottomrule[1.1pt]
    
    \end{tabular}%
    }
    \end{minipage}%
    \begin{minipage}{.55\linewidth}
      \centering
        \captionsetup{format=hang} \caption*{Relative area of quantile regions} 
  \scalebox{0.77}{
      \centering
    \begin{tabular}{cccc}

    \toprule[1.1pt]
    \textbf{\texttt{ST-DQR}} &  \textbf{\texttt{Na\"ive QR}} & \textbf{\texttt{NPDQR}} & \textbf{\texttt{VQR}} \\
    \midrule
    1           &           3.88 &           3.25 &        1.476 \\
    1           &          4.372 &          4.222 &        1.264 \\
    1           &          4.573 &          3.926 &            - \\
    1           &          3.922 &          3.406 &            - \\
    1           &          3.369 &          2.934 &         2.73 \\
    1           &         24.611 &         21.396 &        8.161 \\
    1           &          35.21 &         27.897 &            - \\
    1           &         72.037 &        217.172 &            - \\
    1           &        183.672 &        440.817 &            - \\
    \bottomrule[1.1pt]
    
    \end{tabular}%

    }
    \end{minipage} 
    
    \captionsetup{format=hang} \caption{Simulated data experiments. The standard errors are given in Appendix~\ref{sec:syn_std}. See Appendix~\ref{sec:syn_datasets_details} for more details about the synthetic data sets.}
    \label{tab:syn_data_results}
    
\end{table}

\renewcommand\imagewidth{0.5}
\begin{figure}[htbp]
\setstretch{1.1}
  \centering

    {\centering{\includegraphics[width=\imagewidth\linewidth]{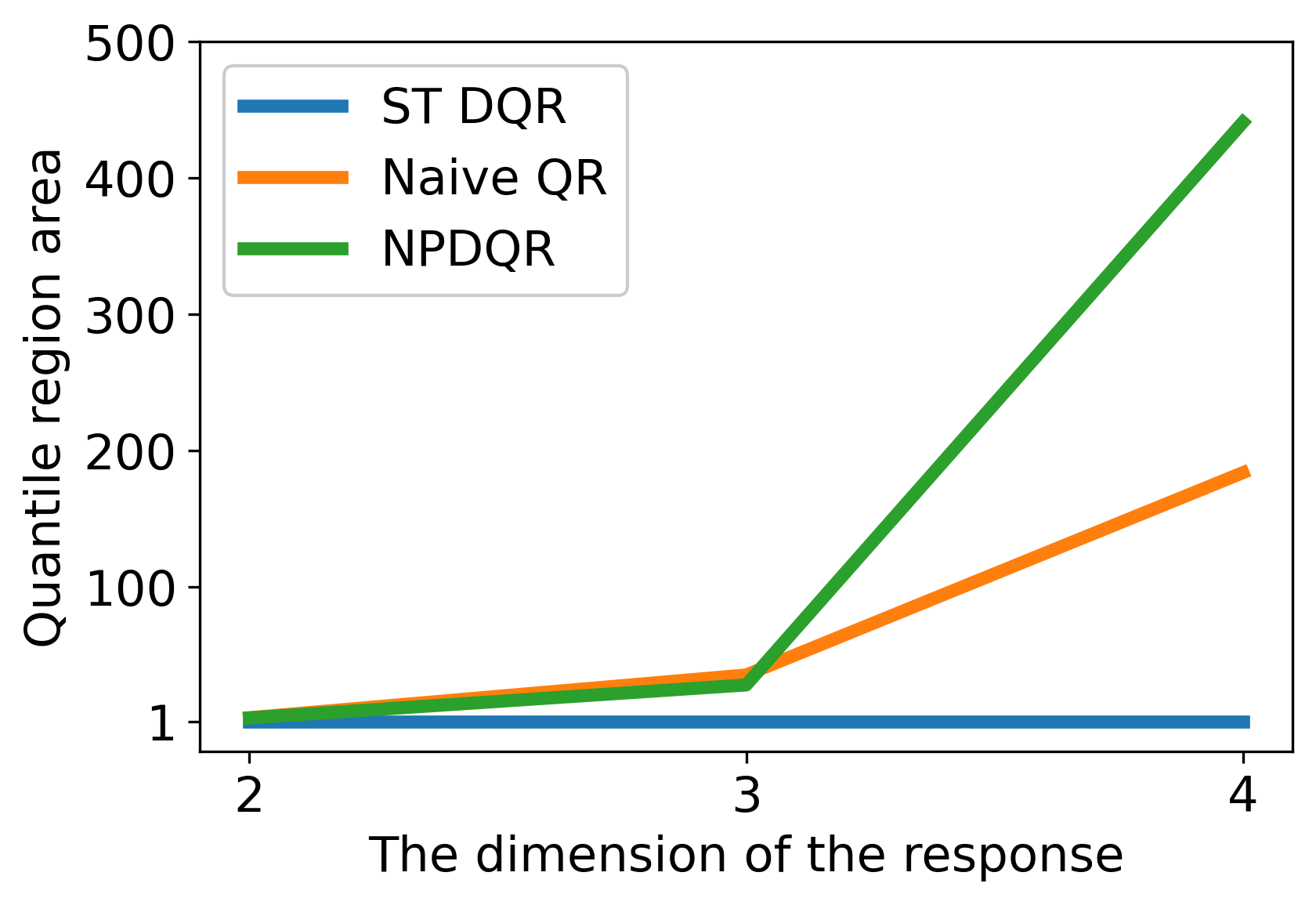}}}

    \captionsetup{format=hang} \caption{Quantile region area vs. the dimension of the response. The area is scaled by the area of the quantile region constructed by our method, averaged on 20 random splits of the data. The data set used is the non-linear synthetic data with $p=10$.}
    
\label{fig:area_vs_d}%
\end{figure}%

\subsection{Real Data Results}\label{sec:real_data_results}
Next, we compare the performance of the proposed \texttt{ST-DQR} method to \texttt{NPDQR}, and \texttt{Na\"ive QR} on six benchmarks data sets as in \citep{CQR, sesia2020comparison}: blog feedback (blog\_data), physicochemical properties of protein tertiary structure (bio), House Sales in King County, USA (house), and medical expenditure panel survey number 19-21 (meps\_19, meps\_20,and meps\_21). We modify each data set to have a 2-dimensional response as described in Appendix~\ref{sec:real_datasets_details}, which also provides additional information about each data set. We follow the experimental protocol and training strategy described in Section~\ref{sec:syn_data_results}. Specifically, we randomly split each data set into disjoint training (38.4\%), calibration (25.6\%), validation (16\%), and testing sets (20\%), and further normalized the feature vector and response variables to have a zero mean and a unit variance each. Due to the under-coverage problem of \texttt{DQR} presented in Section \ref{sec:need_conf}, we estimate a directional quantile of a level higher than the nominal 90\% rate for \texttt{NPDQR} and \texttt{ST-DQR}; see Table \ref{tab:real_coverage_rates_info}. For \texttt{Na\"ive QR} and \texttt{VQR}, we set the quantile level to be equal to the target $1-\alpha=90\%$ rate.

Table~\ref{tab:real_data_results} summarizes the performance metrics, showing that all calibrated methods consistently attain the nominal coverage rate, as guaranteed by Theorem \ref{thm:cov_guarantee}. In addition, the same table indicates that the regions constructed by our method are significantly smaller than the ones produces by the competitive methods. Similar to the synthetic case, \texttt{VQR} is infeasible to deploy and thus omitted from that table. Instead, in Table~\ref{tab:reduced_real_datasets_results} we report the results on a smaller version of the data sets, in which each feature vector is reduced to dimension 10 using PCA so that \texttt{VQR} is feasible. The table shows that even for these modified versions of the data sets, our method outperforms all the others.

We also test the quality of the constructed regions on sub-populations of the data as follows.
We split the test set into three disjoint clusters $c_0, c_1, c_2$, where each contains at least 20\% of the data. The split is done using the K-means algorithm \citep{kmeans}. For illustration purposes, we define the quantile region of a cluster $c$ as
\begin{equation*}
S^{\gamma_\textrm{cal}}(c) = \bigcup\limits_{x \in c} S^{\gamma_\textrm{cal}}(x).
\end{equation*}
Figure \ref{fig:bio_data_results} displays the quantile region constructed by each method for each of the three clusters for the {bio} data set. The figure shows that the regions constructed by our method reflect the distribution of $Y \mid X\in{c}$ better than other methods. Appendix~\ref{sec:additional_real_figs} presents the quantile regions constructed on the {house} data set, leading to a similar conclusion.

\begin{table}[!htb]
    \centering

    \setstretch{1.4}
    \begin{minipage}{.5\linewidth}
        \captionsetup{format=hang} \caption*{Coverage rate}
        \centering
  \scalebox{0.8}{
      \centering

    \begin{tabular}{c|ccc}

    \toprule[1.1pt]
    
    \textbf{Data Set name} & \textbf{\texttt{ST-DQR}} & \textbf{\texttt{Na\"ive QR}} & \textbf{\texttt{NPDQR}}  \\

    \midrule
\textbf{bio}       &                 90.0 &         90.002 &         89.892 \\
\textbf{house}     &               90.157 &         89.978 &         89.876 \\
\textbf{blog\_data} &               90.145 &         90.064 &         90.016 \\
\textbf{meps\_19}   &               90.144 &         89.892 &         89.865 \\
\textbf{meps\_20}   &               89.997 &           90.0 &         89.913 \\
\textbf{meps\_21}   &               89.899 &         89.879 &         89.676 \\

    \bottomrule[1.1pt]
    
    \end{tabular}%
    }
    \end{minipage}%
    \begin{minipage}{.53\linewidth}
      \centering
        \captionsetup{format=hang} \caption*{Relative area of quantile regions}
  \scalebox{0.8}{
      \centering
    \begin{tabular}{ccc}

    \toprule[1.1pt]
    
     \textbf{\texttt{ST-DQR}} & \textbf{\texttt{Na\"ive QR}} & \textbf{\texttt{NPDQR}}  \\

    \midrule
1           &          1.223 &          1.222 \\
1           &          1.168 &          1.149 \\
1           &          1.551 &          1.821 \\
1           &          2.215 &          2.169 \\
1           &           2.27 &          2.209 \\
1           &          2.191 &          2.212 \\

    \bottomrule[1.1pt]
    
    \end{tabular}%

    }
    \end{minipage} 
        \captionsetup{format=hang} \caption{Real data experiments. The standard errors are given in Appendix~\ref{sec:real_std}. See Appendix~\ref{sec:real_datasets_details} for more details about the real data sets.}
            \label{tab:real_data_results}

\end{table}

\begin{table}[!htb]
    \centering

    \setstretch{1.4}
    \begin{minipage}{.5\linewidth}
        \captionsetup{format=hang} \caption*{Coverage rate}
        \centering
  \scalebox{0.8}{
      \centering
    \begin{tabular}{c|cccc}

    \toprule[1.1pt]
    
    \textbf{Data Set name} & \textbf{\texttt{ST-DQR}} & \textbf{\texttt{Na\"ive QR}} & \textbf{\texttt{NPDQR}} & \textbf{\texttt{VQR}} \\

    \midrule
\textbf{bio}       &                 90.0 &         90.002 &         89.892 &        89.87 \\
\textbf{house}     &               89.923 &         90.094 &         90.067 &       90.119 \\
\textbf{blog\_data} &               90.078 &         90.035 &         89.823 &       90.034 \\
\textbf{meps\_19}   &               90.179 &          90.16 &         89.919 &       90.149 \\
\textbf{meps\_20}   &               90.087 &         89.925 &         90.036 &       90.134 \\
\textbf{meps\_21}   &               90.061 &         89.957 &         89.965 &       89.887 \\

    \bottomrule[1.1pt]
    
    \end{tabular}%
    }
    \end{minipage}%
    \begin{minipage}{.6\linewidth}
      \centering
        \captionsetup{format=hang} \caption*{Relative area of quantile regions}
  \scalebox{0.8}{
      \centering
    \begin{tabular}{cccc}

    \toprule[1.1pt]
    
     \textbf{\texttt{ST-DQR}} &  \textbf{\texttt{Na\"ive QR}} & \textbf{\texttt{NPDQR}} & \textbf{\texttt{VQR}} \\

    \midrule
1           &          1.223 &          1.222 &        1.591 \\
1           &          1.501 &          1.425 &        1.359 \\
1           &          1.744 &           1.95 &        1.658 \\
1           &           2.82 &          1.527 &        1.078 \\
1           &          2.761 &          1.488 &          1.1 \\
1           &          2.848 &          1.538 &        1.061 \\
    \bottomrule[1.1pt]
    
    \end{tabular}%

    }
    \end{minipage} 
        \captionsetup{format=hang} \caption{Real data experiments. All feature vectors were reduced to dimension 10 using PCA. The standard errors are given in Appendix~\ref{sec:real_std}.}
            \label{tab:reduced_real_datasets_results}

\end{table}

\renewcommand\imagewidth{0.25}
\renewcommand\texthspace{1.1cm}
\newcommand\coveragetexthspace{1.25cm}

\begin{figure}[htbp]
\setstretch{1.1}
  \centering
  
    \scalebox{1.}{
    \begin{tabular}{cccc}
    \multicolumn{1}{c}{\textbf{Method}} & \multicolumn{3}{c}{\textbf{Quantile region}}\\ 
    
    {} & {\parbox{\imagewidth\linewidth} {\quad \hspace{\texthspace} $c=0$}}  & {\parbox{\imagewidth\linewidth} {\quad \hspace{\texthspace} $c=1$}} &  {\parbox{\imagewidth\linewidth} {\quad \hspace{\texthspace} $c=2$}}   \\ \\
    
    \texttt{Na\"ive QR} & {\centering{\rowincludegraphics[width=\imagewidth\linewidth]{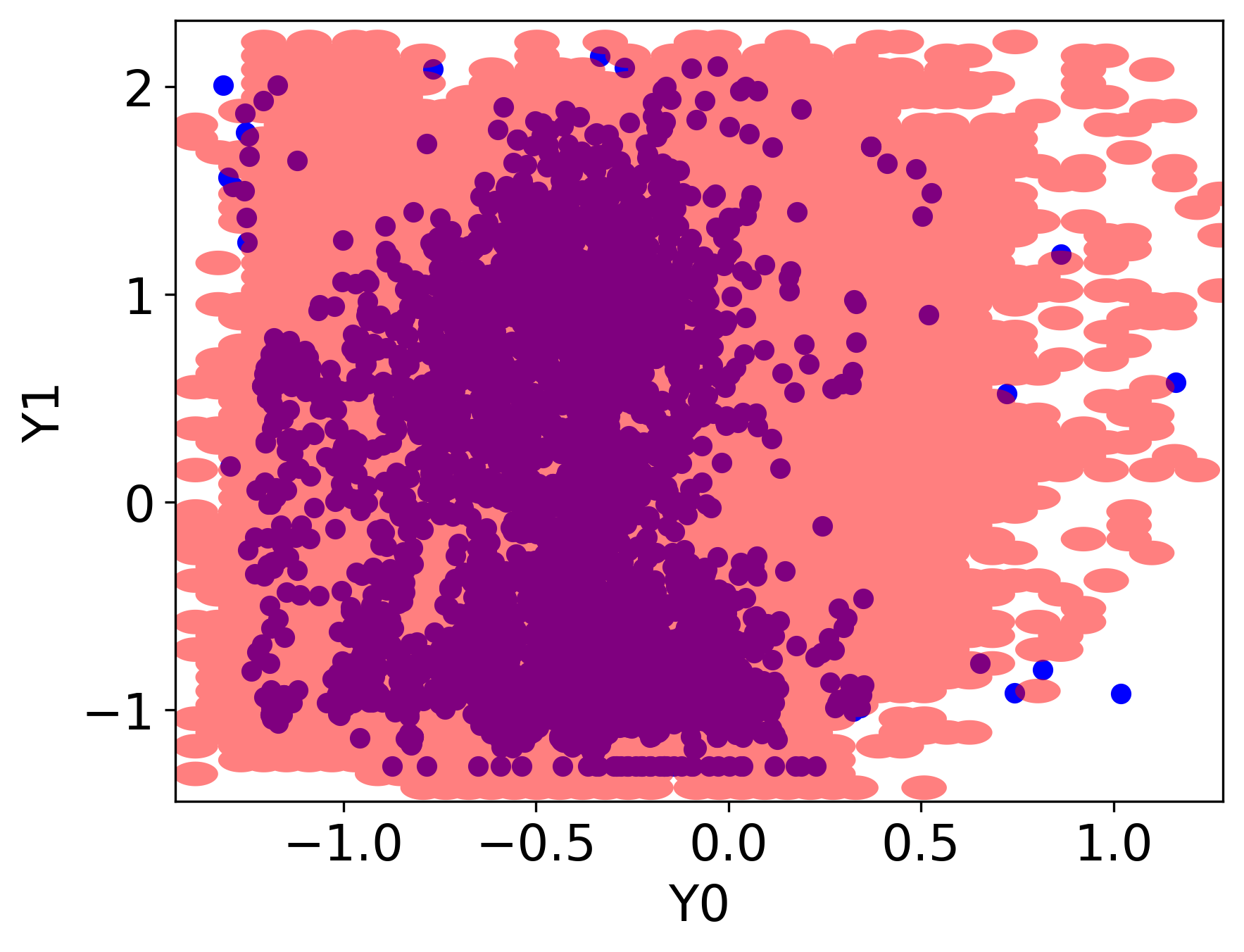}}} & {\centering{\rowincludegraphics[width=\imagewidth\linewidth]{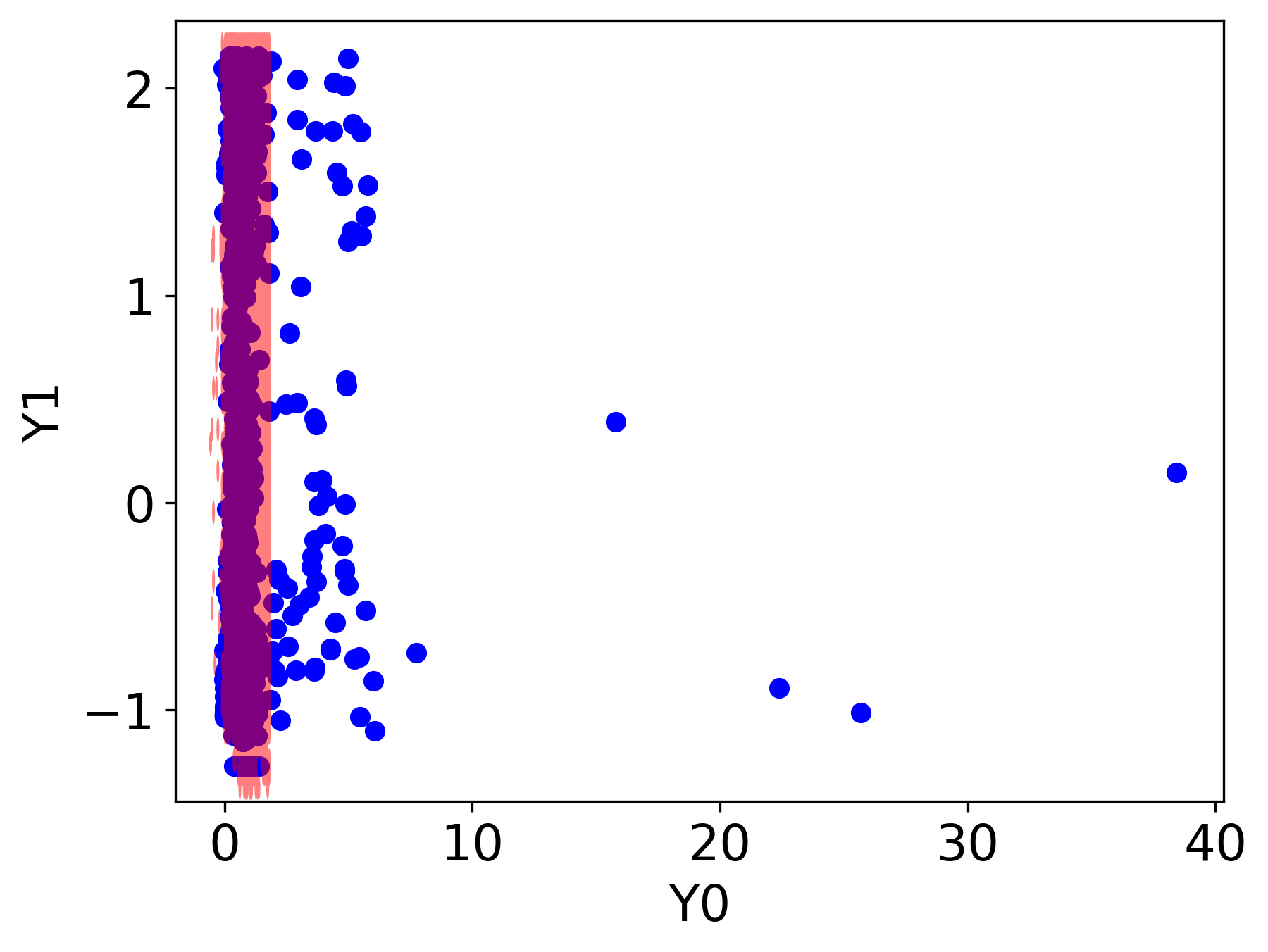}}} & {\centering{\rowincludegraphics[width=\imagewidth\linewidth]{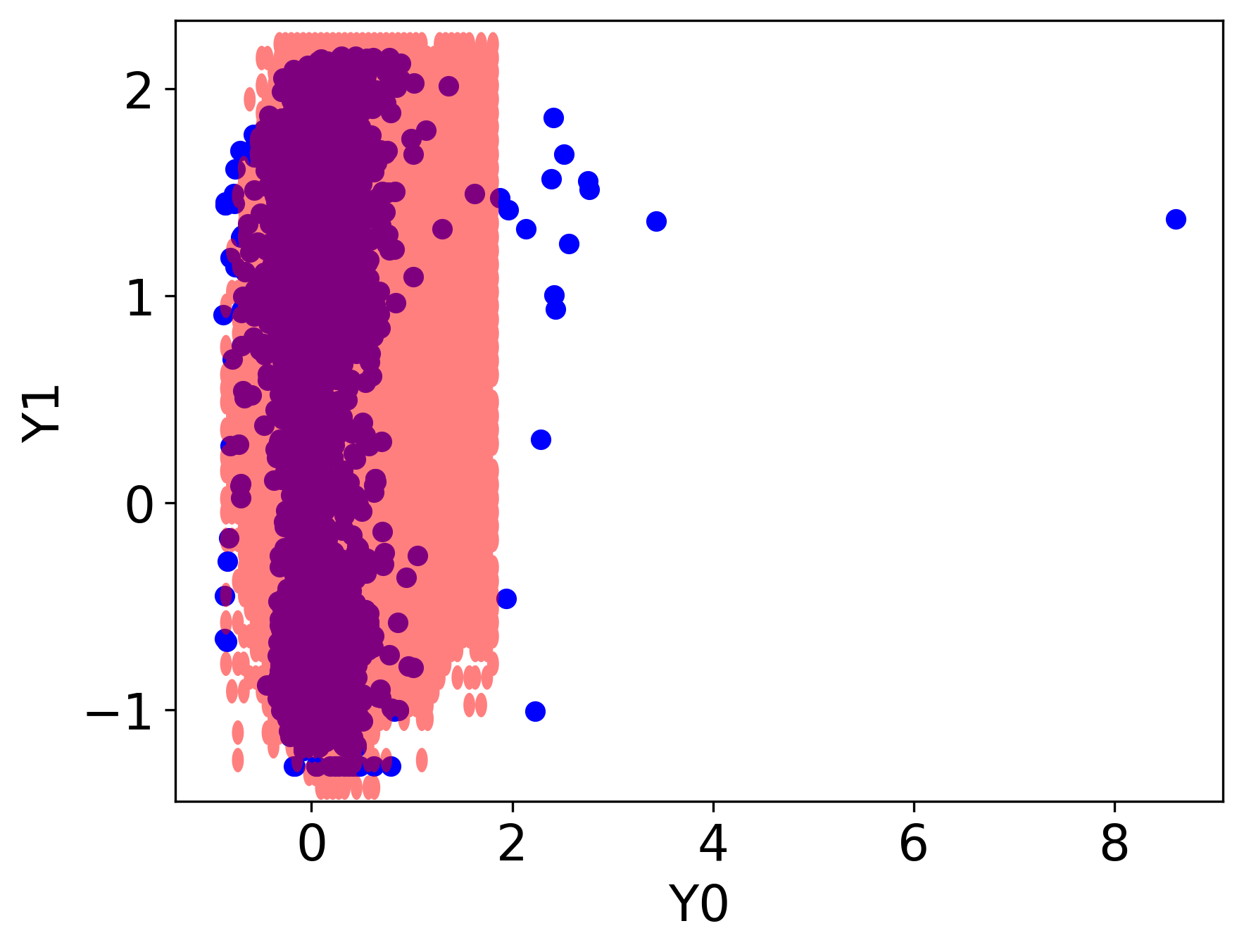}}} \\
            Coverage &{\parbox{\imagewidth\linewidth} {\quad \hspace{\coveragetexthspace} 91.80\% }}& {\parbox{\imagewidth\linewidth} {\quad \hspace{\coveragetexthspace} 91.23\% }} & {\parbox{\imagewidth\linewidth} {\quad \hspace{\coveragetexthspace} 90.65\% }} \\
    \\
    
    \texttt{NPDQR} & {\centering{\rowincludegraphics[width=\imagewidth\linewidth]{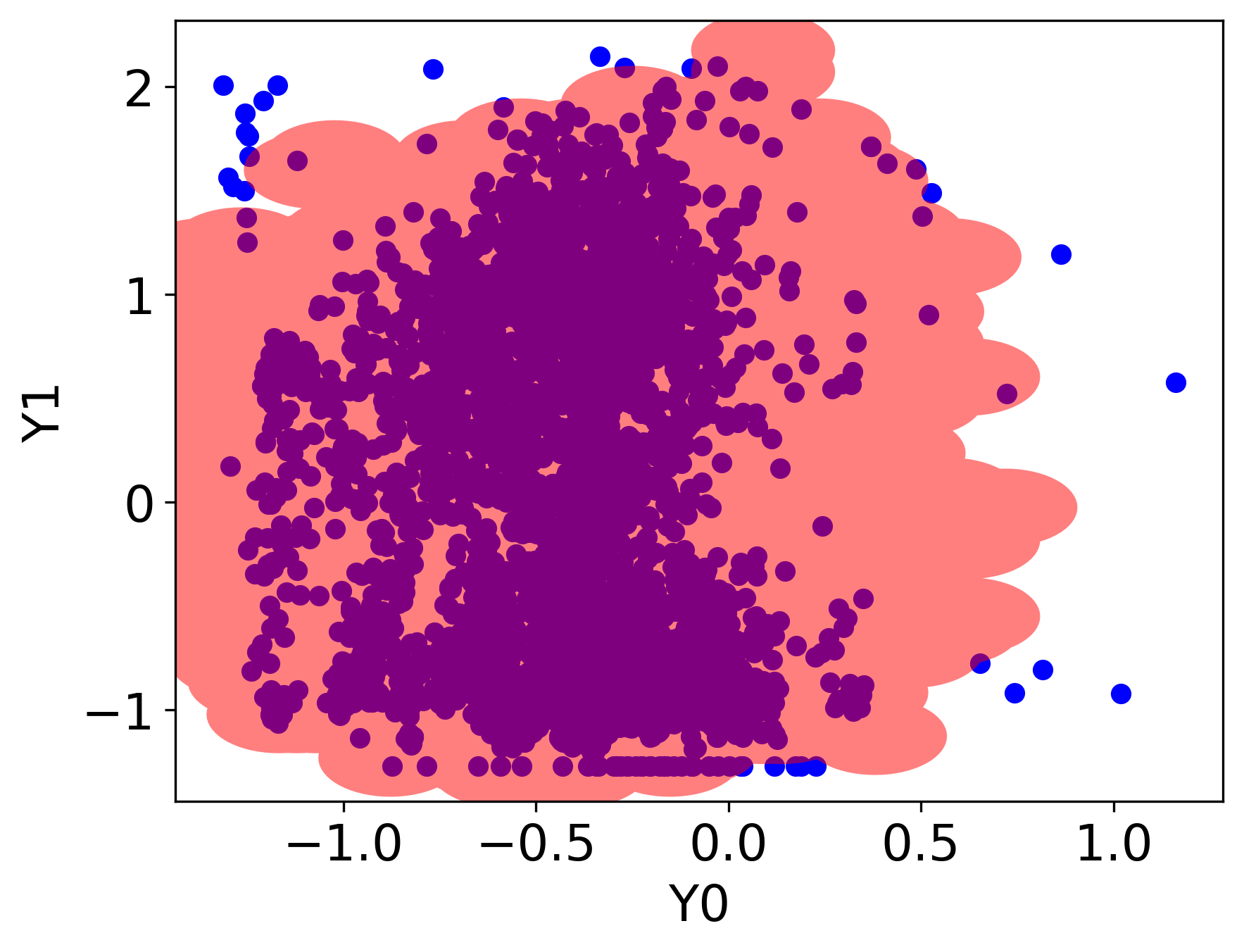}}} & 
    {\centering{\rowincludegraphics[width=\imagewidth\linewidth]{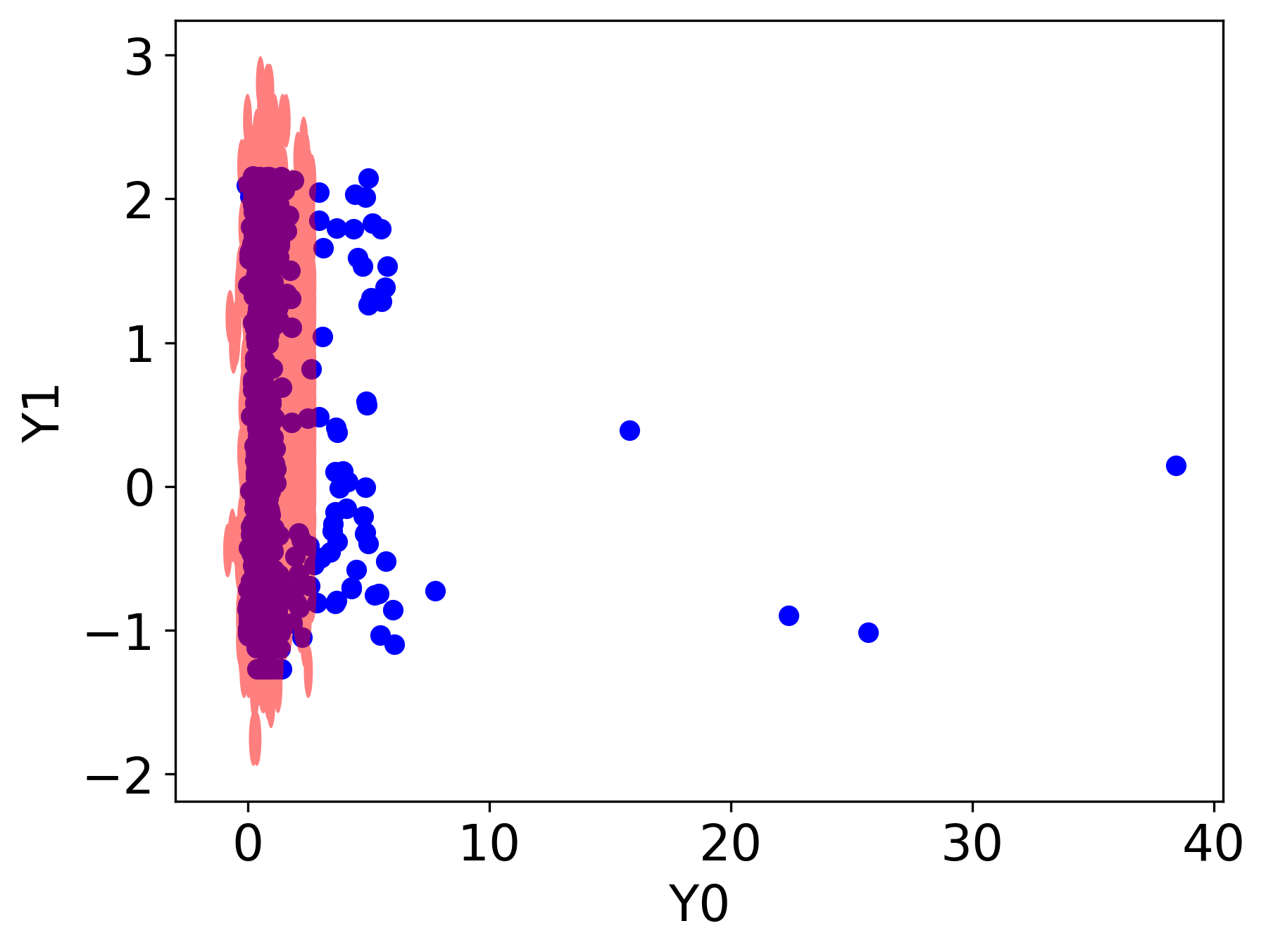}}} & 
    {\centering{\rowincludegraphics[width=\imagewidth\linewidth]{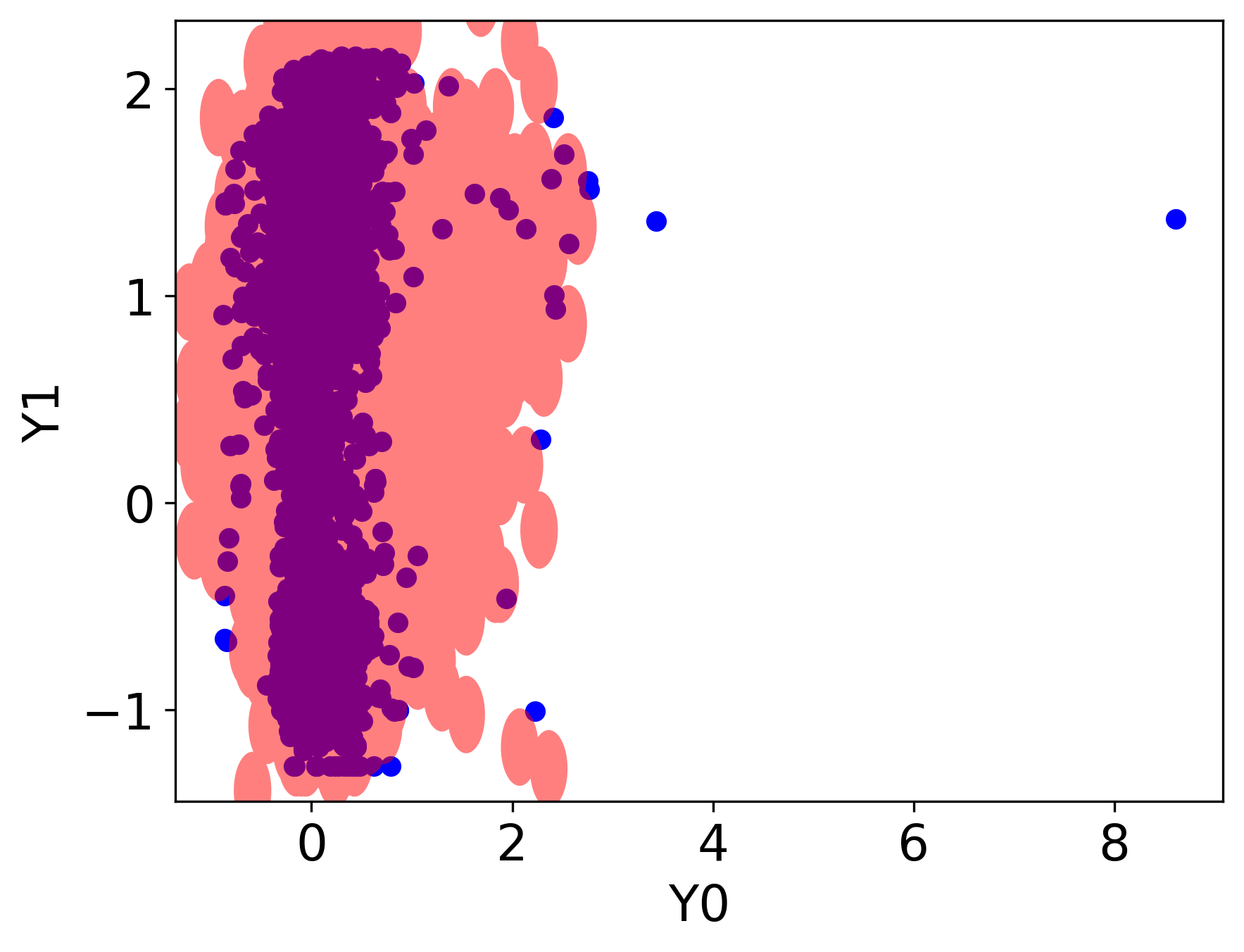}}} \\
            Coverage &{\parbox{\imagewidth\linewidth} {\quad \hspace{\coveragetexthspace} 89.85\% }}& {\parbox{\imagewidth\linewidth} {\quad \hspace{\coveragetexthspace} 89.25\% }} & {\parbox{\imagewidth\linewidth} {\quad \hspace{\coveragetexthspace} 88.70\% }} \\
    \\
    
    \texttt{VQR} & {\centering{\rowincludegraphics[width=\imagewidth\linewidth]{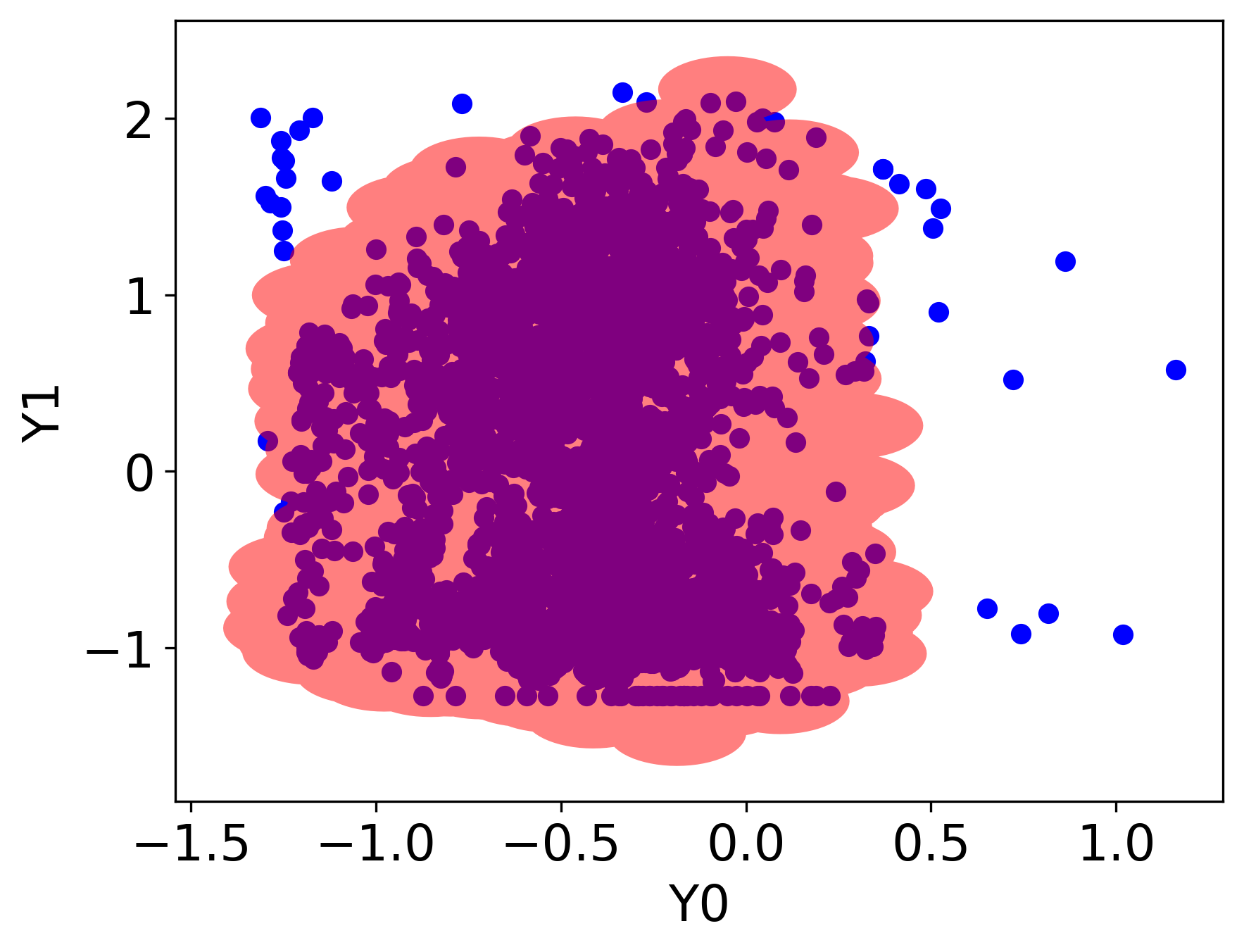}}} &
    {\centering{\rowincludegraphics[width=\imagewidth\linewidth]{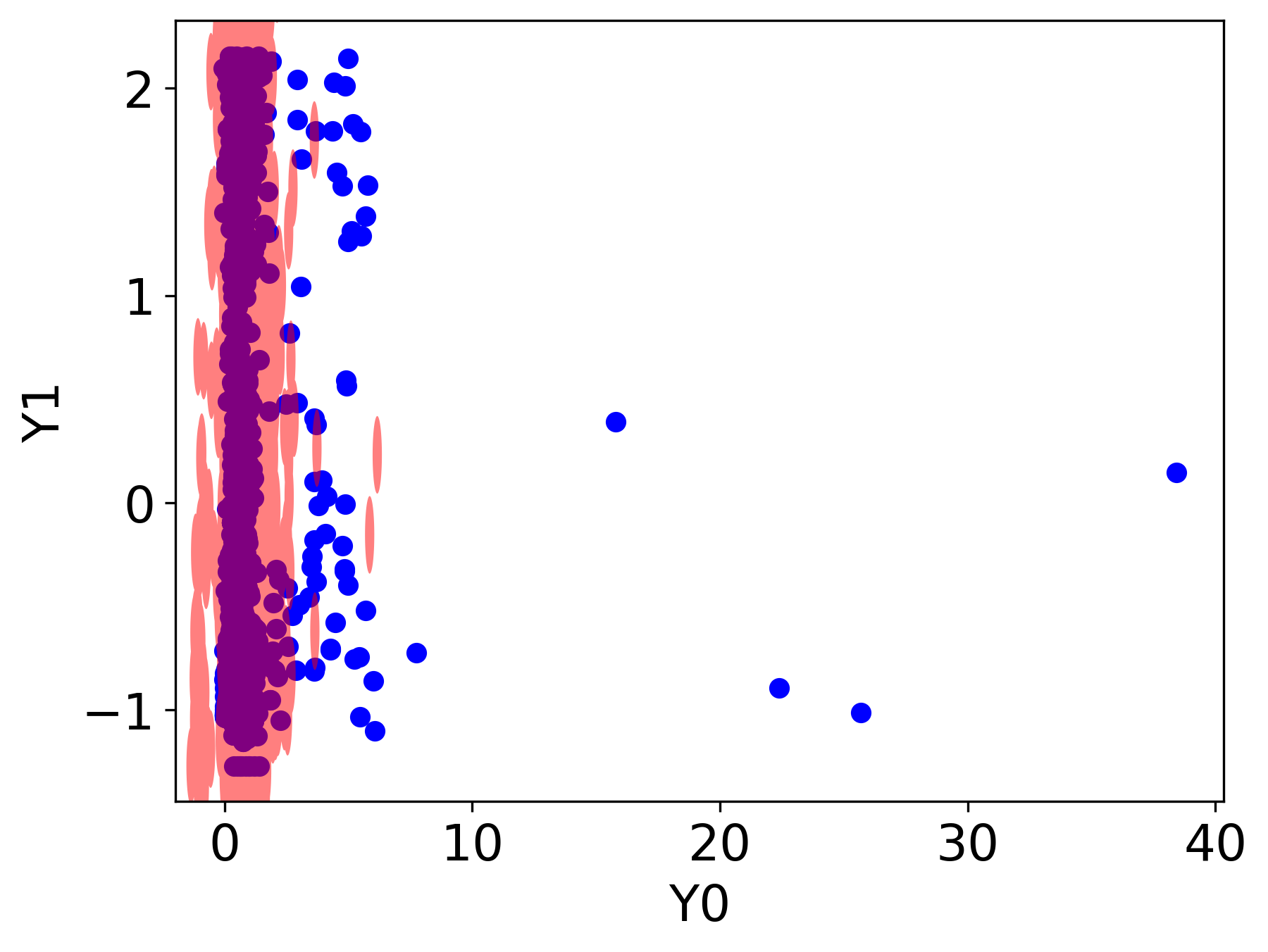}}} & 
    {\centering{\rowincludegraphics[width=\imagewidth\linewidth]{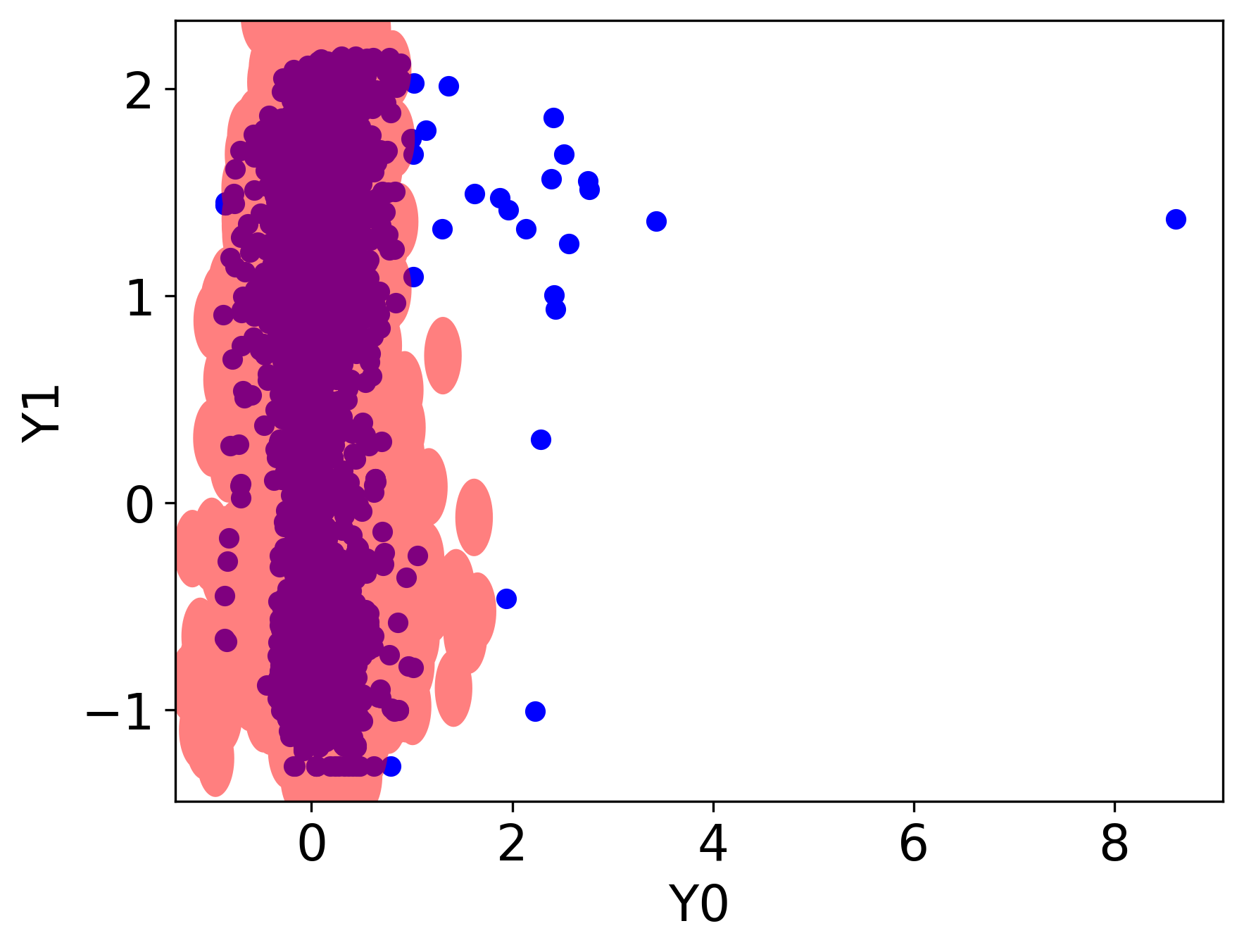}}} \\
            Coverage &{\parbox{\imagewidth\linewidth} {\quad \hspace{\coveragetexthspace} 89.60\% }}& {\parbox{\imagewidth\linewidth} {\quad \hspace{\coveragetexthspace} 88.70\% }} & {\parbox{\imagewidth\linewidth} {\quad \hspace{\coveragetexthspace} 88.84\% }} \\
    \\

    \texttt{ST-DQR} & {\centering{\rowincludegraphics[width=\imagewidth\linewidth]{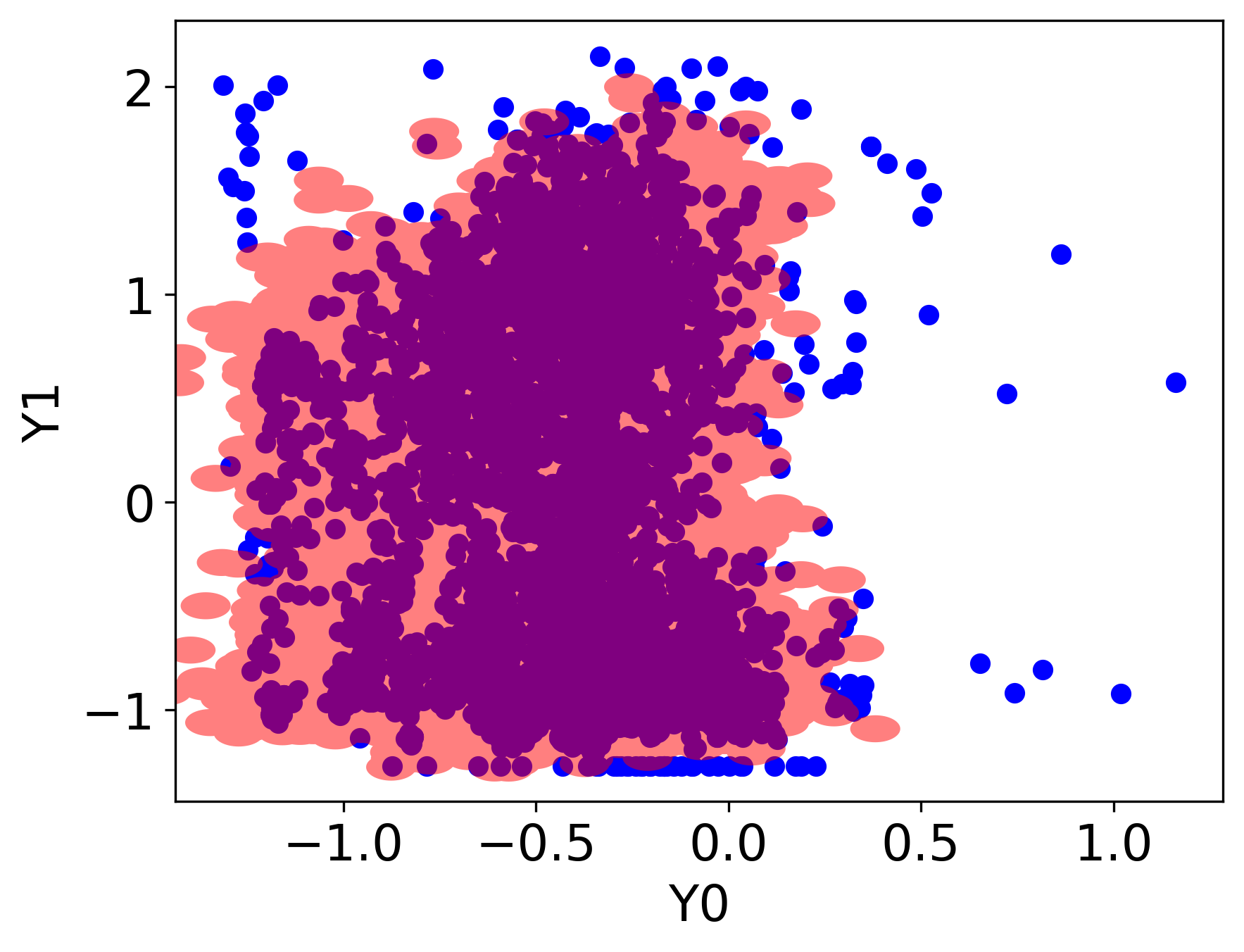}}} &
    {\centering{\rowincludegraphics[width=\imagewidth\linewidth]{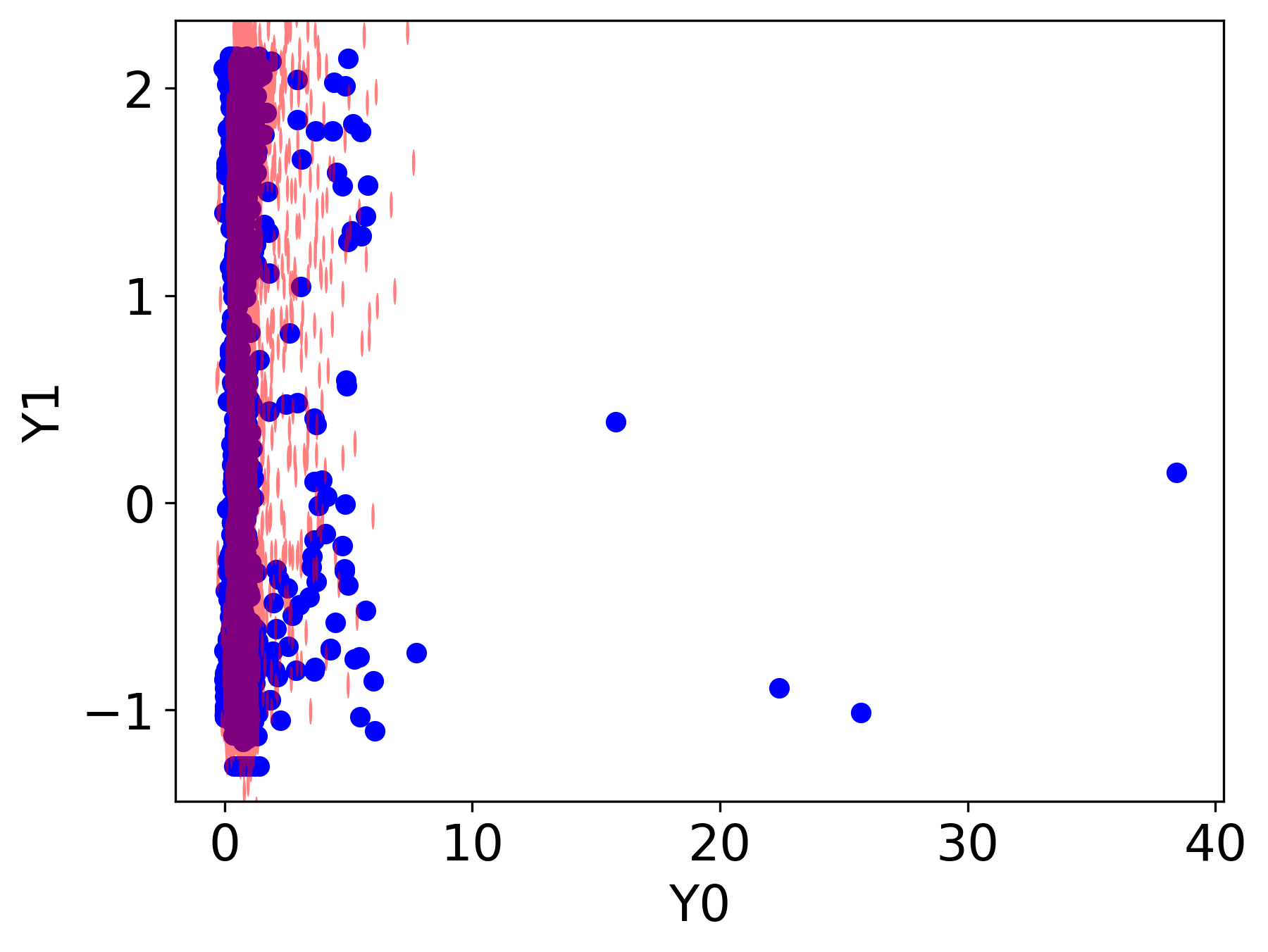}}} & 
    {\centering{\rowincludegraphics[width=\imagewidth\linewidth]{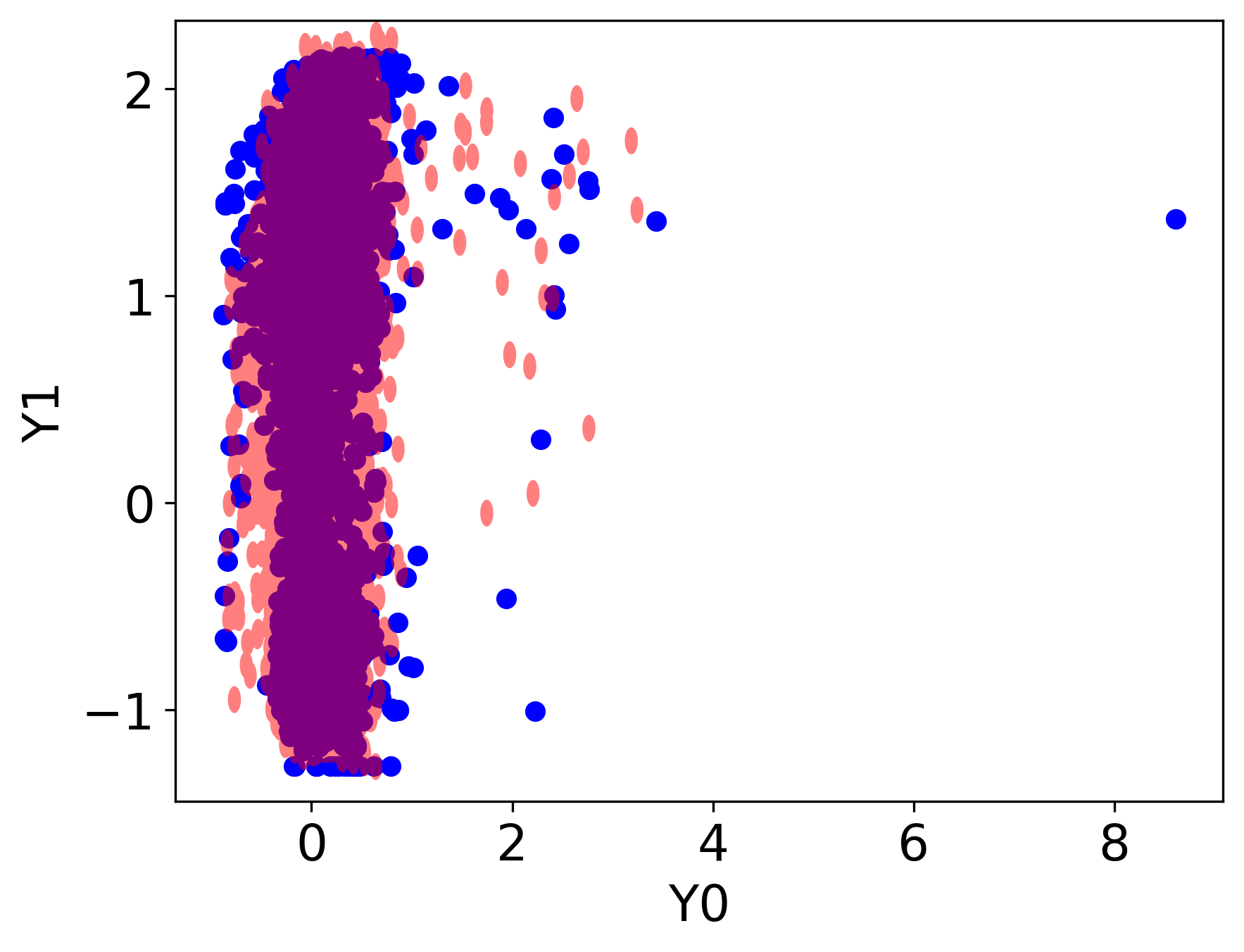}}} \\
            Coverage &{\parbox{\imagewidth\linewidth} {\quad \hspace{\coveragetexthspace} 90.40\% }}& {\parbox{\imagewidth\linewidth} {\quad \hspace{\coveragetexthspace} 90.60\% }} & {\parbox{\imagewidth\linewidth} {\quad \hspace{\coveragetexthspace} 88.78\% }} \\

    \end{tabular}%
    }
    \captionsetup{format=hang} \caption{Quantile regions constructed for Bio data set. The regions were obtained by each of the methods: \texttt{Na\"ive QR}, \texttt{NPDQR}, \texttt{VQR}, and our \texttt{ST-DQR}. In this data set, $Y_0$ and $Y_1$ are two protein structural features.}
    
\label{fig:bio_data_results}
\end{figure}

\section{Conclusion}\label{sec:conclusion}

In this work, we introduced the \texttt{ST-DQR} method to construct non-parametric and flexible quantile regions of an arbitrary shape. We also proposed a modular extension of conformal prediction to the multivariate response setting that guarantees any pre-specified coverage level. Experiments showed that our method generates informative quantile regions for data with response vectors and features of high dimensions.

A promising future direction could be to exploit the property that the response is approximately normally distributed in the latent space $Z_x$, and construct a quantile region for normally distributed data instead of using \texttt{DQR}. Another direction could be to replace the CVAE model (used to transform points from space $\mathcal{Y}$ to space $\mathcal{Z}$ and vice versa) with more recent techniques, such as normalizing flows \citep{kobyzev2020normalizing,rezende2015variational}. Turning the calibration procedure, for very high-dimensional responses, other notions of statistical error beyond marginal coverage---such as the false-negative rate across coordinates---may be more appropriate. Extensions of our procedure to control other error rates would be possible in combination with generalizations of conformal prediction~\citep{bates2021distributionfree, angelopoulos2021learn}, and we view this as an important next step.
Lastly, it would be exciting to explore the conditional coverage of multivariate quantile regression methods, and offer techniques that further improve it, e.g., by generalizing the one proposed by \cite{oqr} to this setting.

\acks{Y.R. and S.F. were supported by the ISRAEL SCIENCE FOUNDATION (grant No. 729/21). Y.R. thanks the Career Advancement Fellowship, Technion, for providing research support. Y.R. and S.F. thank Alex Bronstein, Sai-Sanketh Vedula, and Aviv Rosenberg for insightful discussions.} 

\appendix

\section{Theoretical Results}\label{sec:theo_results}
We now present the proofs of the theoretical results presented in the main manuscript.
\subsection{Coverage Guarantee of Na\"ive Multivariate Quantile Regression}\label{sec:naive_qr_proof}
The na\"ive multivariate quantile regression introduced in Section \ref{sec:naive_qr} achieves the desired coverage level, as proved next.
\begin{equation}
\begin{split}
\mathbb{P}[Y_{n+1}\in R(x)] & = \mathbb{P}\left[\bigwedge\limits_{j=1}^d Y_{n+1}\in C^j(x)\right] \\
&= 1-\mathbb{P}\left[\bigvee\limits_{j=1}^d Y_{n+1}\notin C^j(x)\right]\\ 
& \geq 1- \sum_{j=1}^d \mathbb{P}\left[Y_{n+1}\notin C^j(x)\right] \\
&= 1- \sum_{j=1}^d \alpha/d \\
&= 1-\alpha.
\end{split}
\end{equation}

\subsection{Proof of Theorem 1}\label{sec:thm_1_proof}
\begin{proof}[Proof of Theorem~\ref{thm:qr_in_y}]
We begin by proving that for a fixed $X=x$,
\begin{equation}\label{eq:d_domain_and_codomain}
\mathcal{D}: \mathbb{R}^r \rightarrow \textrm{supp}(Y\mid X=x).
\end{equation}
Assume for the sake of contradiction that Equation~\eqref{eq:d_domain_and_codomain} does not hold. That is, there exists $y\in\textrm{Im}(\mathcal{D})$ such that $y \notin  \textrm{supp}(Y\mid X=x)$. Therefore, there exists $\varepsilon>0$ such that the ball $B := \{ a \in \mathbb{R}^d : \norm{a-y}_2 \leq \varepsilon \}$ satisfies: 
\begin{equation}
B \cap \textrm{supp}(Y\mid X=x) = \emptyset, \quad \mathbb{P}(\mathcal{D}(Z_x;X=x) \in B \mid X=x) > 0.
\end{equation} 
However, under the assumption, $\mathcal{D}(Z_x;X=x)\stackrel{d}{=}Y\mid X=x$, so it follows that: 
\begin{equation}
\mathbb{P}(Y \in B \mid X=x) = \mathbb{P}(\mathcal{D}(Z_x;X=x) \in B \mid X=x) > 0,
\end{equation}
which contradicts $B \cap \textrm{supp}(Y\mid X=x) = \emptyset$. We conclude $ \textrm{Im}(D) \subseteq \textrm{supp}(Y\mid X=x)$.
Finally, since a quantile region in space $\mathcal{Y}$ satisfies $R_{\mathcal{Y}}(x) = \mathcal{D}(R_{\mathcal{Z}}(x);x)$ by construction, we have
\begin{equation}
R_{\mathcal{Y}}(x) \subseteq \textrm{Im}(\mathcal{D}) \subseteq \textrm{supp}(Y\mid X=x).
\end{equation}
\end{proof}

\subsection{Proof of Theorem 2}\label{sec:thm_2_proof}
\begin{proof}[Proof of Theorem \ref{thm:cov_guarantee}] We provide the proof for Case 1, in which $c_\textrm{init} \leq 1-\alpha$. The proof for the complementary case is similar. Recall that the quantile region is defined as:
\begin{equation*}
S^{\gamma_\textrm{cal}}(x) = \left\{ y \in \mathbb{R}^d : \min_{a\in R_\mathcal{Y}(x)}d(a,y) \leq \gamma_\textrm{cal} \right\},
\end{equation*}
where
\begin{equation*}
\gamma_\textrm{cal} = \hat{Q}_{1-\alpha}( \{E^+_i \}_{i \in \mathcal{I}_2}), \quad E^+_i = \min_{a\in  R_\mathcal{Y}(X_i)}d(a,Y_i),
\end{equation*}
and $\hat{Q}_{1-\alpha}( \{E^+_i \}_{i \in \mathcal{I}_2})$ is the $\ceil*{(1-\alpha)(1+ |\mathcal{I}_2 |)}$-th smallest value in $\{E^+_i \}_{i \in \mathcal{I}_2}$.
This implies that:
\begin{equation}\label{eq:conf_score_quantile}
Y_{n+1} \in S^{\gamma_\textrm{cal}}(X_{n+1}) \iff E^+_{^+n+1} \leq \hat{Q}_{1-\alpha}( \{E^+_i \}_{i \in \mathcal{I}_2}).
\end{equation}
Since the conformity scores $\left\{E^+_i \right\}_{i \in \mathcal{I}_2}$ and $E^+_{n+1}$ are exchangeable, the probability of the event in~\eqref{eq:conf_score_quantile} is at least $1-\alpha$.
The remaining technical details for proving this statement follow from \cite{CQR}. The upper bound guarantee of the coverage follows from~\eqref{eq:conf_score_quantile} as well, by applying \citep[][Lemma 2]{CQR}.
\end{proof}

\subsection{\texttt{ST-DQR} Coverage Rate Lower Bound}\label{sec:thm_3_proof}

While the property guaranteed in Theorem \ref{thm:qr_in_y} is highly desired, it does not suffice, as it can be trivially satisfied by an empty region $R_\mathcal{Y}(x)=\emptyset$, which is a subset of $\textrm{supp}(Y\mid X=x)$, hence satisfies the property of Theorem \ref{thm:qr_in_y}. We therefore require the quantile region $R_\mathcal{Y}(x)$ to achieve a high coverage rate as well. For complex distributions of $Y \mid X=x$, the task of constructing a quantile region with a good coverage rate is difficult to achieve via \texttt{NPDQR}. By contrast, the distribution of the response in space $\mathcal{Z}$ is spherical (recall that $Z \sim \mathcal{N}(0,1)^r$) and thus much simpler to handle in the sense that the region constructed by \texttt{NPDQR} in space $\mathcal{Z}$ is likely to achieve a better coverage rate. Therefore, we would like this coverage property to be preserved when transforming the quantile region back to space $\mathcal{Y}$. We now show that our method satisfies this property, i.e., the coverage attained in space $\mathcal{Y}$ is at least as good as the one in space $\mathcal{Z}$. 
\begin{proposition}\label{thm:cov_preservation}
Suppose $\left( \mathcal{E}(y;x), \mathcal{D}(z;x) \right)$ is a CVAE model as in Theorem~\ref{thm:qr_in_y}.
Suppose $R_{\mathcal{Z}}(x)$ is a quantile region in space $\mathcal{Z}$, and $R_{\mathcal{Y}}(x)$ is a quantile region in space $\mathcal{Y}$, as defined in Theorem~\ref{thm:qr_in_y}. Assuming the coverage rate of $R_{\mathcal{Z}}(x)$ is $1-\beta$, then the coverage rate of $R_{\mathcal{Y}}(x)$ is at least $1-\beta$.
\end{proposition}
\begin{proof}
By the assumption of the coverage rate:
\begin{equation}
\mathbb{P}(Z_x \in R_{\mathcal{Z}}(x) \mid X=x) = 1-\beta.
\end{equation}
Since $\mathcal{D}$ is a function, it follows that:
\begin{equation}
Z_x \in R_{\mathcal{Z}}(x) \implies \mathcal{D}(Z_x; X=x) \in \mathcal{D}(R_{\mathcal{Z}}(x); X=x)
\end{equation}
Therefore:
\begin{equation}
\mathbb{P} \left( \mathcal{D}(Z_x;x) \in \mathcal{D}\left(R_{\mathcal{Z}}(x) ;x\right) \mid X=x \right) \geq \mathbb{P}(Z_x \in R_{\mathcal{Z}}(x) \mid X=x)  = 1-\beta.
\end{equation}
Finally, since $Y\mid X=x  \stackrel{d}{=} \mathcal{D}(Z_x; X=x)$, and, $R_{\mathcal{Y}}(x) = \mathcal{D}(R_{\mathcal{Z}}(x) ; x)$, we conclude that:
\begin{equation}
\mathbb{P}(Y \in R_{\mathcal{Y}}(x) \mid X=x) = \mathbb{P} \left( \mathcal{D}\left(Z_x; x\right) \in \mathcal{D}\left(R_{\mathcal{Z}}(x) ; x\right) \mid X=x \right) \geq  1-\beta.
\end{equation}
\end{proof}

\section{Data Sets Details}\label{sec:datasets_details}
This section provides details about the generation of the synthetic data, and about the real data sets used in the experiments.

\subsection{Synthetic Data Details}\label{sec:syn_datasets_details}

The generation of the feature vector and the response variable of the \emph{linear} version of the synthetic data is done in the following way:
\begin{align*}
& \hat{\beta} \sim \textrm{Uniform} (0,1)^{p}, & \\
& \beta = \frac{\hat{\beta}}{\Vert\hat{\beta}\Vert_1}, & \\
& Z\sim \textrm{Uniform} (-\pi,\pi), &\\
& \phi \sim \textrm{Uniform} (0,2\pi), &\\
& R \sim  \textrm{Uniform} (-0.1,0.1), &\\
& X \sim  \textrm{Uniform} (0.8,3.2)^p, &\\
& Y_0 =  \frac{Z}{\beta^T  X} + R \cos(\phi), &\\
& Y_1 =  \frac{1}{2}\left(-\cos(Z) + 1 \right) + R \sin(\phi), &
\end{align*}
where Uniform$(a,b)$ is a uniform distribution on the interval $(a,b)$. The \emph{non-linear} version of the synthetic data is generated in the same way, except for an additional non-linear dependence between $X$ and $Y_1$:
\begin{align*}
& Y_1 =  \frac{1}{2}\left(-\cos(Z) + 1 \right) + R \sin(\phi) + \sin(\frac{1}{n}{\sum_{i=1}^{p}X_i}).
\end{align*}
In the three-dimensional response case, we define the third response variable as:
\begin{equation}
Y_2 =  \sin (\frac{Z}{\beta^TX}),
\end{equation}
and in the four-dimensional response case, the fourth response $Y_3$ is defined as:
\begin{equation}
Y_3 = \cos(\sin (\frac{Z}{\beta^TX})) + R \cos(\phi) \sin(\phi)
\end{equation}
We report the number of samples of each data set in Table \ref{tab:syn_datasets_info}.

\begin{table}[htbp]

\setstretch{1.5}
  \centering
\scalebox{1.}{
\centering
\begin{tabular}{ccc}
    \toprule[1.1pt]
    \textbf{Setting} & $\boldsymbol{p}$ & \textbf{Number of Samples} \\
    \midrule
    
    linear & 1 & 20000 \\
    linear & 10 & 20000 \\
    linear & 50 & 80000 \\
    linear & 100 & 100000 \\

    non-linear & 1 & 20000 \\
    non-linear & 10 & 20000 \\

    \bottomrule[1.1pt]
    \end{tabular}%
}
\captionsetup{format=hang} \caption{Synthetic data sets information. The number of samples as a function of the feature dimension.}
    \label{tab:syn_datasets_info}

\end{table}

\subsection{Real Data Details}\label{sec:real_datasets_details}

The real data sets originally contained a one-dimensional response, so we increase the target dimension by considering one of the features as a response variable instead. The feature concatenated to the response was chosen to be highly correlated to it, and to have a small correlation to the other features, so it will not be easy to predict.
Table~\ref{tab:real_datasets_info} displays the size of each data set, the feature dimension, the response dimension, and the index of the feature that is used as a response instead of an input variable.

\begin{table}[htbp]

\setstretch{1.5}
  \centering
\scalebox{0.8}{
\centering
\begin{tabular}{ccccc}
    \toprule[1.1pt]
    \textbf{Data Set Name} & \textbf{Number of Samples} & $\boldsymbol{p}$ & $\boldsymbol{d}$ & \textbf{Additional Response} \\
    \midrule

    \textbf{\cite{blog_data}}  & 52397 & 279 & 2 & The time between the blog post publication and base-time \\

    \textbf{\cite{bio_data}}  & 45730 & 8 & 2 & F7 - Euclidean distance \\
    
    \textbf{\cite{house_data}}  & 21613 & 17 & 2 & Latitude of a house \\

    \textbf{\cite{meps19_data}}  & 15785 & 138 & 2 & Overall rating of feelings \\

    \textbf{\cite{meps20_data}} & 17541 & 138 & 2 & Overall rating of feelings  \\

    \textbf{\cite{meps21_data} } & 15656 & 138 & 2 & Overall rating of feelings  \\
    
    \bottomrule[1.1pt]
    \end{tabular}%
}
\captionsetup{format=hang} \caption{Information about the real data sets.}
    \label{tab:real_datasets_info}

\end{table}

\section{Experimental Setup}\label{sec:exp_settings}

The network we used receives as an input a vector of size $p+d$ (where $p$ is the feature dimension, and $d$ is the response dimension). The first $p$ variables in the input vector correspond to the elements of the feature vector, and the last variables correspond to the desired quantile level. We split the data sets (both real and synthetic) into a training set (38.4\%), calibration (25.6\%), validation set (16\%) used for early stopping, and a test set (20\%) to evaluate performance.
Then, the feature vectors and the responses were preprocessed using z-score normalization.
The neural network consists of 3 layers of 64 hidden units, and a leaky ReLU activation function with parameter $0.2$. The learning rate used is $1e^{-3}$, the optimizer is Adam \citep{adam}, and the batch size is 256 for all methods. 
The maximum number of epochs is 10000, but the training is stopped early if the validation loss does not improve for 100 epochs, and in this case, the model with the lowest loss is chosen. 
The number of distinct directions used in each gradient step is 32, and they are taken from a fixed collection of 2048 directions that were sampled once, before the training process. The number of directions used to determine the quantile region belonging is 256, and they are sampled from the same collection of directions.
The results are averaged over all seeds in the range between 0 and 19 (inclusive).
We reduce the dimension of the feature vector to 50 in meps data sets, and to 100 in blog data set, using PCA.
The code we use is based on the implementation of \cite{beyond_pinball_loss}.

\subsection{Na\"ive Quantile Regression Setup}\label{sec:naive_setup}

We trained a quantile regression model using the pinball loss to predict two quantile levels for each dimension $d$: $1-\alpha/d$ and $\alpha/d$. The quantiles are then used to construct a prediction interval, according to the explanation in~\ref{sec:quantile_regression}.
The calibration scheme used extends CQR, developed by \cite{CQR}, to the multi-dimensional response case, which we describe in detail next. Let $\alpha_\textrm{lo}=\alpha/2$ and $\alpha_\textrm{hi}=1-\alpha/2$. The CQR conformity scores are defined as:
\begin{equation}
E_{i} = \max \{ \max \{\hat{q}^j_\textrm{lo}(X_i) - Y_i, Y_i - \hat{q}^j_\textrm{hi}(X_i) \} : j \in \{1,..,d\}\},
\end{equation}
where $\hat{q}^j_\textrm{lo}, \hat{q}^j_\textrm{hi}$ are the estimated lower and upper quantiles of the $j$-th dimension of $Y \mid X$, respectively. We define by $Q$ the $(1-\alpha)(\mid \mathcal{I}_2 \mid +1)$-th empirical quantile of $\{E_i \}_{i\in \mathcal{I}_2}$. The calibrated quantile region is given by:
\begin{equation}
\hat{R}(x) = \bigtimes_{j \in \{1,...,d\}} [\hat{q}^j_\textrm{lo}(x) -Q, \hat{q}^j_\textrm{hi}(x) +Q]
\end{equation}

\subsection{Directional Quantile Regression Setup}

As explained in Section \ref{sec:need_conf}, the empirical coverage rate of an uncalibrated \texttt{DQR} model is significantly lower than the nominal level. We therefore require the model to achieve a higher coverage level, according to Tables \ref{tab:syn_coverage_rates_info}, \ref{tab:real_coverage_rates_info}. The quantile region obtained by \texttt{DQR} was then calibrated to achieve 90\% coverage rate, according to Algorithm~\ref{alg:cal}. The level of the estimated directional quantiles is chosen using an independent train/validation/test split, where the value that achieves the highest coverage level is chosen. The examined directional-quantile levels are 90\%, 93\%, and 95\%. For the four-dimensional response data sets, we also examined a directional-quantile level of 98\% for \texttt{NPDQR}.

\subsection{Implementation Details of Our Method}
Table \ref{tab:CVAE_layers} displays the hidden dimension of each layer in the encoder $\mathcal{E}$ and decoder $\mathcal{D}$ of the CVAE. The dimensions were chosen according to the model's performance over the linear synthetic data set, with feature vectors of different dimensions. The networks include a dropout with parameter $0.1$, and a batch-norm layer for blog data set only. The learning rate used to train the CVAE is $1e^{-4}$ for the real data sets and $1e^{-3}$ for the synthetic data sets, and for both the batch size used is $512$. The activation function used is the leaky ReLU function with parameter $0.2$. The maximum number of epochs is 10000, but the training is stopped early if the validation loss does not improve for 200 epochs, and in this case, the model with the lowest loss is chosen. In Tables \ref{tab:syn_coverage_rates_info}, \ref{tab:real_coverage_rates_info} we report the nominal levels of the estimated directional quantiles used to construct the uncalibrated quantile regions. The regions are then calibrated according to Algorithm~\ref{alg:cal}.

\begin{table}[htbp]

\setstretch{1.5}
  \centering
\scalebox{1}{
\centering
\begin{tabular}{cc}
    \toprule[1.1pt]
    $\boldsymbol{p}$ & \textbf{Hidden layers dimension} \\
    \midrule
    
    \textbf{$p \leq 5$} & 32, 64, 128, 256, 128, 64, 32 \\

    \textbf{$5 < p \leq 8$} & 64, 128, 256, 128, 64  \\

    \textbf{$8 < p \leq 10$} & 64, 128, 256, 512, 256, 128, 64  \\
    
    \textbf{$10 < p \leq 25$} & 64, 128, 256, 256, 128, 64  \\
    
    \textbf{$25 < p$} & 128, 256, 512, 512, 256, 128  \\

    \bottomrule[1.1pt]
    \end{tabular}%
}
\captionsetup{format=hang} \caption{Dimension of each hidden layer for the CVAE architecture as a function of $p$.}
    \label{tab:CVAE_layers}
\end{table}

\begin{table}[htbp]

\setstretch{1.5}
  \centering
\scalebox{0.8}{
\centering
\begin{tabular}{ccc|cc}
    \toprule[1.1pt]
    \textbf{Data Set setting} & $\boldsymbol{d}$ &  $\boldsymbol{p}$ & \textbf{\texttt{NPDQR} nominal coverage level} &  \textbf{\texttt{ST-DQR} nominal coverage level} \\
    \midrule
    
    \textbf{linear} & 2 & 10 & 95\% & 95\% \\
    \textbf{linear} & 2 & 50 & 95\% & 95\% \\
    \textbf{linear} & 2 & 100 & 95\% & 95\% \\

    \textbf{non-linear} & 2 & 1 & 95\% & 93\% \\
    \textbf{non-linear} & 2 & 10 & 95\% & 93\% \\
    \textbf{non-linear} & 3 & 1 & 95\% & 93\% \\
    \textbf{non-linear} & 3 & 10 & 95\% & 93\% \\
    \textbf{non-linear} & 4 & 1 & 98\% & 93\% \\
    \textbf{non-linear} & 4 & 10 & 98\% & 95\% \\

    \bottomrule[1.1pt]
    \end{tabular}%
}
\captionsetup{format=hang} \caption{The directional-quantile levels used for a \texttt{NPDQR} model in the synthetic data sets.}
    \label{tab:syn_coverage_rates_info}

\end{table}

\begin{table}[htbp]

\setstretch{1.5}
  \centering
\scalebox{0.8}{
\centering
\begin{tabular}{c|cc}
    \toprule[1.1pt]
    \textbf{Data Set Name} & \textbf{\texttt{NPDQR} nominal coverage level} &  \textbf{\texttt{ST-DQR} nominal coverage level} \\
    \midrule
    
    \textbf{blog\_data} & 95\% & 93\%  \\

    \textbf{bio}& 95\% & 95\% \\
    
    \textbf{house} & 95\% & 95\% \\

    \textbf{meps\_19} &  95\% & 93\% \\

    \textbf{meps\_20} & 95\% & 93\%  \\

    \textbf{meps\_21} & 95\% & 93\%  \\
    
    \bottomrule[1.1pt]
    \end{tabular}%
}
\captionsetup{format=hang} \caption{The directional-quantile levels used for a \texttt{NPDQR} model in the real data sets.}
    \label{tab:real_coverage_rates_info}

\end{table}

\subsection{Machine's Spec}\label{sec:exp_spec}

The resources used for the experiments are:
\begin{itemize}
    \item \textbf{CPU}: Intel(R) Xeon(R) E5-2650 v4.
    \item \textbf{GPU}: Nvidia TITAN-X, 1080TI, 2080TI.
    \item \textbf{OS}: Ubuntu 18.04.

\end{itemize}

\section{Quantile Regions Constructed for House Data Set}\label{sec:additional_real_figs}

Similarly to figures \ref{fig:syn_data_results} and \ref{fig:bio_data_results}, we display in Figure \ref{fig:house_data_results} the constructed quantile regions for the House data set. We split the data into three clusters, as described in \ref{sec:real_data_results}, and color in red the quantile region of each cluster. It is clear from the figure that our method constructs the most informative quantile regions, among all existing methods. The \texttt{VQR} method could not produce quantile regions for this data set since the dimension of the feature vector is too large to handle for the software we used; see Table \ref{tab:vqr_time_memory} in the Appendix.
\renewcommand\imagewidth{0.25}
\renewcommand\texthspace{1.1cm}
\begin{figure}[htbp]
\setstretch{1.1}
  \centering
  
    \scalebox{1.}{
    \begin{tabular}{cccc}
    \multicolumn{1}{c}{\textbf{Method}} & \multicolumn{3}{c}{\textbf{Quantile region}}\\ 
    
    {} & {\parbox{\imagewidth\linewidth} {\quad \hspace{\texthspace} $c=0$}}  & {\parbox{\imagewidth\linewidth} {\quad \hspace{\texthspace} $c=1$}} &  {\parbox{\imagewidth\linewidth} {\quad \hspace{\texthspace} $c=2$}}   \\ \\
    
    \texttt{Na\"ive QR} & {\centering{\rowincludegraphics[width=\imagewidth\linewidth]{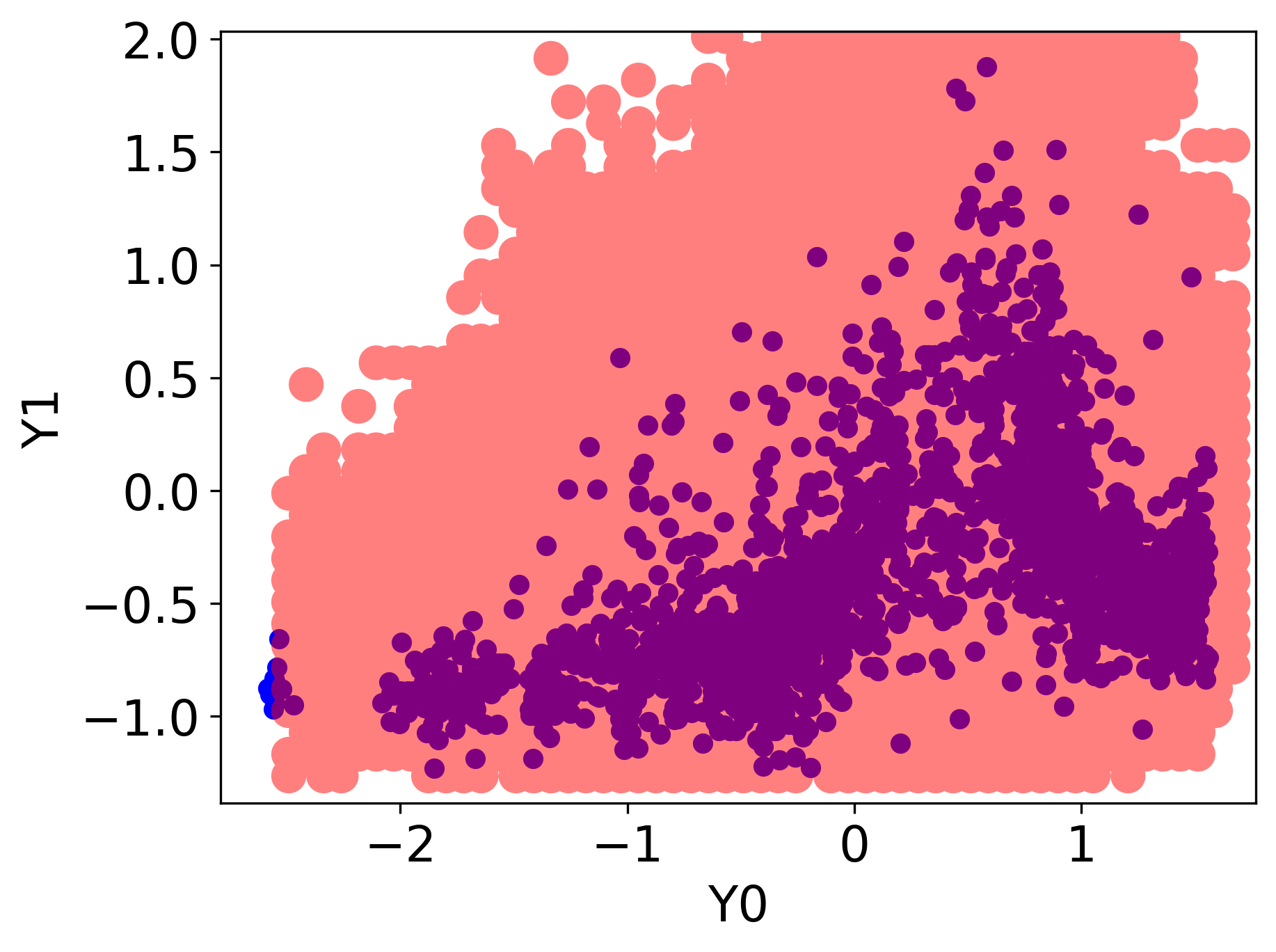}}} & {\centering{\rowincludegraphics[width=\imagewidth\linewidth]{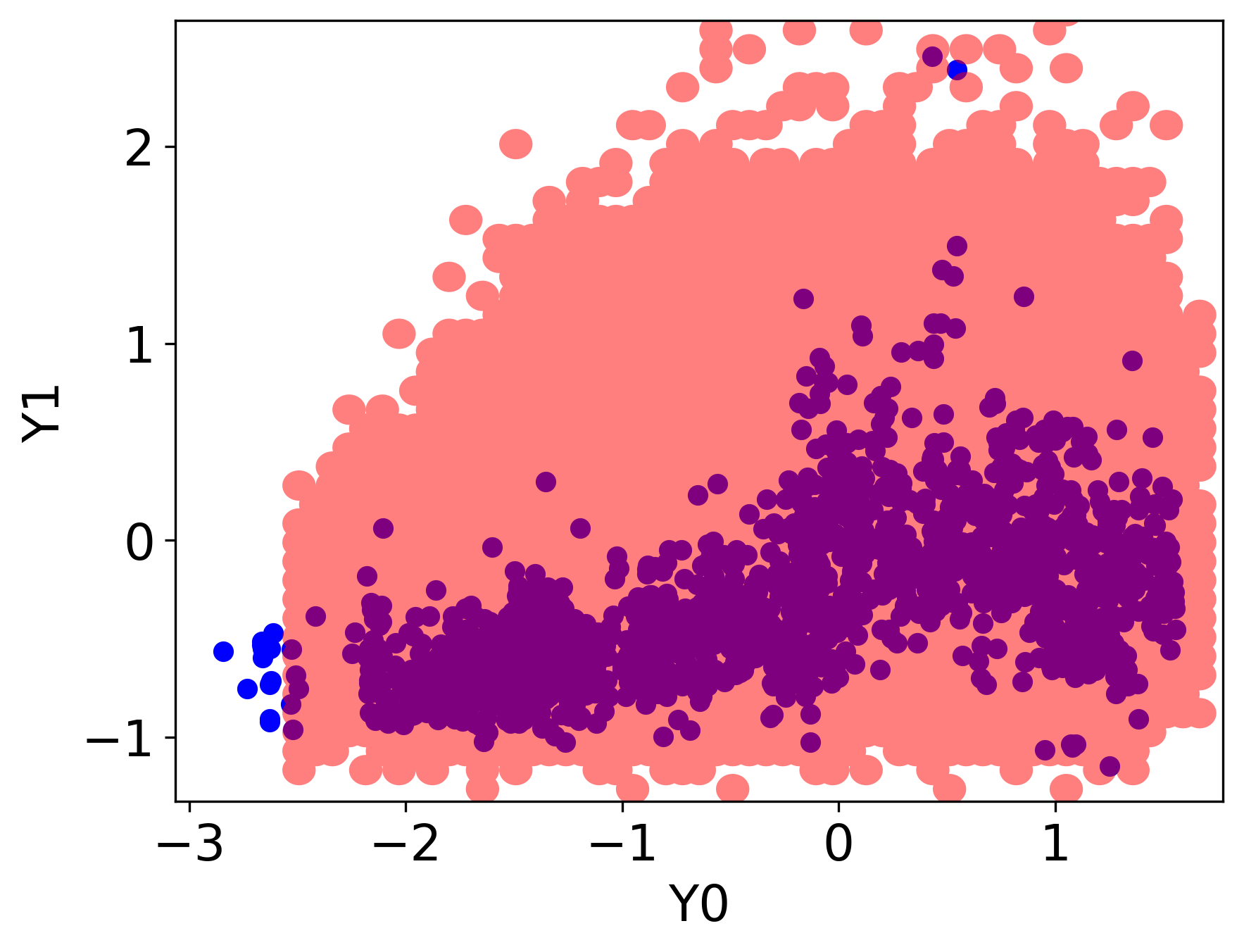}}} & {\centering{\rowincludegraphics[width=\imagewidth\linewidth]{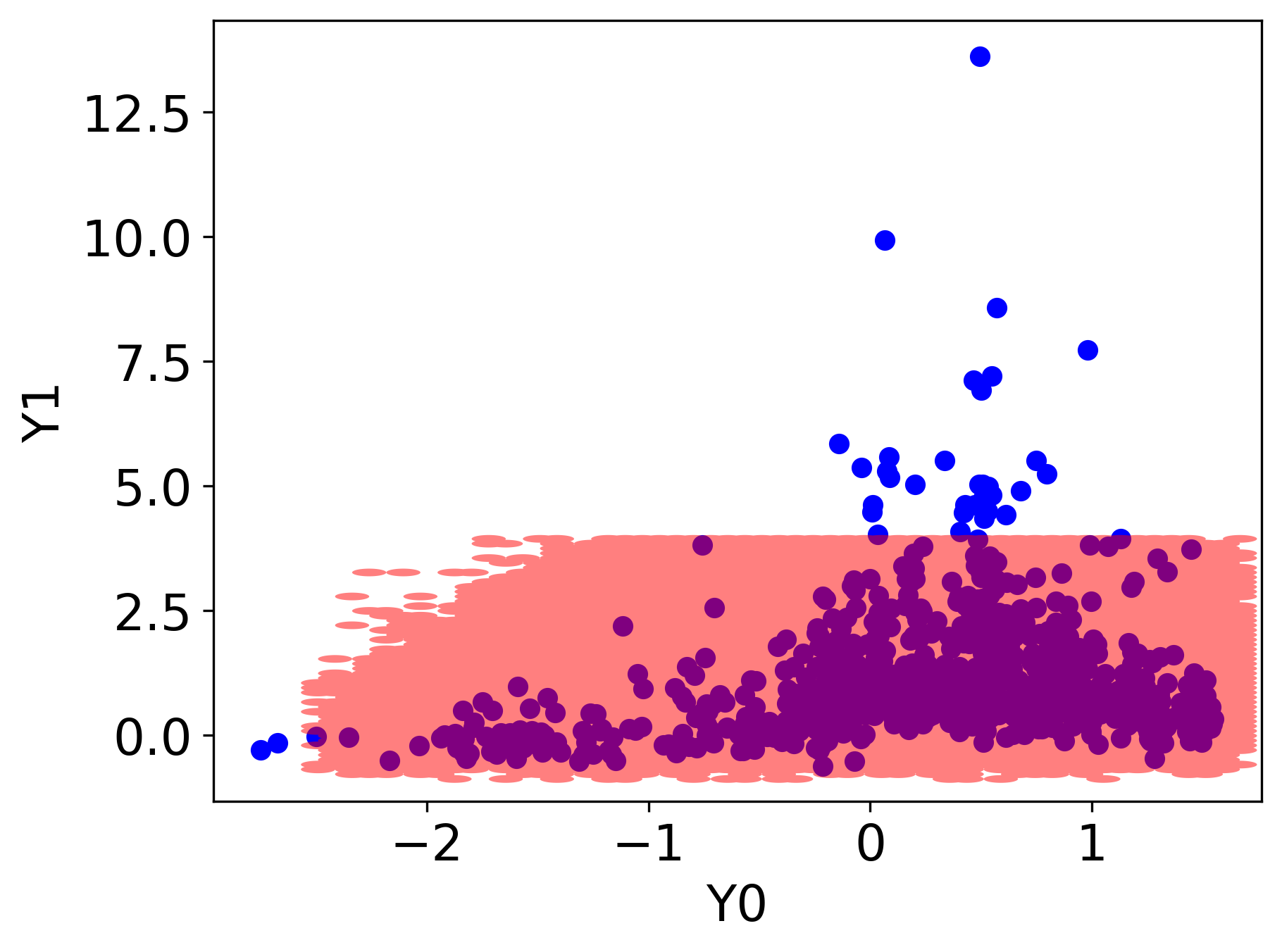}}} \\
    
    \texttt{NPDQR} & {\centering{\rowincludegraphics[width=\imagewidth\linewidth]{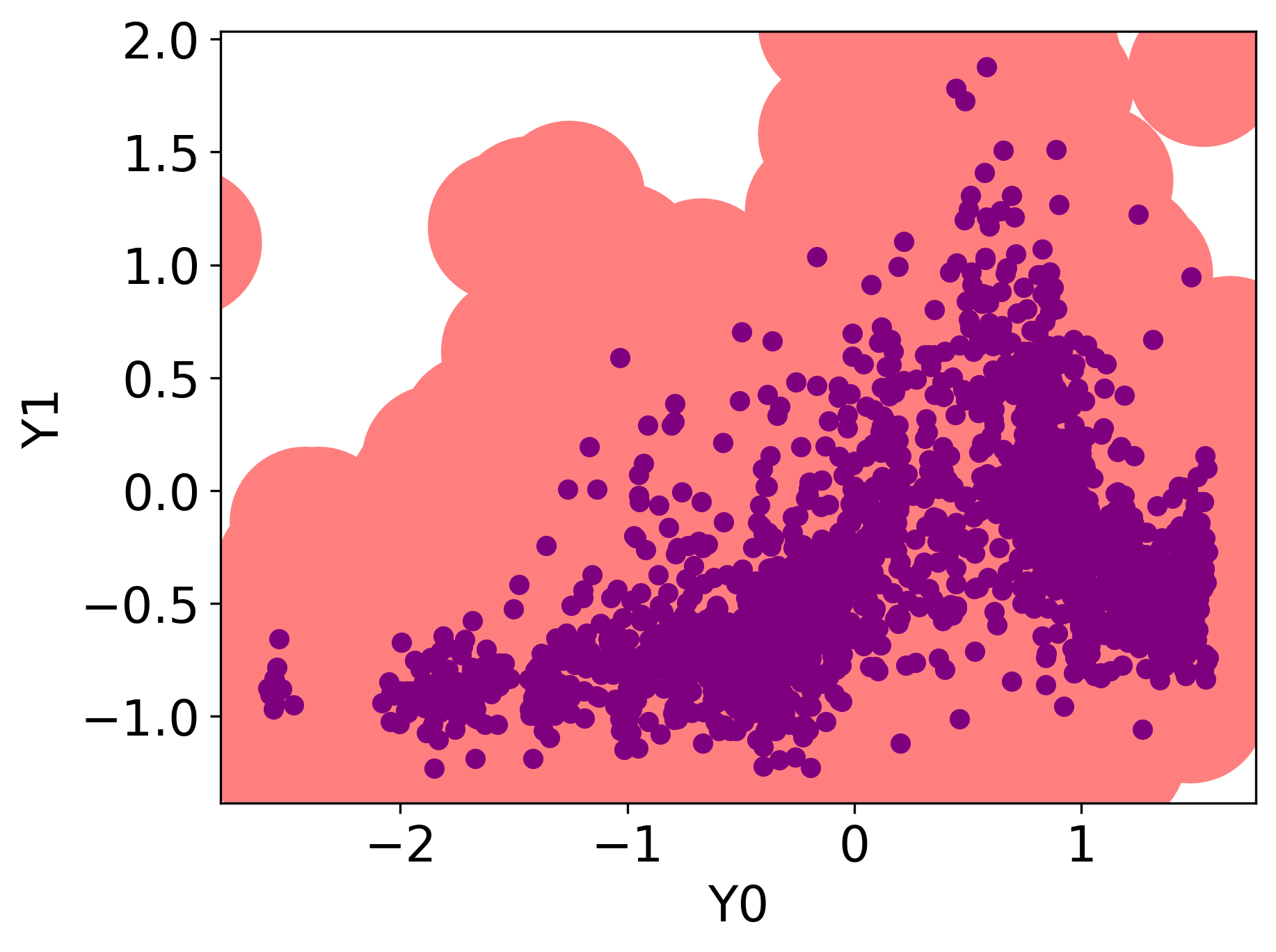}}} & 
    {\centering{\rowincludegraphics[width=\imagewidth\linewidth]{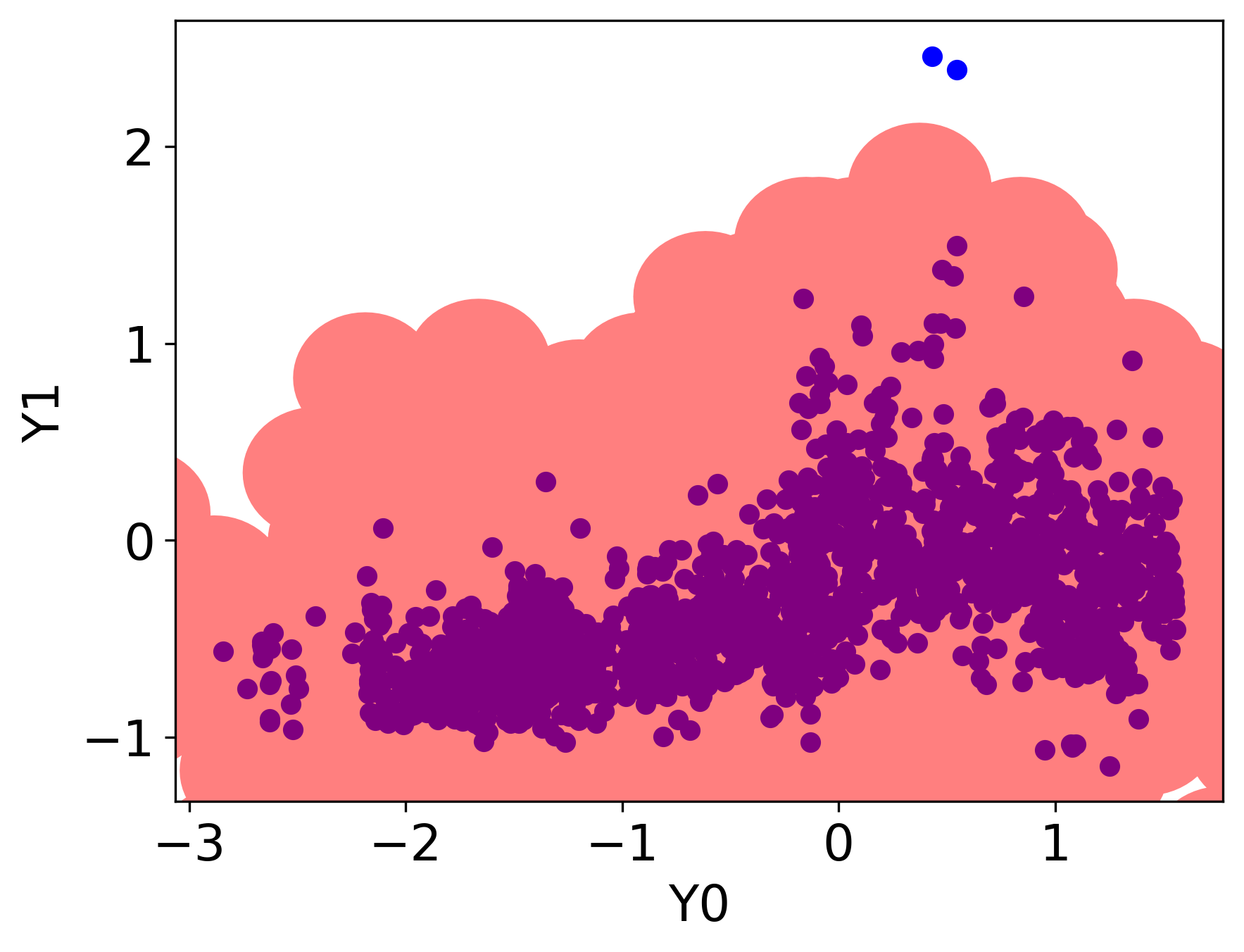}}} & 
    {\centering{\rowincludegraphics[width=\imagewidth\linewidth]{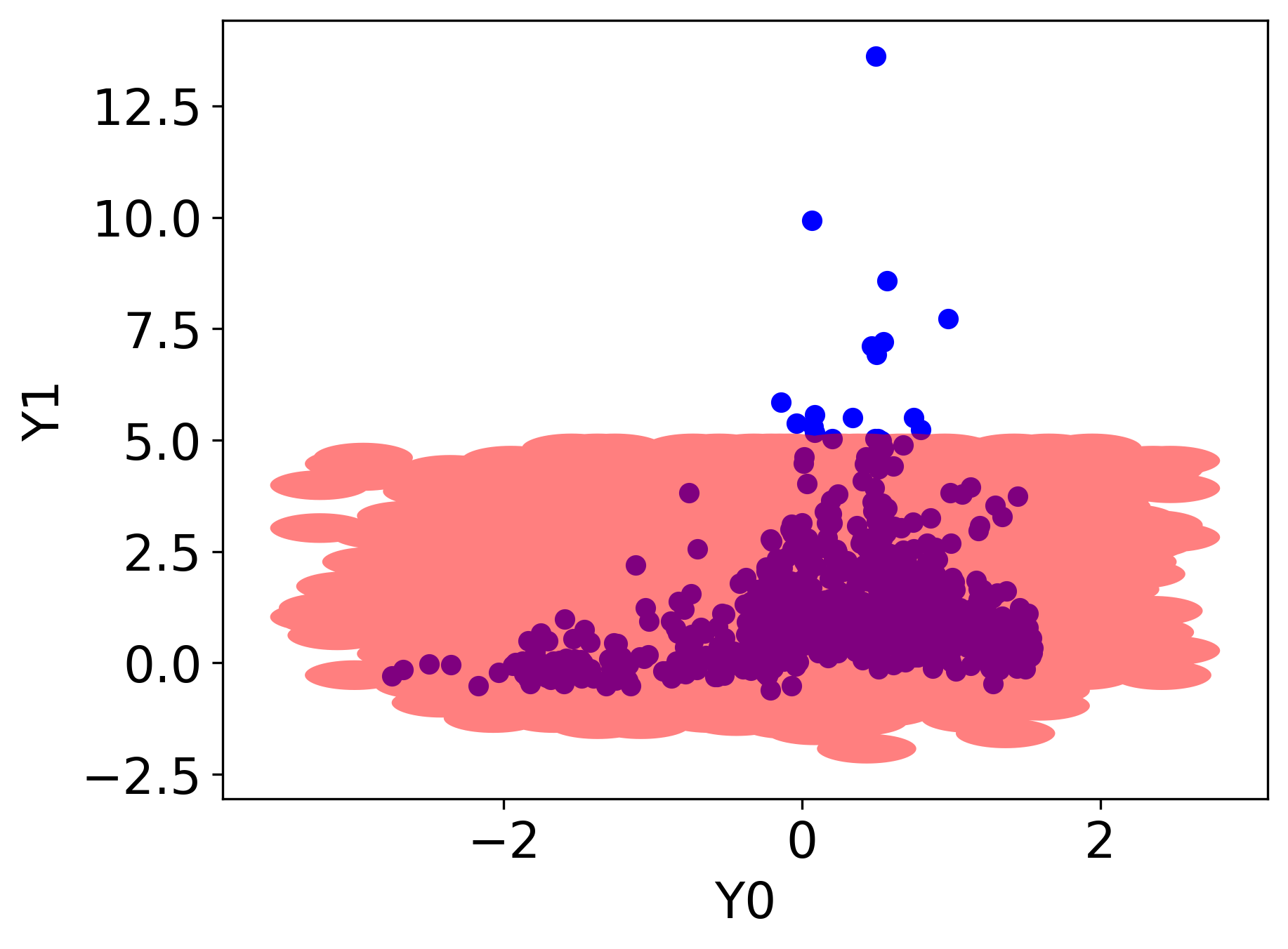}}} \\

    \texttt{ST-DQR} & {\centering{\rowincludegraphics[width=\imagewidth\linewidth]{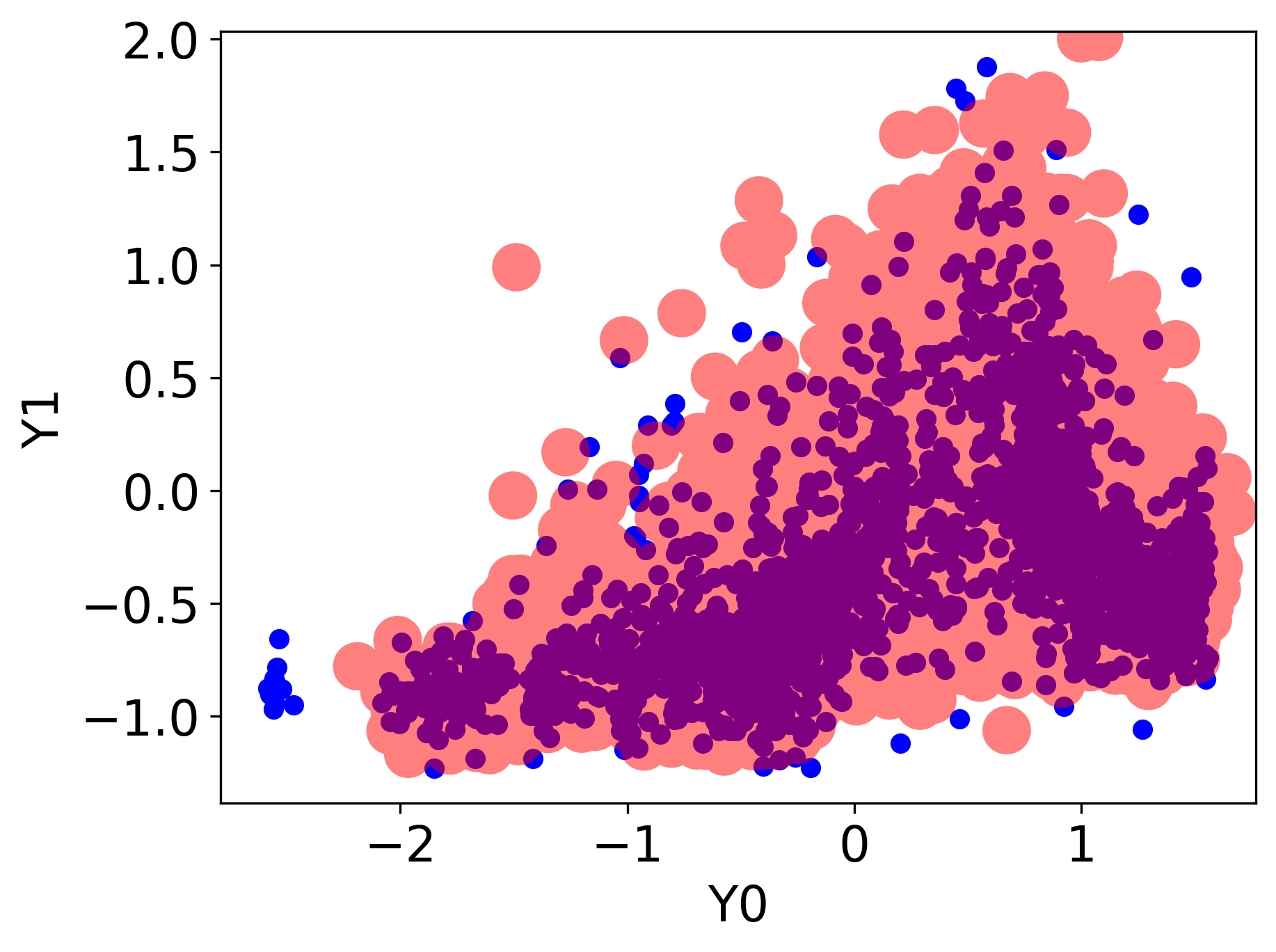}}} &
    {\centering{\rowincludegraphics[width=\imagewidth\linewidth]{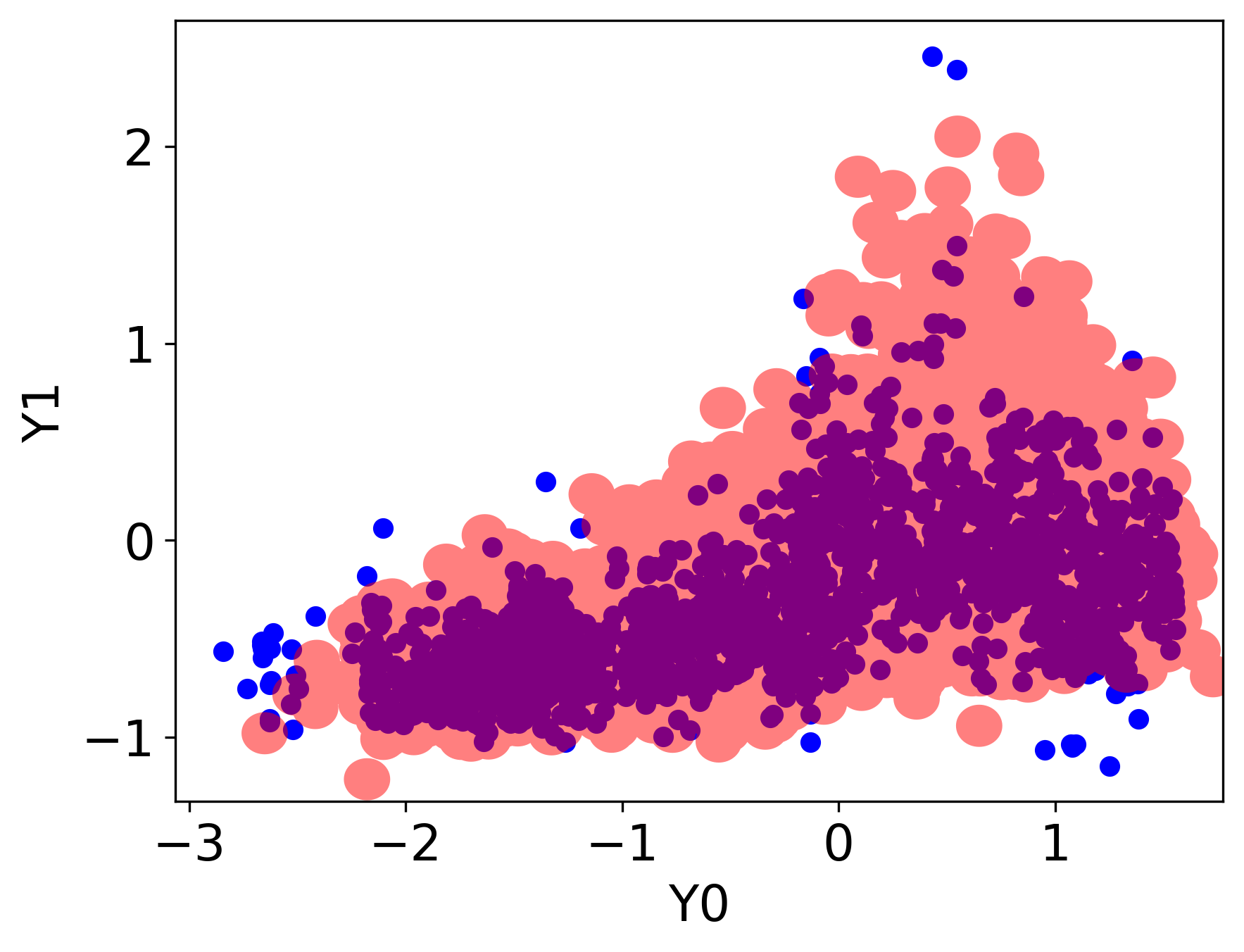}}} & 
    {\centering{\rowincludegraphics[width=\imagewidth\linewidth]{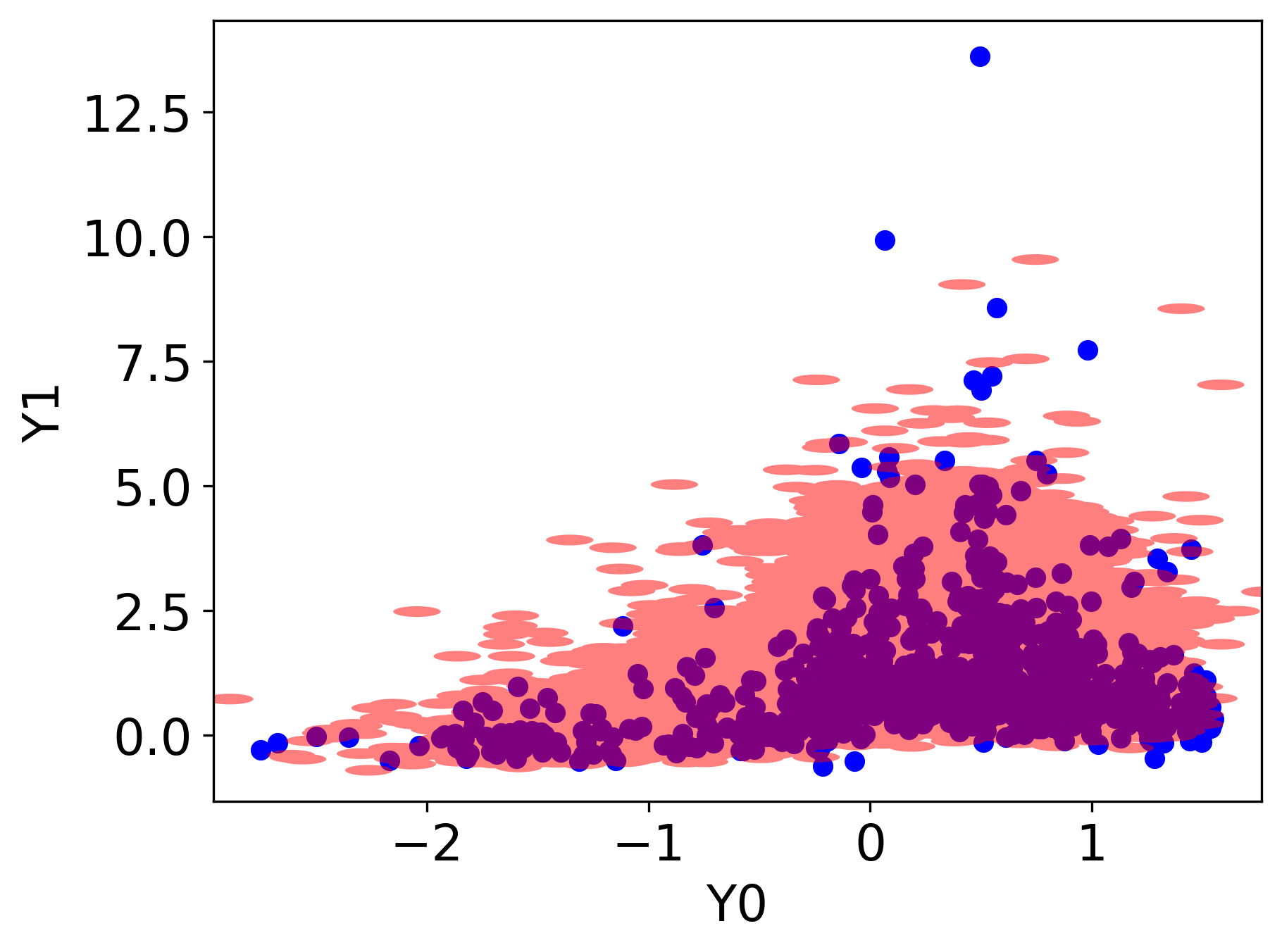}}} \\

    \end{tabular}%
    }
    \captionsetup{format=hang} \caption{House data set results. Quantile region obtained by each of the methods: \texttt{Na\"ive QR}, \texttt{NPDQR}, and our \texttt{ST-DQR}. In this data set, $Y_0$ and $Y_1$ are the price and latitude of a house.}
    
\label{fig:house_data_results}
\end{figure}

\section{Additional Experiments}
In this section we provide additional experiments, analyzing the conditional coverage, the effect of the calibration set, and more, using the methods discussed in this work. The experimental setup is identical to the one described in Section \ref{sec:experiments}, unless explicitly stated otherwise.

\subsection{Calibrating on the Training Set}

In this section, we display the performance of each method calibrated on the training set instead of the calibration set. That is, here, all methods were calibrated to achieve $1-\alpha=90\%$ coverage rate on the training set. 
A similar experiment examining the effect of calibrating with a training set instead of a calibration set was previously suggested by \cite{barber2021predictive}, showing that naively calibrated intervals do not attain the right coverage level.
The results given below indicate that a model calibrated using the training data does not achieve the desired coverage level, emphasizing the necessity of a calibration set.

Table \ref{tab:uncal_real_data_coverage_area_std} presents the coverage rates and the areas of each method on the real data sets. This table shows that the coverage attained by these methods is far from the nominal level. Furthermore, this table reveals that our method achieves the best marginal coverage.

\begin{table}[!htb]
    \centering
    \setstretch{1.4}
    \begin{minipage}{.5\linewidth}
        \captionsetup{format=hang} \caption*{Coverage rate}
        \centering
  \scalebox{0.75}{
      \centering

    \begin{tabular}{c|ccc}

    \toprule[1.1pt]
    
    \textbf{Data Set name} & \textbf{\texttt{ST-DQR}} & \textbf{\texttt{Na\"ive QR}} & \textbf{\texttt{NPDQR}}  \\

    \midrule
\textbf{bio}       &       88.51 (.101) &  86.902 (.147) &  85.907 (.173) \\
\textbf{house}     &       88.29 (.164) &  82.228 (.348) &   75.555 (.28) \\
\textbf{blog\_data} &      87.675 (.159) &  82.802 (.186) &     78.9 (.32) \\
\textbf{meps\_19}   &      88.028 (.171) &  83.432 (.231) &  82.653 (.871) \\
\textbf{meps\_20}   &       88.28 (.148) &  84.595 (.179) &  81.536 (.946) \\
\textbf{meps\_21}   &      87.664 (.207) &  83.456 (.278) &  82.996 (.943) \\
    \bottomrule[1.1pt]
    
    \end{tabular}%
    }
    \end{minipage}%
    \begin{minipage}{.5\linewidth}
      \centering
        \captionsetup{format=hang} \caption*{Area of quantile regions} 
  \scalebox{0.75}{
      \centering
    \begin{tabular}{ccc}

    \toprule[1.1pt]
    
    \textbf{\texttt{ST-DQR}} & \textbf{\texttt{Na\"ive QR}} & \textbf{\texttt{NPDQR}}  \\

    \midrule
306.77 (3.869)  &   388.469 (4.699) &   325.729 (5.141) \\
338.502 (6.657) &   347.108 (7.483) &   259.874 (4.539) \\
95.812 (7.953)  &    184.51 (4.369) &   146.073 (5.806) \\
160.33 (6.21)   &  399.095 (22.885) &    320.96 (16.72) \\
162.464 (3.133) &  419.837 (13.336) &   311.72 (17.832) \\
162.164 (4.409) &  409.609 (17.192) &  344.641 (16.135) \\
    \bottomrule[1.1pt]
    
    \end{tabular}%

    }
    \end{minipage} 
    \captionsetup{format=hang} \caption{Real data experiments. Coverage, area, and their standard error of the quantile regions constructed by each method calibrated on the training set.}
    \label{tab:uncal_real_data_coverage_area_std}

\end{table}

\subsection{Evaluating Conditional Coverage}
In this section, we analyze the conditional validity of the proposed method and compare it to the existing ones.

\subsubsection{Synthetic Data Sets}

Figure \ref{fig:syn_cov_vs_x} presents the coverage as a function of the feature vector on the non-linear synthetic data with $d=2$ response values and $n=10000$ samples. We generated the feature vector to have the value of $x$ repeated $p=10$ times. This figure indicates that our \texttt{ST-DQR} achieves the best conditional coverage compared to the competitors. 
Table \ref{fig:syn_cov_vs_n} shows the coverage as a function of the sample size on the synthetic data described above for different values of $x$. This figure reveals that as the sample size increases, the conditional coverage of all methods gets closer to the nominal level. Additionally, while \texttt{ST-DQR} attains poor conditional coverage for small samples, it performs better than existing methods for larger samples. Furthermore, it constructs quantile regions that represent better underlying uncertainty compared to competitors, as discussed next.
Figures \ref{fig:syn_changing_n_data_results} display the quantile regions constructed by \texttt{ST-DQR} on the non-linear synthetic data described above with different sample sizes. This figure shows that even with a small data set, the regions reflect well the conditional uncertainty of the response. This stands in striking contrast to the regions constructed by \texttt{NPDQR}, visualized in Figure \ref{fig:npdqr_syn_changing_n_data_results} In addition, the regions seem to converge to the true conditional density of $Y\mid X=x$ as the sample size increases.

The experiments conducted in this section indicate that our \texttt{ST-DQR} achieves good conditional coverage and performs better as the sample size increases.
\begin{figure}[htbp]
\setstretch{1.1}
  \centering

    {\centering{\includegraphics[width=0.5\linewidth]{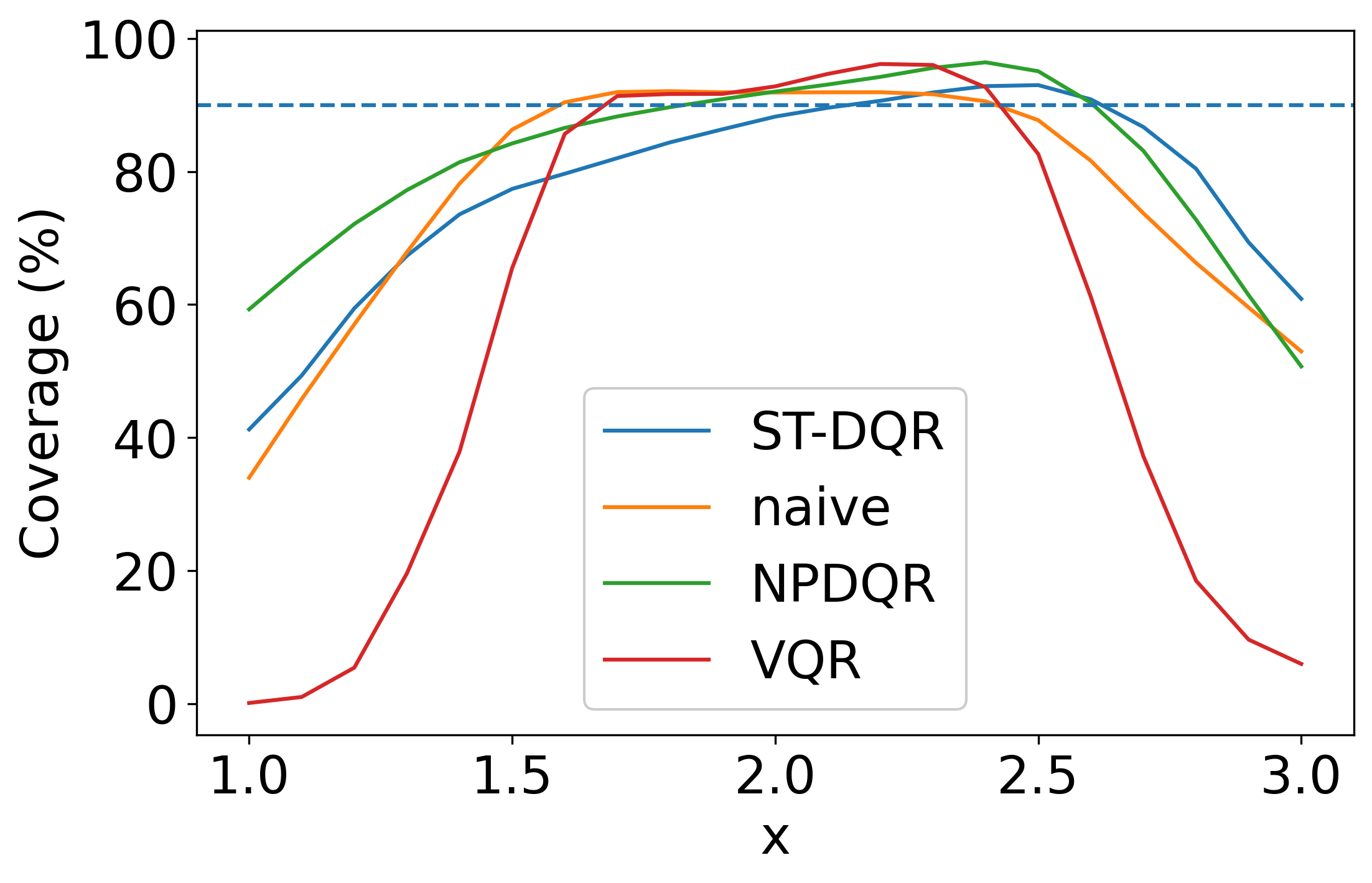}}}

    \captionsetup{format=hang} \caption{Conditional coverage achieved over the non-linear synthetic data with 10000 samples.}
    
\label{fig:syn_cov_vs_x}%
\end{figure}%

\renewcommand\imagewidth{0.5}

\begin{figure}[htbp]
\setstretch{1.1}
  \centering
  
    \scalebox{1.}{
    \begin{tabular}{c}
    \\ 

    {\centering{\rowincludegraphics[width=\imagewidth\linewidth]{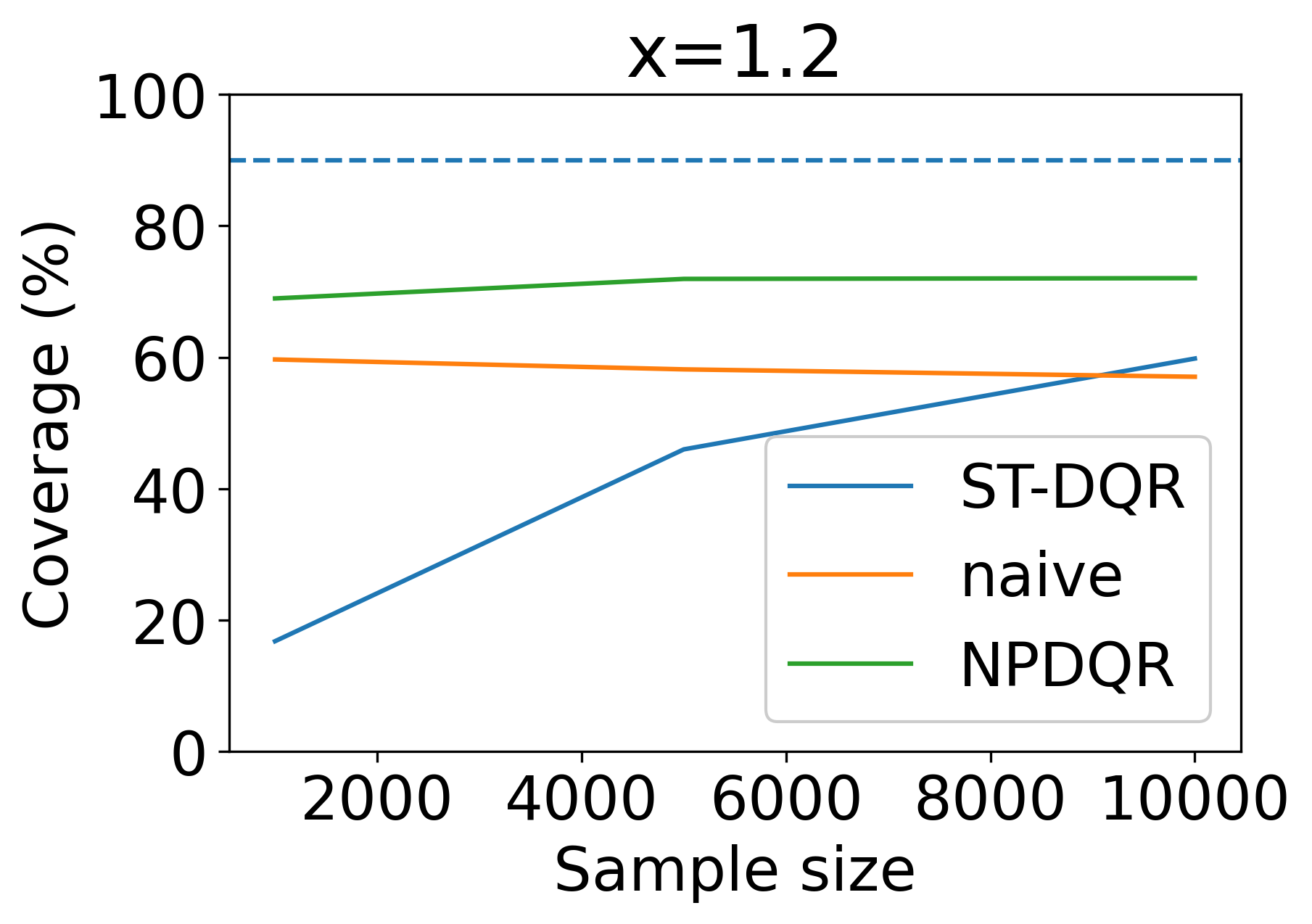}}} \\ {\centering{\rowincludegraphics[width=\imagewidth\linewidth]{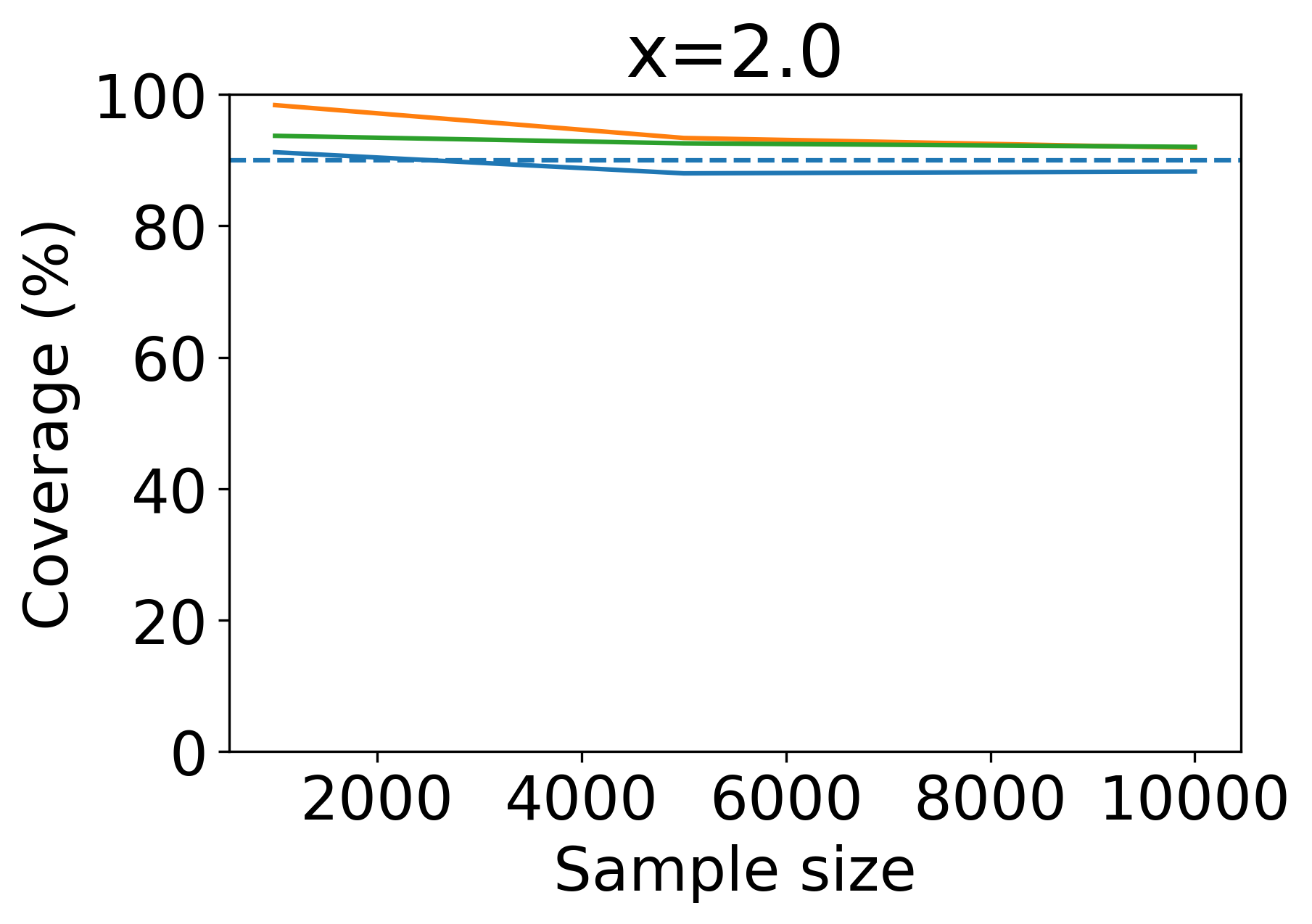}}} \\ {\centering{\rowincludegraphics[width=\imagewidth\linewidth]{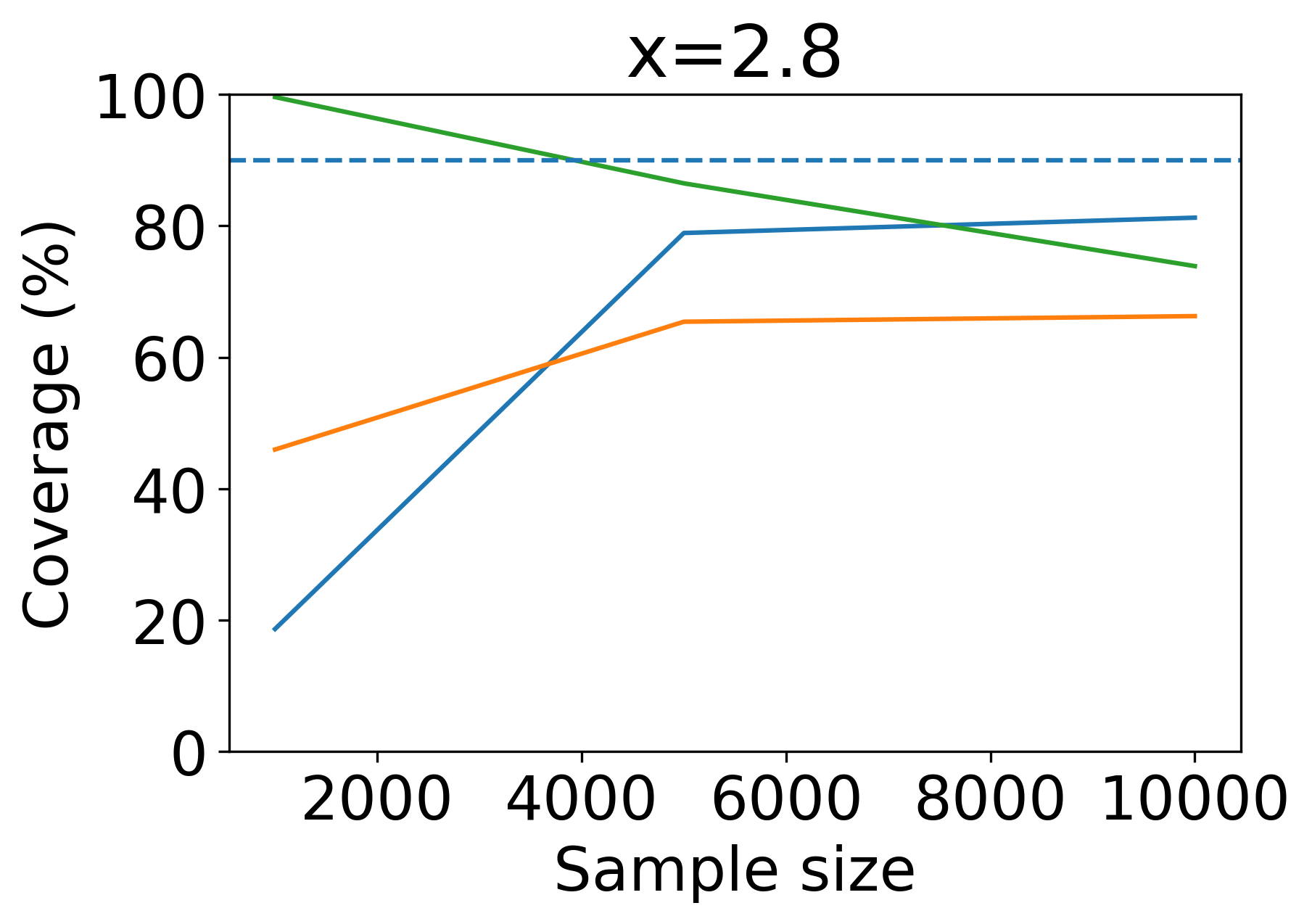}}}
    \end{tabular}%
    }
    \captionsetup{format=hang} \caption{Conditional coverage achieved over the non-linear synthetic data with increasing sample size and for different feature vectors.}
    
\label{fig:syn_cov_vs_n}
\end{figure}

\renewcommand\imagewidth{0.25}
\renewcommand\texthspace{1.15cm}
\renewcommand\coveragetexthspace{1.3cm}

\begin{figure}[htbp]
\setstretch{1.1}
  \centering
  
    \scalebox{1.}{
    \begin{tabular}{cccc}
    \multicolumn{1}{c}{Data Set Size} & \multicolumn{3}{c}{\textbf{Quantile region}}\\ 
    
    {} & {\parbox{\imagewidth\linewidth} {\quad \hspace{\texthspace} $x=1.5$}}  & {\parbox{\imagewidth\linewidth} {\quad \hspace{\texthspace} $x=2$}} &  {\parbox{\imagewidth\linewidth} {\quad \hspace{\texthspace} $x=2.5$}}   \\ \\
    
    500 & {\centering{\rowincludegraphics[width=\imagewidth\linewidth]{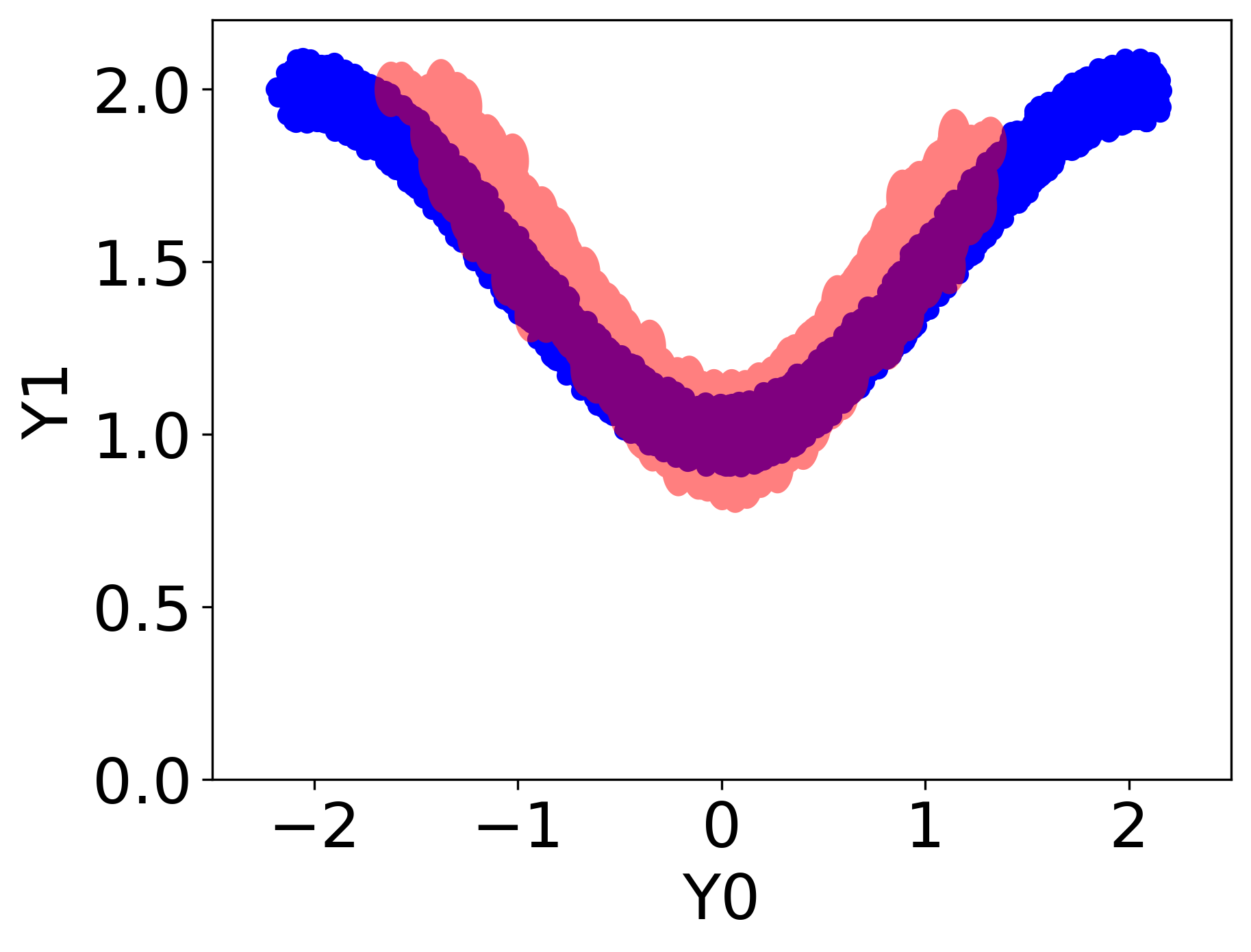}}} & {\centering{\rowincludegraphics[width=\imagewidth\linewidth]{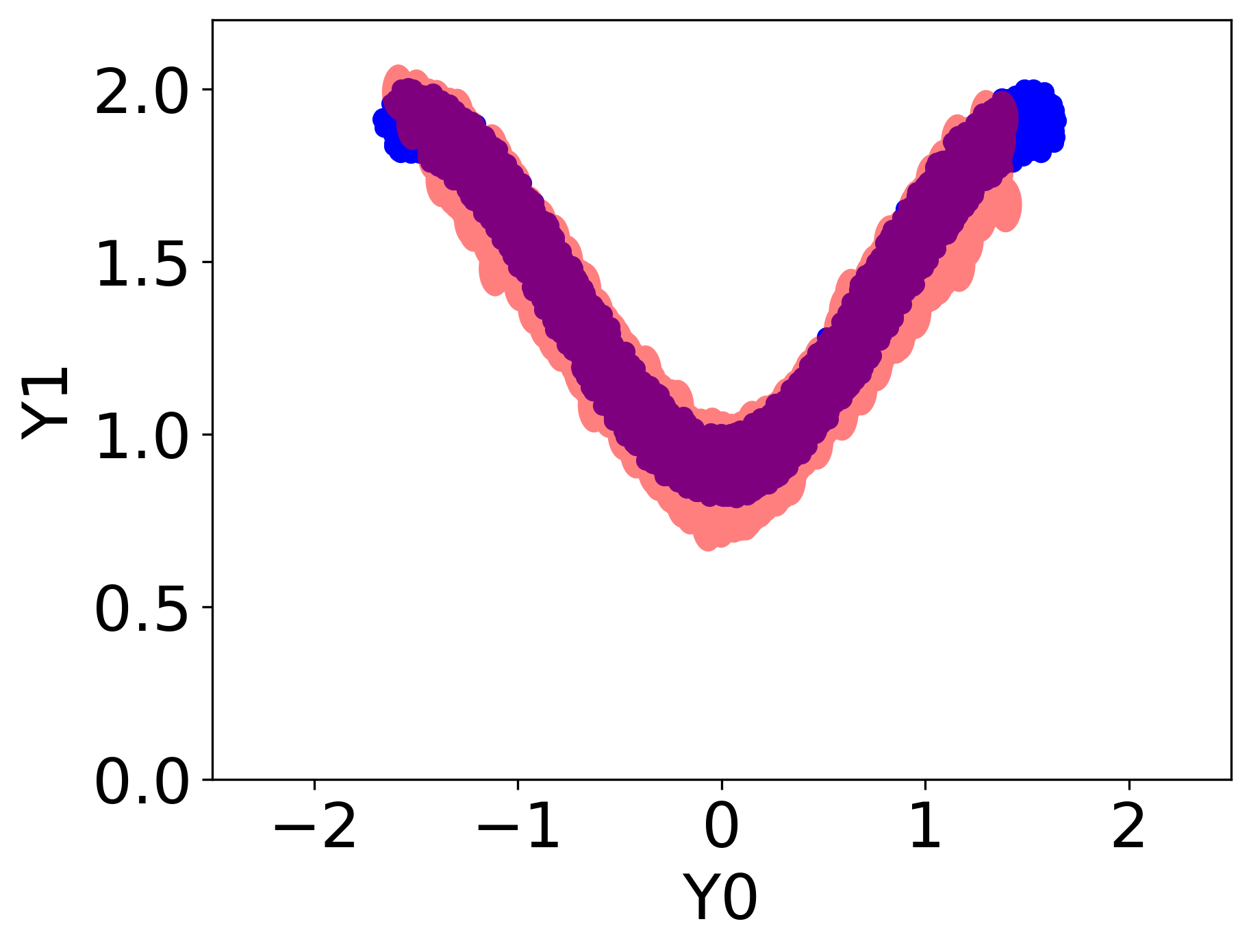}}} & {\centering{\rowincludegraphics[width=\imagewidth\linewidth]{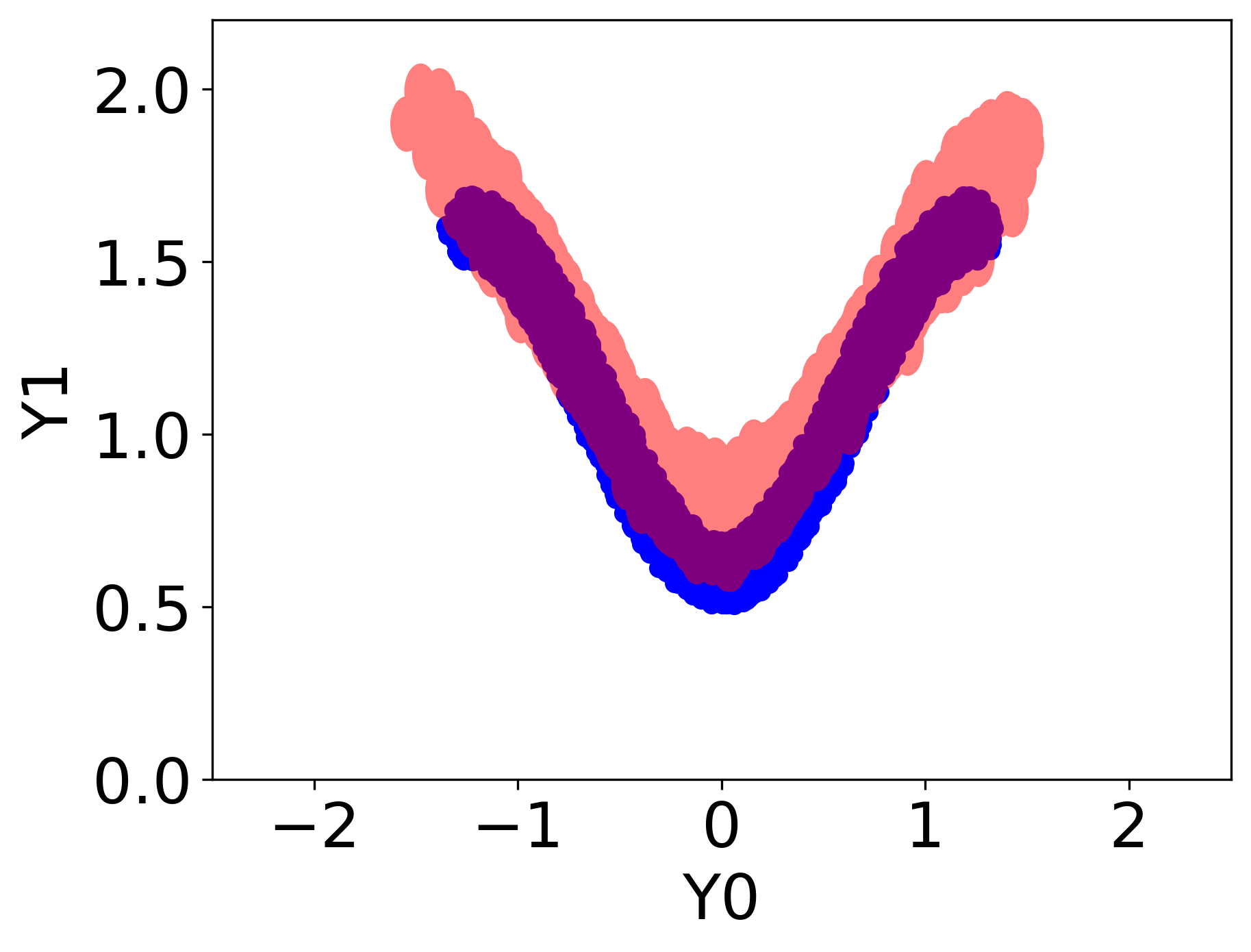}}} \\
    
    Coverage &{\parbox{\imagewidth\linewidth} {\quad \hspace{\coveragetexthspace} 61.5\% }}& {\parbox{\imagewidth\linewidth} {\quad \hspace{\coveragetexthspace} 92\% }} & {\parbox{\imagewidth\linewidth} {\quad \hspace{\coveragetexthspace} 91.58\% }} \\
    \\
    10000 & {\centering{\rowincludegraphics[width=\imagewidth\linewidth]{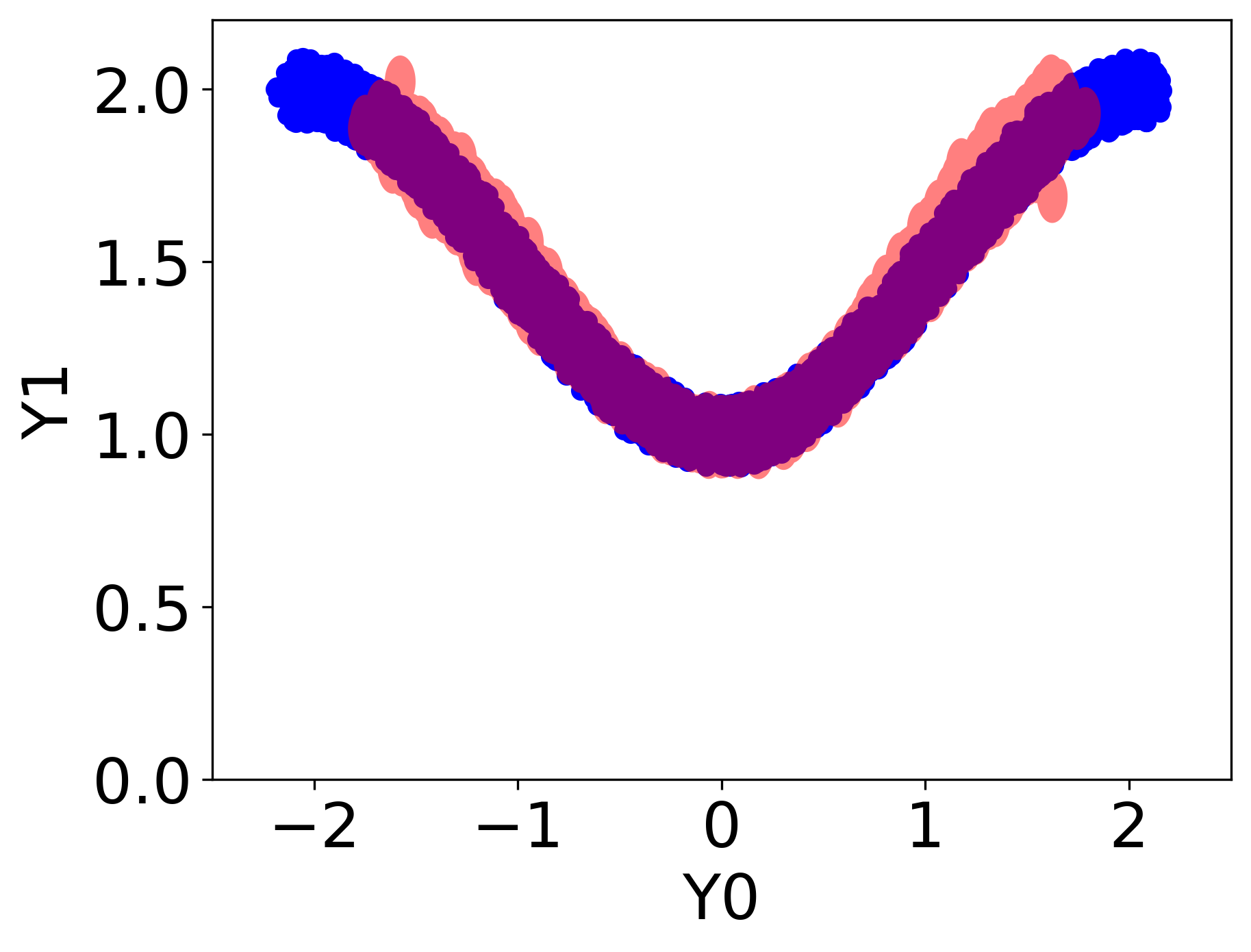}}} &
    {\centering{\rowincludegraphics[width=\imagewidth\linewidth]{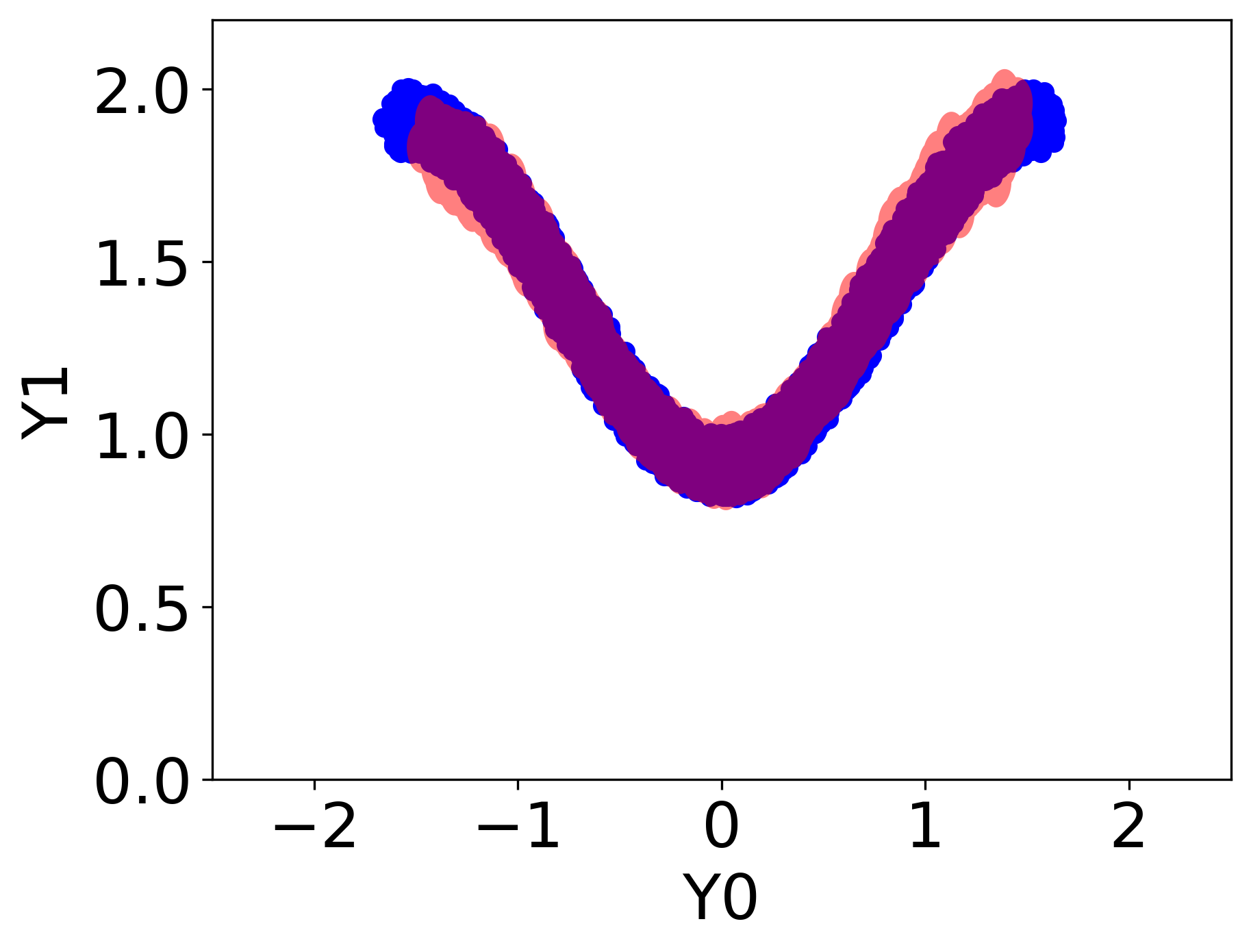}}} & 
    {\centering{\rowincludegraphics[width=\imagewidth\linewidth]{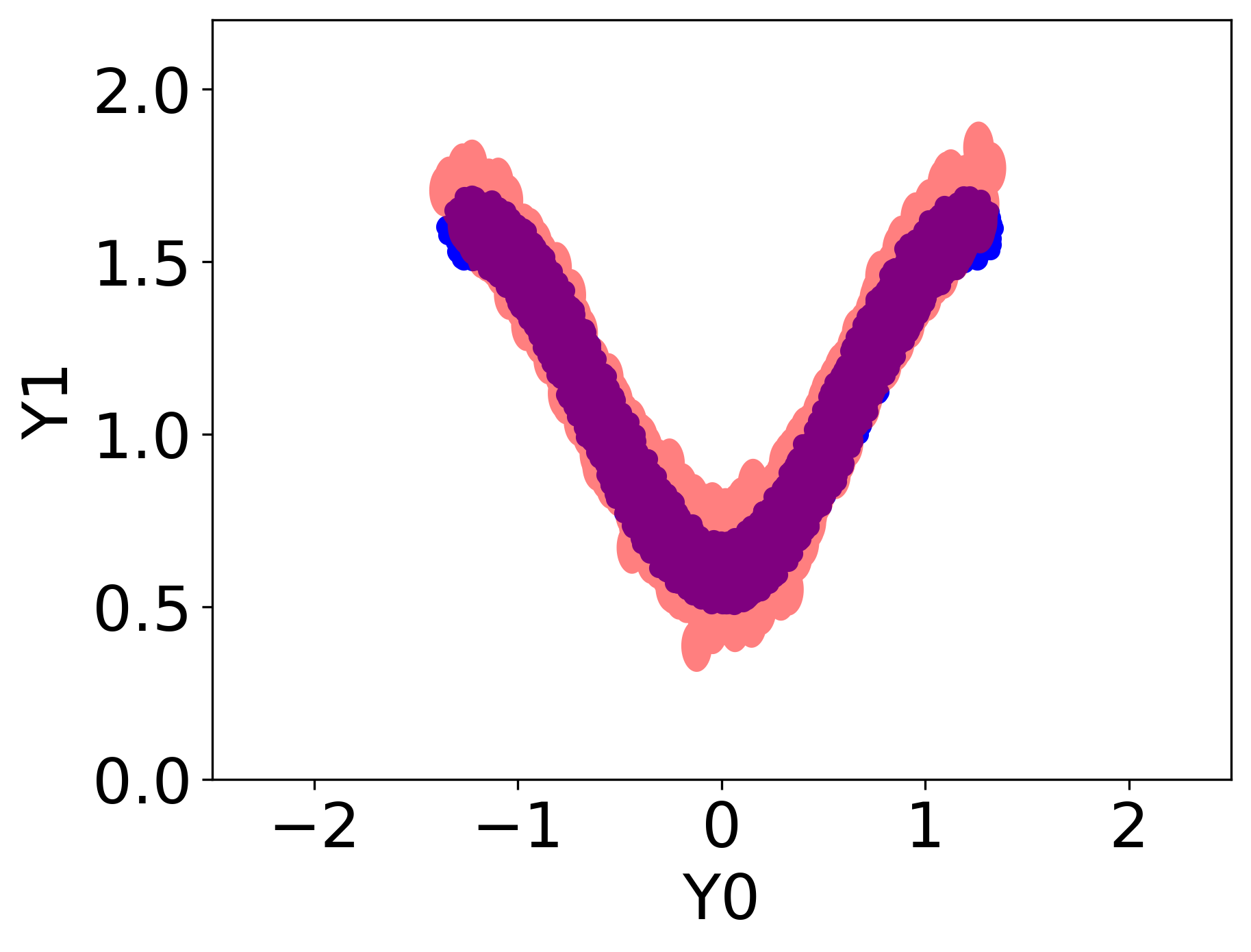}}} \\
        
    Coverage &{\parbox{\imagewidth\linewidth} {\quad \hspace{\coveragetexthspace} 84.12\% }}& {\parbox{\imagewidth\linewidth} {\quad \hspace{\coveragetexthspace} 93.28\% }} & {\parbox{\imagewidth\linewidth} {\quad \hspace{\coveragetexthspace} 99.32\% }} \\
    \\
    50000 & {\centering{\rowincludegraphics[width=\imagewidth\linewidth]{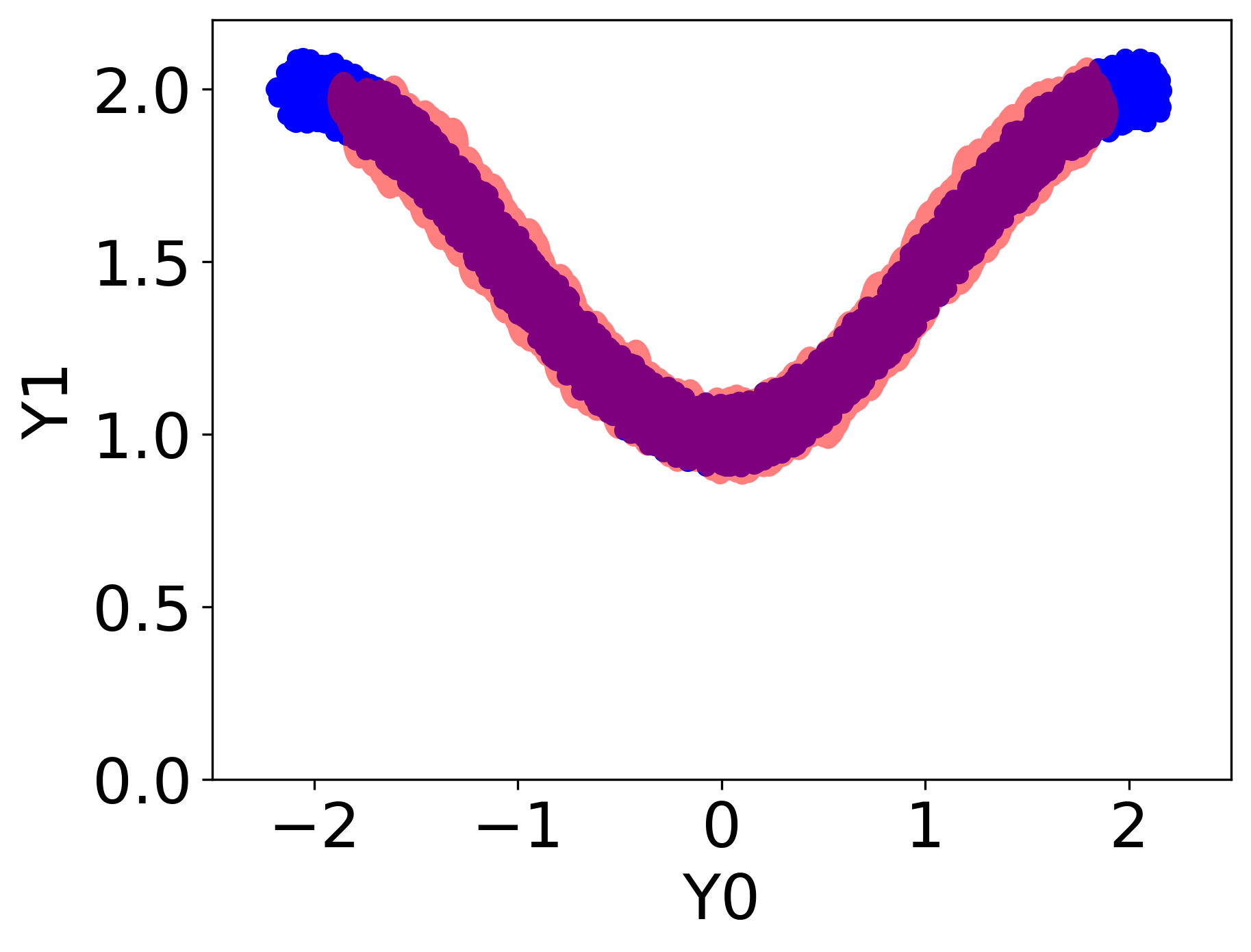}}} &
    {\centering{\rowincludegraphics[width=\imagewidth\linewidth]{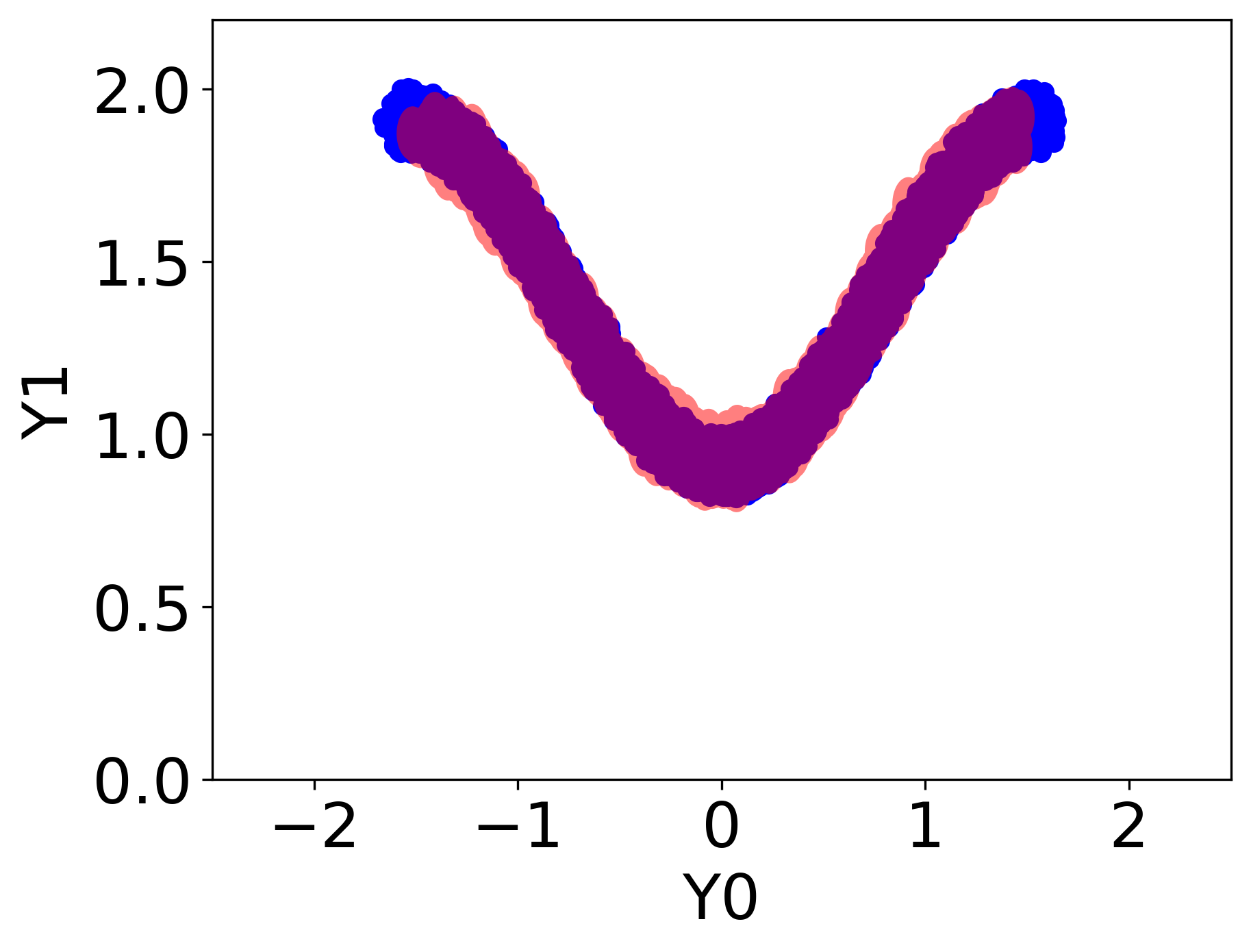}}} & 
    {\centering{\rowincludegraphics[width=\imagewidth\linewidth]{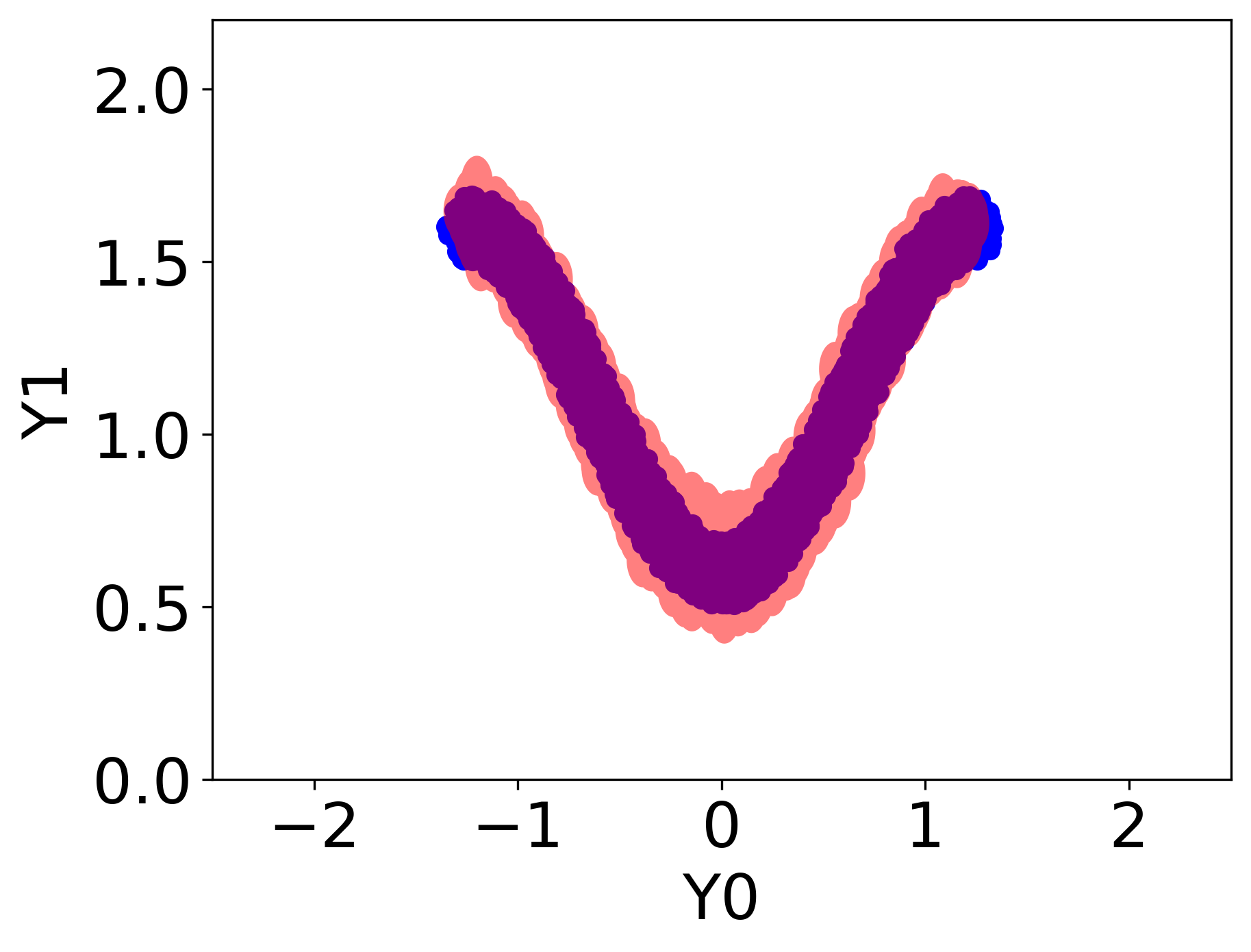}}} \\
    
    Coverage &{\parbox{\imagewidth\linewidth} {\quad \hspace{\coveragetexthspace} 88.01\% }}& {\parbox{\imagewidth\linewidth} {\quad \hspace{\coveragetexthspace} 93.42\% }} & {\parbox{\imagewidth\linewidth} {\quad \hspace{\coveragetexthspace} 98.95\% }} \\
    
    \end{tabular}%
    }
    \captionsetup{format=hang} \caption{Calibrated quantile regions constructed by \texttt{ST-DQR} for the linear synthetic data set with $p=10$ for different data set sizes.}
    
\label{fig:syn_changing_n_data_results}
\end{figure}

\renewcommand\imagewidth{0.25}
\renewcommand\texthspace{1.15cm}
\renewcommand\coveragetexthspace{1.3cm}

\begin{figure}[htbp]
\setstretch{1.1}
  \centering
  
    \scalebox{1.}{
    \begin{tabular}{cccc}
    \multicolumn{1}{c}{Data Set Size} & \multicolumn{3}{c}{\textbf{Quantile region}}\\ 
    
    {} & {\parbox{\imagewidth\linewidth} {\quad \hspace{\texthspace} $x=1.5$}}  & {\parbox{\imagewidth\linewidth} {\quad \hspace{\texthspace} $x=2$}} &  {\parbox{\imagewidth\linewidth} {\quad \hspace{\texthspace} $x=2.5$}}   \\ \\
    
    500 & {\centering{\rowincludegraphics[width=\imagewidth\linewidth]{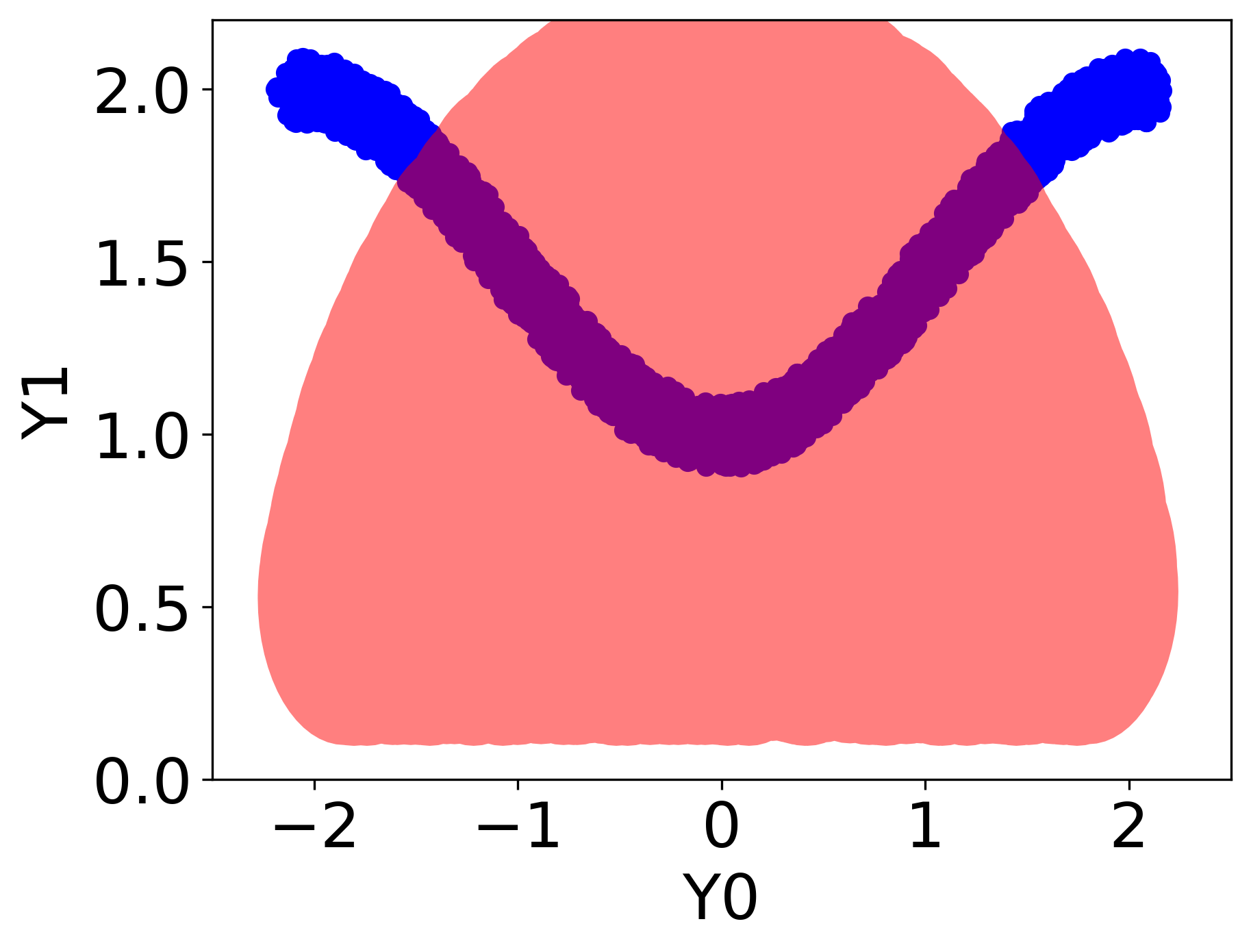}}} & {\centering{\rowincludegraphics[width=\imagewidth\linewidth]{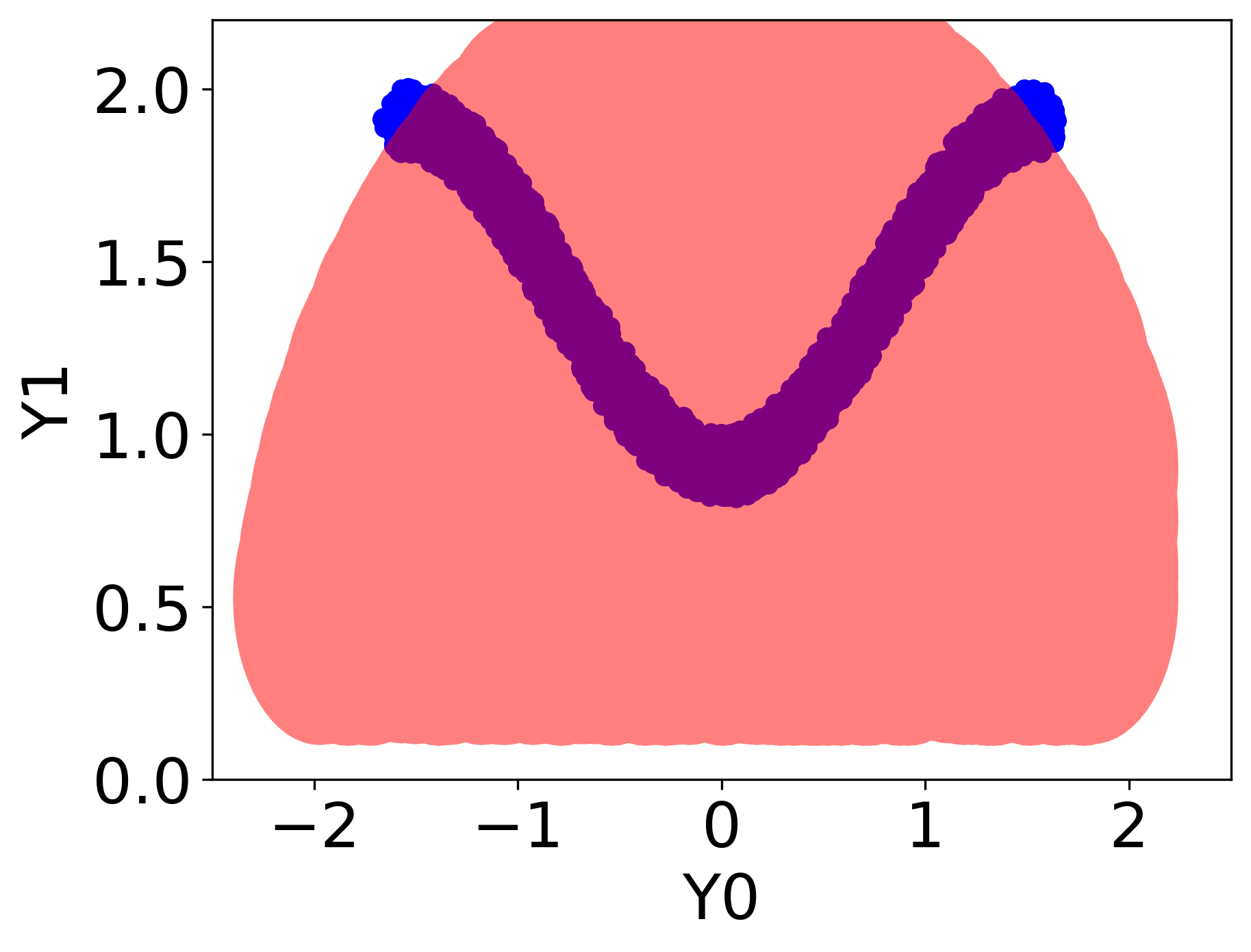}}} & {\centering{\rowincludegraphics[width=\imagewidth\linewidth]{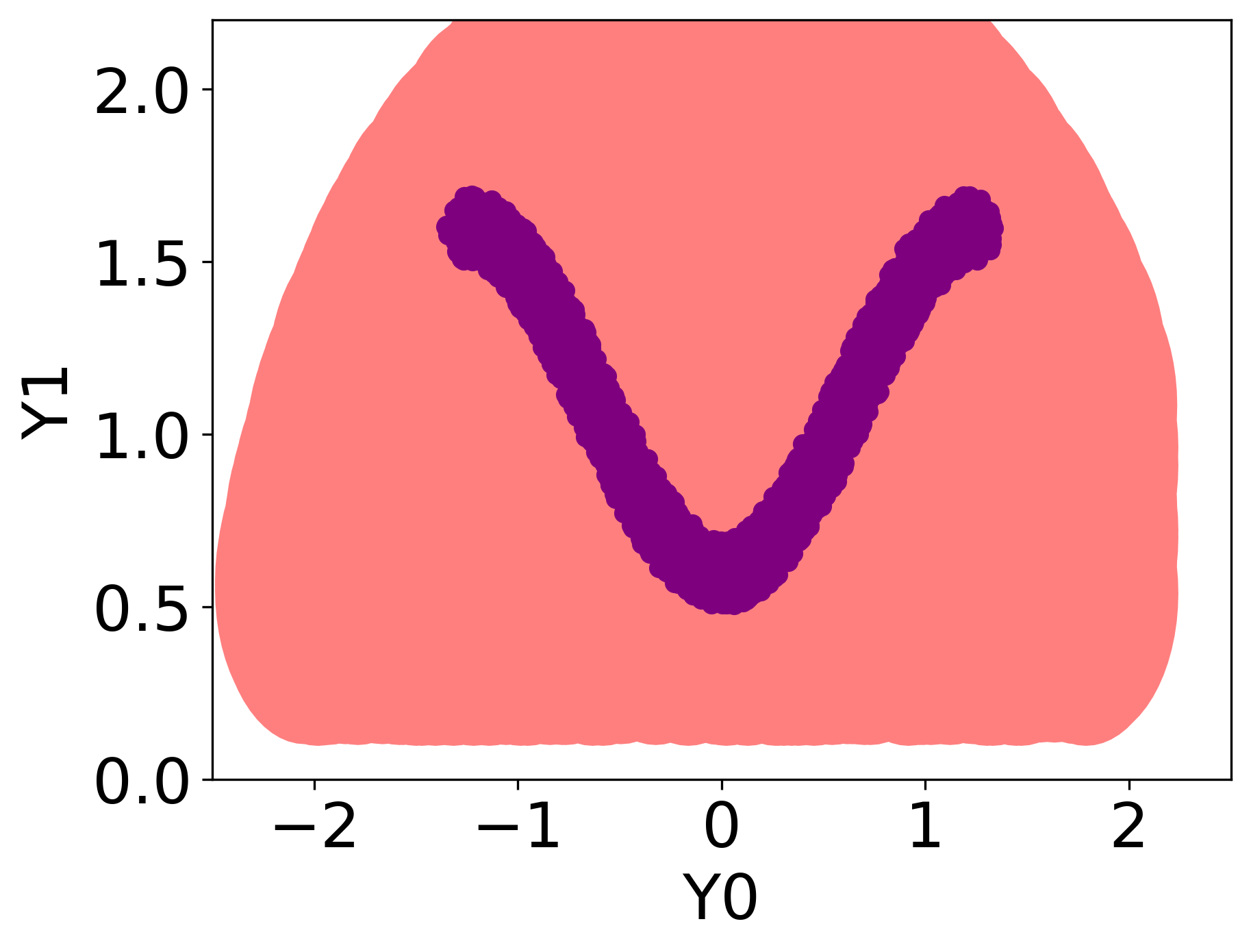}}} \\
    
    Coverage &{\parbox{\imagewidth\linewidth} {\quad \hspace{\coveragetexthspace} 71.16\% }}& {\parbox{\imagewidth\linewidth} {\quad \hspace{\coveragetexthspace} 97.63\% }} & {\parbox{\imagewidth\linewidth} {\quad \hspace{\coveragetexthspace} 100\% }} \\
    \\

    10000 & {\centering{\rowincludegraphics[width=\imagewidth\linewidth]{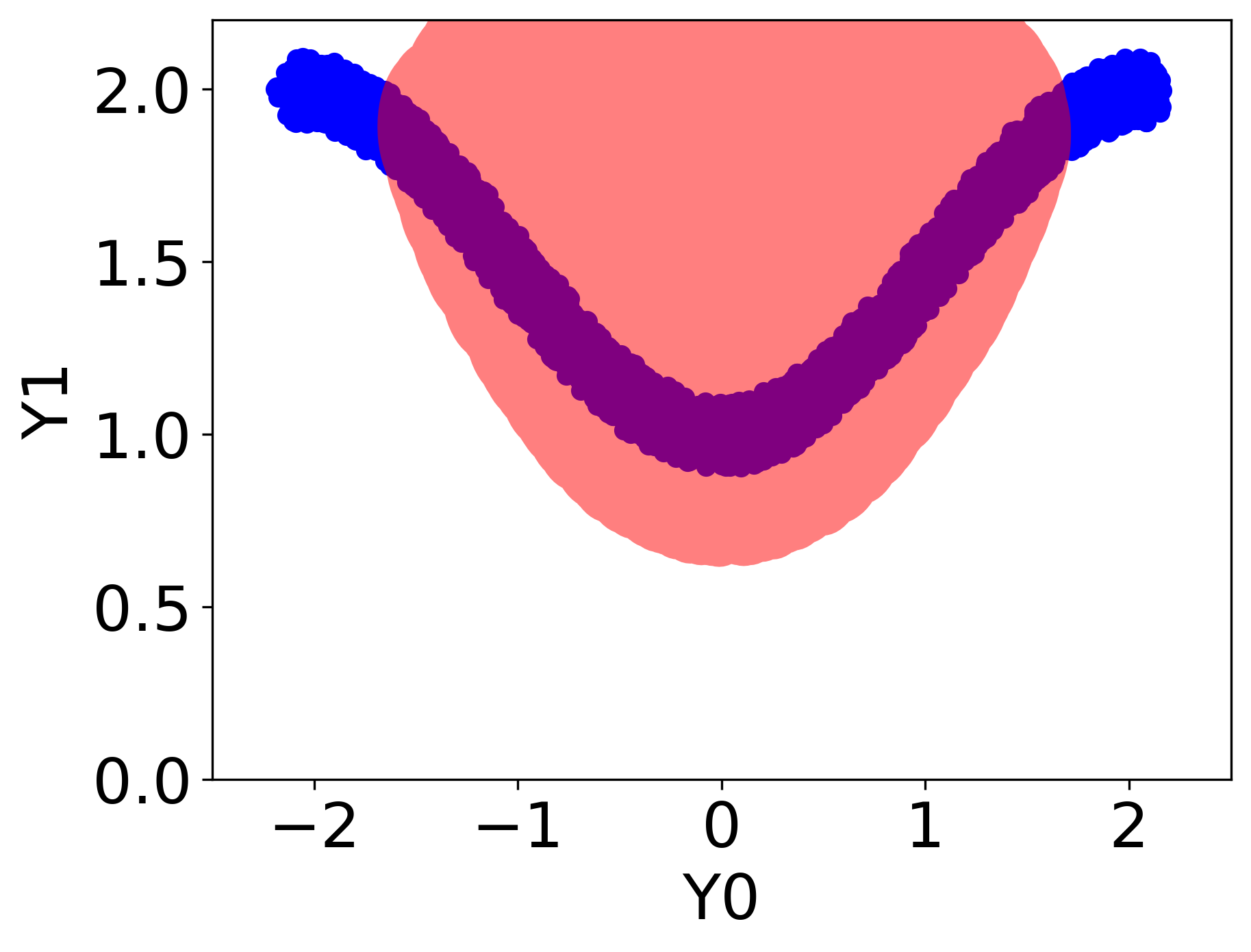}}} &
    {\centering{\rowincludegraphics[width=\imagewidth\linewidth]{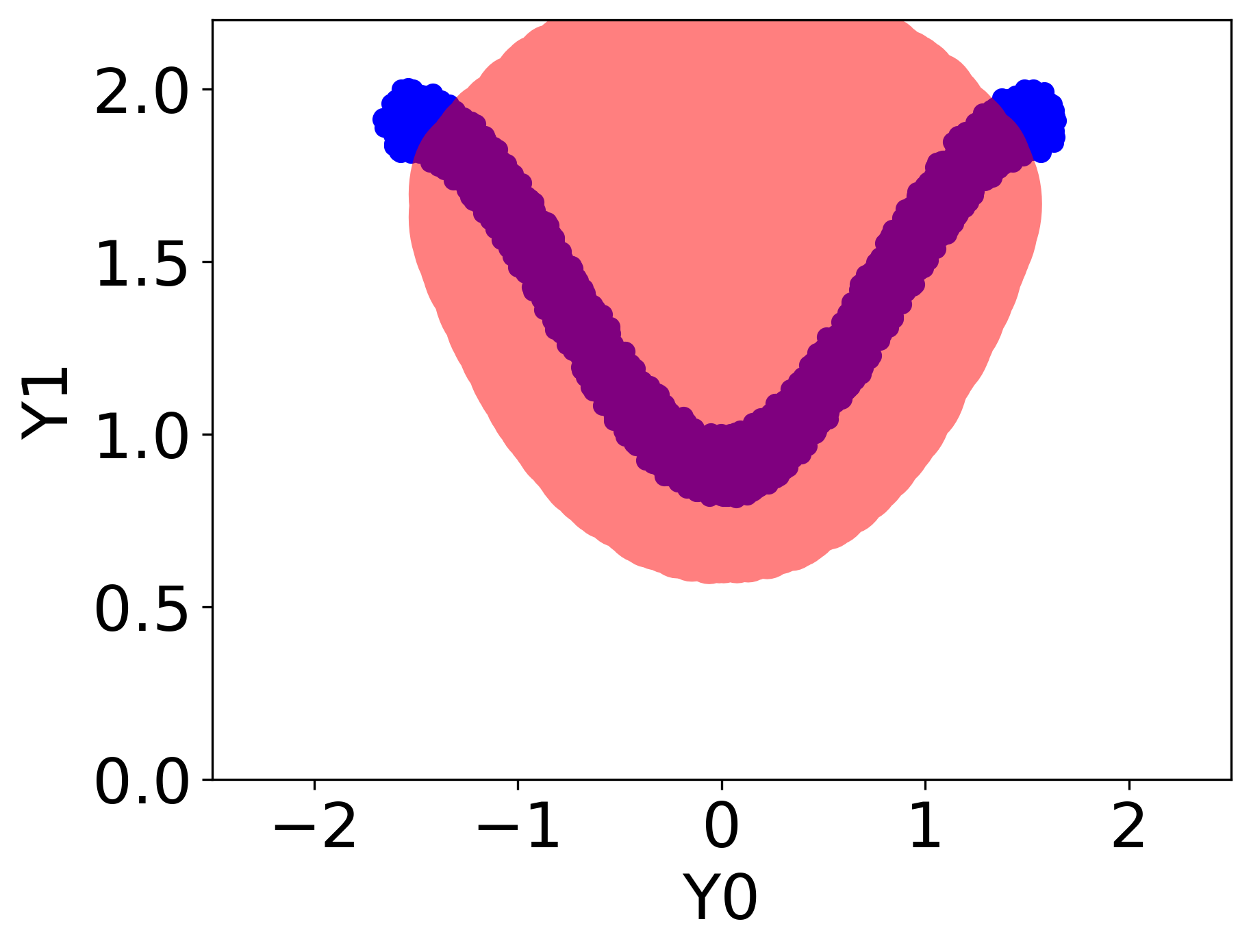}}} & 
    {\centering{\rowincludegraphics[width=\imagewidth\linewidth]{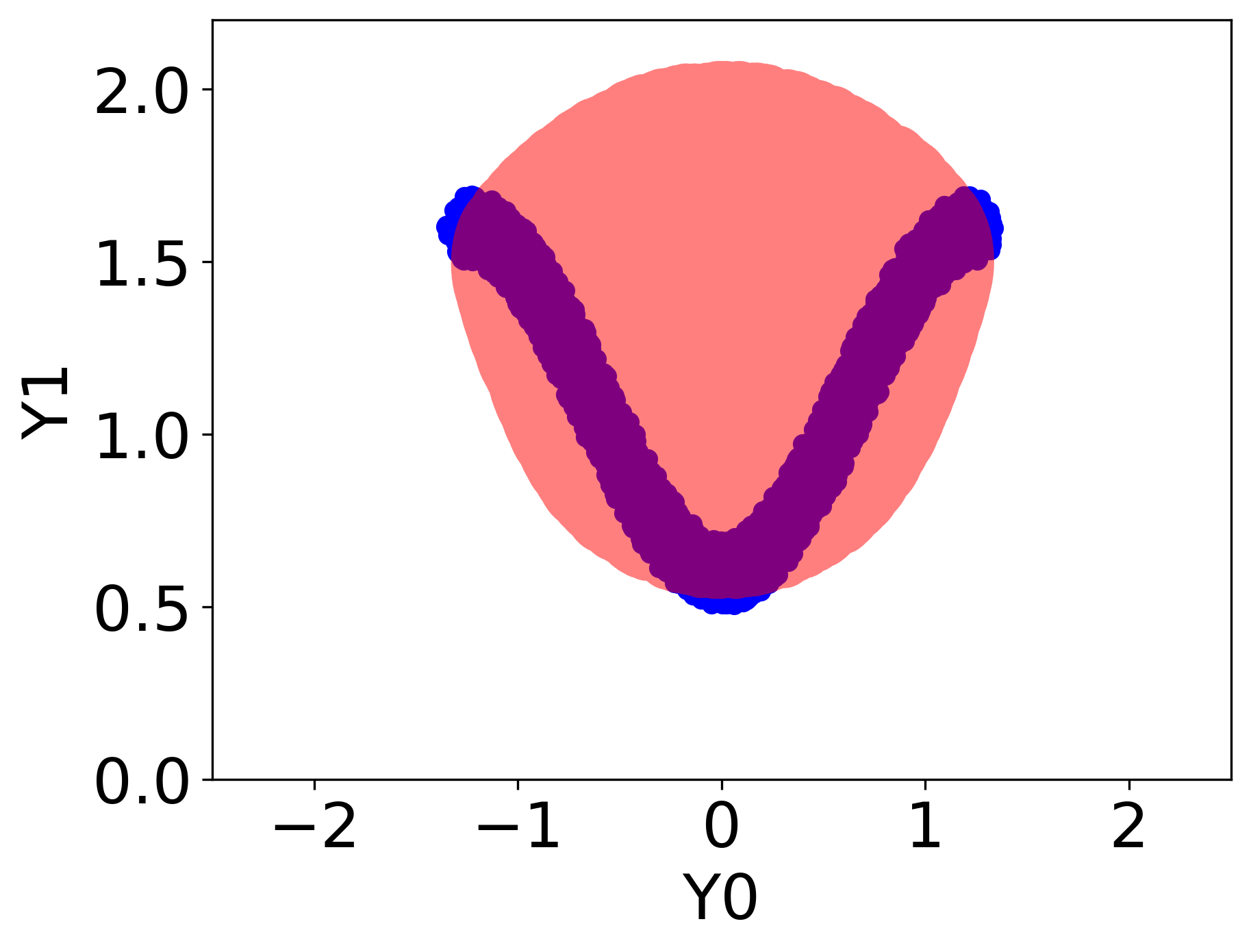}}} \\
        
        Coverage &{\parbox{\imagewidth\linewidth} {\quad \hspace{\coveragetexthspace} 79.5\% }}& {\parbox{\imagewidth\linewidth} {\quad \hspace{\coveragetexthspace} 91\% }} & {\parbox{\imagewidth\linewidth} {\quad \hspace{\coveragetexthspace} 98.2\% }} \\
    \\
    50000 & {\centering{\rowincludegraphics[width=\imagewidth\linewidth]{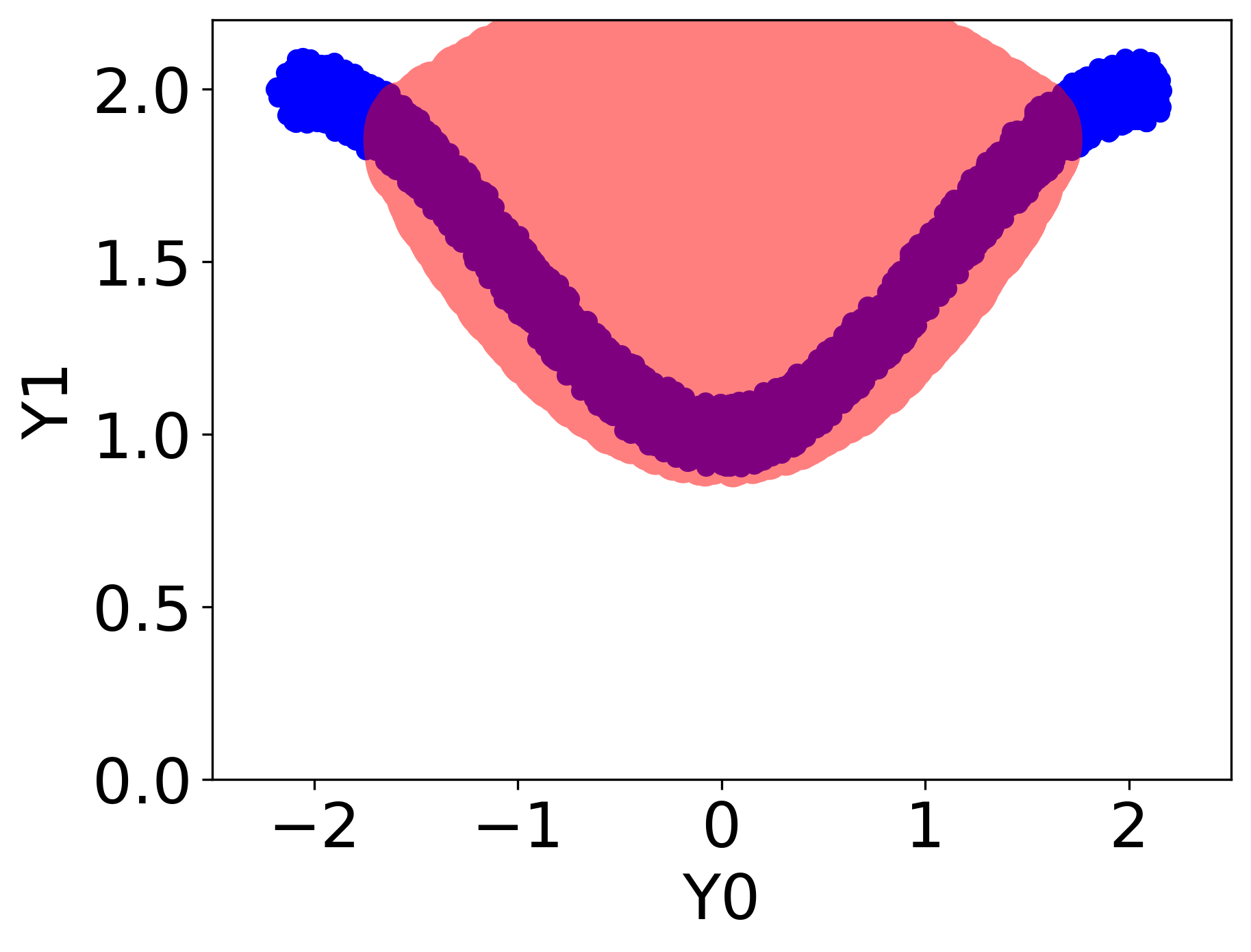}}} &
    {\centering{\rowincludegraphics[width=\imagewidth\linewidth]{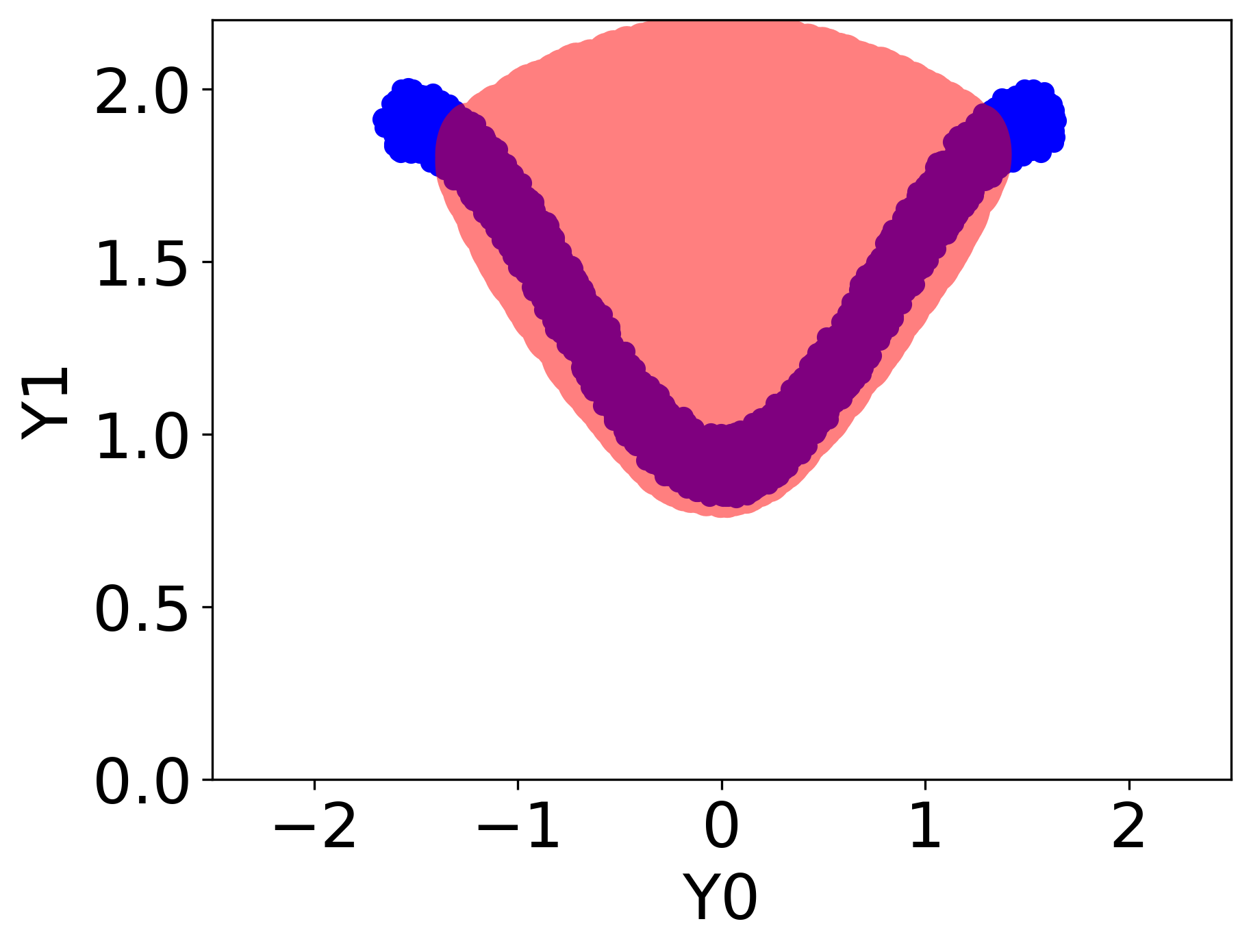}}} & 
    {\centering{\rowincludegraphics[width=\imagewidth\linewidth]{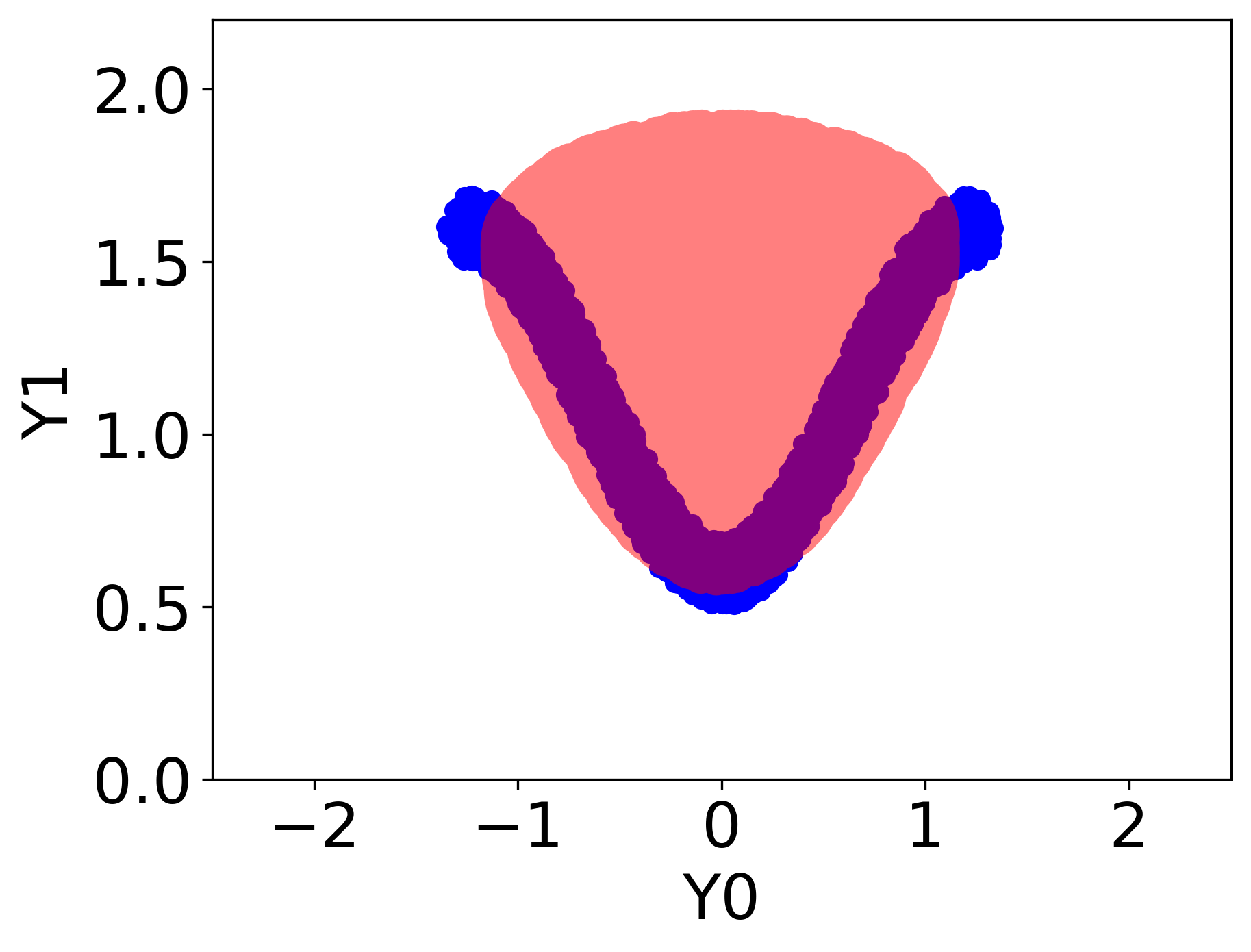}}} \\
    
    Coverage &{\parbox{\imagewidth\linewidth} {\quad \hspace{\coveragetexthspace} 81.95\% }}& {\parbox{\imagewidth\linewidth} {\quad \hspace{\coveragetexthspace} 87.15\% }} & {\parbox{\imagewidth\linewidth} {\quad \hspace{\coveragetexthspace} 91.15\% }} \\
    
    \end{tabular}%
    }
    \captionsetup{format=hang} \caption{Calibrated quantile regions constructed by \texttt{NPDQR} for the linear synthetic data set with $p=10$ for different data set sizes.}
    
\label{fig:npdqr_syn_changing_n_data_results}
\end{figure}

\subsubsection{Real Data sets}

In this section, we assess conditional coverage violation on the real data sets by
measuring the deviation of the cluster coverage from the nominal level. Formally, we define the \texttt{$\Delta$Coverage} as:
\begin{equation}
\texttt{$\Delta$Coverage} = \frac{1}{|C|} \mathlarger{\mathlarger{\sum}}_{c\in C} \left| \frac{1}{|c|} \sum_{x_i\in c} \mathbbm{1}\{ y_i \in S^{\gamma_{\text{cal}}}(x_i) \} - (1-\alpha) \right| ,
\end{equation}
where $C$ is a split of the test set split into clusters, as defined in Section~\ref{sec:real_data_results}.
Figure \ref{fig:real_data_delta_cov} displays the \texttt{$\Delta$Coverage} achieved by the techniques discussed in this paper on the real data sets introduced in Section \ref{sec:real_data_results}. This figure reveals that our \texttt{ST-DQR} attains the best \texttt{$\Delta$Coverage}, indicating for a good conditional coverage compared to existing methods.

In Figure \ref{fig:reduced_real_data_delta_cov} we report the \texttt{$\Delta$Coverage} on a different version of the real data sets, in which each feature vector is reduced to dimension 10 using PCA so that \texttt{VQR} is feasible. The table shows that even for these modified versions of the data sets, our method achieves better conditional coverage compared to \texttt{VQR}.

\renewcommand\imagewidth{0.8}
\begin{figure}[htbp]
\setstretch{1.1}
  \centering
  
    \scalebox{1.}{
    \rowincludegraphics[width=\imagewidth\linewidth]{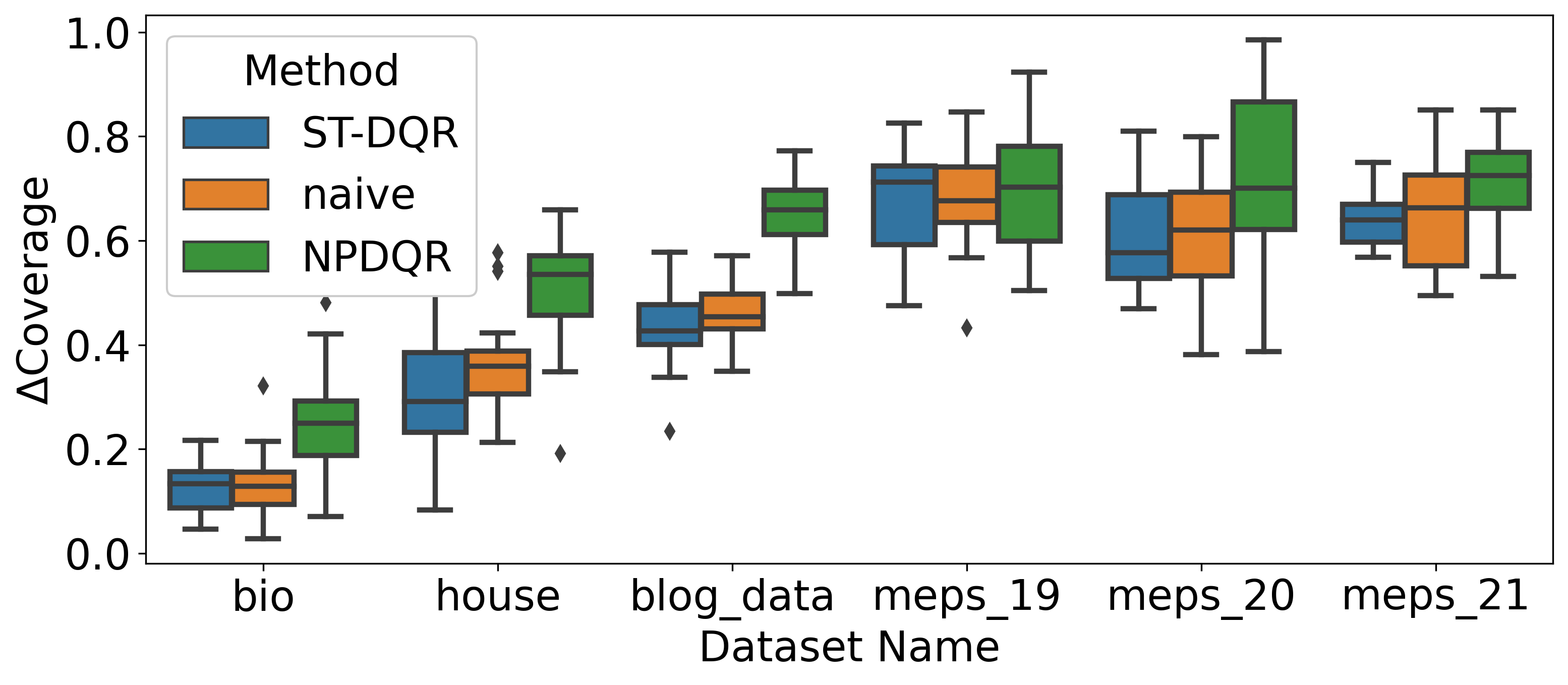}
    }
    \captionsetup{format=hang} \caption{\texttt{$\Delta$Coverage} achieved on the real data sets.}
    
\label{fig:real_data_delta_cov}
\end{figure}

\begin{figure}[htbp]
\setstretch{1.1}
  \centering
  
    \scalebox{1.}{
    \rowincludegraphics[width=\imagewidth\linewidth]{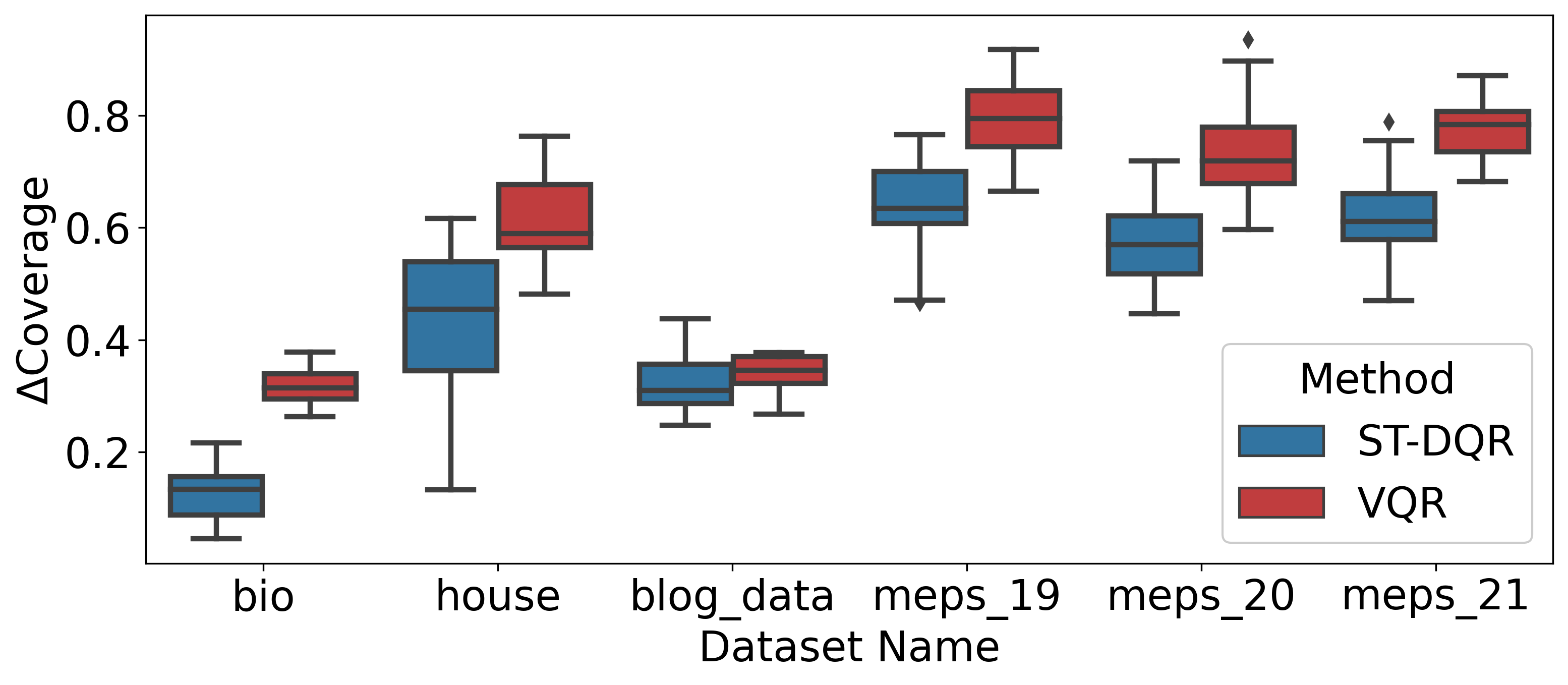}
    }
    \captionsetup{format=hang} \caption{\texttt{$\Delta$Coverage} achieved on the reduced version of the real data sets where all feature vectors were reduced to dimension 10 using PCA.}
    
\label{fig:reduced_real_data_delta_cov}
\end{figure}

\subsection{Standard Errors}\label{sec:std}
In this section, we report the standard errors of the metrics reported in the experiments section.

\subsubsection{Synthetic Data}\label{sec:syn_std}

We report the coverage rate along with its standard error, and the area along with its standard error in Tables \ref{tab:syn_data_coverage_std}, \ref{tab:syn_data_area_std} respectively.

\begin{table}[htbp]

\setstretch{1.5}
  \centering
\scalebox{0.8}{
\centering
    \begin{tabular}{ccc|ccccc}

    \toprule[1.1pt]
    
    \textbf{Setting} & $\boldsymbol{d}$ & $\boldsymbol{p}$ & \textbf{\texttt{ST-DQR}} &  \textbf{\texttt{Na\"ive QR}} & \textbf{\texttt{NPDQR}} & \textbf{\texttt{VQR}} \\

    \midrule
linear    &          2 &          1 &       89.943 (.126) &  90.059 (.147) &  90.041 (.162) &  89.755 (.101) \\
linear    &          2 &         10 &       89.926 (.137) &  89.789 (.154) &  90.131 (.115) &  90.065 (.125) \\
linear    &          2 &         50 &         89.91 (.08) &   89.99 (.058) &    89.96 (.07) &               - \\
linear    &          2 &        100 &       89.963 (.082) &  89.993 (.059) &  90.003 (.062) &               - \\
nonlinear &          2 &          1 &       90.126 (.169) &  90.078 (.171) &  90.165 (.163) &   90.13 (.145) \\
nonlinear &          3 &          1 &       90.165 (.139) &   90.114 (.16) &  90.021 (.141) &  90.156 (.131) \\
nonlinear &          3 &         10 &       89.991 (.164) &  89.881 (.173) &  90.051 (.133) &               - \\
nonlinear &          4 &          1 &       90.031 (.125) &  90.175 (.109) &  89.955 (.121) &               - \\
nonlinear &          4 &         10 &       89.792 (.141) &  89.841 (.155) &  89.956 (.138) &               - \\

    \bottomrule[1.1pt]
    
    \end{tabular}%
}
\captionsetup{format=hang} \caption{Simulated data experiments. Coverage rate and standard error achieved with each method.}
    \label{tab:syn_data_coverage_std}

\end{table}

\begin{table}[htbp]

\setstretch{1.5}
  \centering
\scalebox{0.8}{
\centering
    \begin{tabular}{ccc|cccc}

    \toprule[1.1pt]
    
    \textbf{Setting} &  $\boldsymbol{d}$ & $\boldsymbol{p}$ & \textbf{\texttt{ST-DQR}} &  \textbf{\texttt{Na\"ive QR}} & \textbf{\texttt{NPDQR}} & \textbf{\texttt{VQR}} \\

    \midrule
linear    &          2 &          1 &      315.313 (1.539) &     1223.439 (3.374) &      1024.657 (13.404) &     465.427 (3.019) \\
linear    &          2 &         10 &       384.33 (1.805) &     1680.308 (3.408) &      1622.519 (15.904) &     485.761 (2.263) \\
linear    &          2 &         50 &       388.561 (1.97) &     1777.014 (1.816) &      1525.568 (13.966) &                   - \\
linear    &          2 &        100 &        456.75 (5.01) &     1791.202 (1.751) &      1555.475 (23.214) &                   - \\
nonlinear &          2 &          1 &       236.668 (2.11) &      797.438 (2.514) &        694.339 (8.421) &     646.065 (2.643) \\
nonlinear &          3 &          1 &     933.924 (12.592) &  22984.579 (108.375) &   19982.523 (1050.929) &  7621.711 (135.156) \\
nonlinear &          3 &         10 &    1251.648 (32.406) &   44070.685 (140.94) &   34917.304 (2498.402) &                   - \\
nonlinear &          4 &          1 &       379.863 (7.52) &  27364.108 (154.317) &   82495.522 (5832.992) &                   - \\
nonlinear &          4 &         10 &      293.998 (5.958) &  53999.244 (210.159) &  129599.146 (9656.909) &                   - \\

    \bottomrule[1.1pt]
    
    \end{tabular}%
}

\captionsetup{format=hang} \caption{Simulated data experiments. Area and standard error of the quantile regions constructed by each method.}
    \label{tab:syn_data_area_std}

\end{table}

\subsubsection{Real Data}\label{sec:real_std}

Table \ref{tab:real_data_stderr} shows the standard error of the coverage and area of the quantile regions constructed for each of the real data sets using the methods discussed in this paper. Additionally, we report the area of quantile regions constructed by our method with different values of $r$ in Table~\ref{tab:real_area_std_different_r}. Table~\ref{tab:real_reduced_data_stderr} displays the standard error of each of the metrics for the reduced version of the real data sets. 

\begin{table}[!htb]
    \centering
    \setstretch{1.4}
    \begin{minipage}{.5\linewidth}
        \captionsetup{format=hang} \caption*{Coverage rate}
        \centering
  \scalebox{0.75}{
      \centering

    \begin{tabular}{c|ccc}

    \toprule[1.1pt]
    
    \textbf{Data Set name} & \textbf{\texttt{ST-DQR}} & \textbf{\texttt{Na\"ive QR}} & \textbf{\texttt{NPDQR}}  \\

    \midrule
\textbf{bio}       &         90.0 (.093) &  90.002 (.106) &  89.892 (.102) \\
\textbf{house}     &       90.157 (.133) &  89.978 (.133) &  89.876 (.143) \\
\textbf{blog\_data} &       90.145 (.089) &  90.064 (.076) &  90.016 (.075) \\
\textbf{meps\_19}   &       90.144 (.121) &  89.892 (.108) &  89.865 (.163) \\
\textbf{meps\_20}   &       89.997 (.145) &    90.0 (.135) &  89.913 (.126) \\
\textbf{meps\_21}   &        89.899 (.14) &  89.879 (.092) &   89.676 (.15) \\
    \bottomrule[1.1pt]
    
    \end{tabular}%
    }
    \end{minipage}%
    \begin{minipage}{.5\linewidth}
      \centering
        \captionsetup{format=hang} \caption*{Area of quantile regions} 
  \scalebox{0.75}{
      \centering
    \begin{tabular}{ccc}

    \toprule[1.1pt]
    
    \textbf{\texttt{ST-DQR}} & \textbf{\texttt{Na\"ive QR}} & \textbf{\texttt{NPDQR}}  \\

    \midrule
333.057 (3.8)    &   407.304 (5.333) &   406.852 (6.955) \\
360.616 (6.789)  &   421.087 (4.062) &   414.327 (5.349) \\
138.527 (11.535) &     214.8 (4.857) &   252.249 (7.093) \\
188.503 (7.585)  &  417.525 (19.792) &  408.791 (11.748) \\
190.762 (3.866)  &  433.061 (12.122) &  421.309 (11.149) \\
196.867 (7.593)  &  431.402 (15.341) &  435.418 (10.448) \\
    \bottomrule[1.1pt]
    
    \end{tabular}%

    }
    \end{minipage} 
    \captionsetup{format=hang} \caption{Real data experiments. Coverage, area, and standard error of the quantile regions constructed by each method.}
    \label{tab:real_data_stderr}

\end{table}

\begin{table}[htbp]

\setstretch{1.5}
  \centering
\scalebox{0.8}{
\centering
\begin{tabular}{c|cccc}
    \toprule[1.1pt]
    \textbf{Data Set} & $\boldsymbol{r=1}$ & $\boldsymbol{r=2}$ & $\boldsymbol{r=3}$ & $\boldsymbol{r=4}$ \\
    \midrule

    \textbf{bio} & 317.316 (3.004) &      321.482 (3.447) &        333.057 (3.8) &      342.314 (3.941) \\
    \textbf{house} & 268.405 (2.949) &      326.061 (6.588) &      360.616 (6.789) &      359.062 (5.866) \\
    \textbf{meps\_20} & 192.848 (4.43)  &      198.335 (4.181) &      190.762 (3.866) &      190.021 (4.225) \\
    
    \bottomrule[1.1pt]
    \end{tabular}%
}
\captionsetup{format=hang} \caption{Area and its standard error of quantile regions constructed on real data sets using \texttt{ST-DQR} with different values of $r$.}
    \label{tab:real_area_std_different_r}

\end{table}

\begin{table}[!htb]
    \centering

    \setstretch{1.4}
    \begin{minipage}{.5\linewidth}
        \captionsetup{format=hang} \caption*{Coverage rate}
        \centering
  \scalebox{0.6}{
      \centering

    \begin{tabular}{c|cccc}

    \toprule[1.1pt]
    
    \textbf{Data Set} & \textbf{\texttt{ST-DQR}} & \textbf{\texttt{Na\"ive QR}} & \textbf{\texttt{NPDQR}} & \textbf{\texttt{VQR}}  \\

    \midrule
    \textbf{bio}       &         90.0 (.093) &  90.002 (.106) &  89.892 (.102) &   89.87 (.104) \\
    \textbf{house}     &       89.923 (.159) &  90.094 (.111) &  90.067 (.149) &  90.119 (.172) \\
    \textbf{blog\_data} &       90.078 (.075) &  90.035 (.072) &  89.823 (.078) &  90.034 (.076) \\
    \textbf{meps\_19}   &       90.179 (.121) &   90.16 (.105) &  89.919 (.093) &  90.149 (.128) \\
    \textbf{meps\_20}   &       90.087 (.158) &  89.925 (.139) &  90.036 (.158) &  90.134 (.179) \\
    \textbf{meps\_21}   &       90.061 (.152) &  89.957 (.149) &  89.965 (.166) &  89.887 (.111) \\
    \bottomrule[1.1pt]
    
    \end{tabular}%
    }
    \end{minipage}%
    \begin{minipage}{.5\linewidth}
      \centering
        \captionsetup{format=hang} \caption*{Area of quantile regions}
  \scalebox{0.6}{
      \centering
    \begin{tabular}{cccc}

    \toprule[1.1pt]
    
    \textbf{\texttt{ST-DQR}} & \textbf{\texttt{Na\"ive QR}} & \textbf{\texttt{NPDQR}}& \textbf{\texttt{VQR}}  \\

    \midrule
    333.057 (3.8)   &   407.304 (5.333) &   406.852 (6.955) &   530.053 (5.189) \\
    384.985 (6.94)  &   577.808 (5.167) &   548.684 (6.955) &   523.282 (4.645) \\
    208.05 (3.63)   &   362.808 (6.332) &   405.793 (8.749) &   344.921 (4.684) \\
    249.926 (9.727) &  704.739 (27.762) &  381.729 (11.357) &  269.517 (11.503) \\
    240.692 (5.099) &  664.476 (19.175) &   358.091 (7.766) &   264.666 (6.205) \\
    249.156 (8.992) &   709.56 (24.683) &  383.216 (13.327) &   264.41 (10.889) \\
    \bottomrule[1.1pt]
    
    \end{tabular}%

    }
    \end{minipage} 
        \captionsetup{format=hang} \caption{Reduced real data experiments. Coverage, area, and standard error of the quantile regions constructed by each method. All feature vectors were reduced to dimension 10 using PCA.}
        \label{tab:real_reduced_data_stderr}

\end{table}

\section{Technical Details}\label{sec:tech_details}
In this section, we provide technical details regarding techniques used in this work.

\subsection{The Initial Distance Threshold}\label{sec:gamma_init}
In this section, we formally define $\gamma_\textrm{init}$. Let $R_\mathcal{Y}= \{a_1, a_2,...a_m \}$, and denote the distance of $a_i$ from its neighbor by $D_i = \min_{a\in R_\mathcal{Y}(x)}{d(a_i,a)}$. The initial distance threshold $\gamma_\textrm{init}$ is defined as the 90-th smallest quantile of $\{D_i: i=1,...,m\}$.

\subsection{Grid Size}\label{sec:grid_size}
Recall that our method maps a discretized quantile region $R_\mathcal{Z}(x)$ to space $\mathcal{Y}$. We also discretize the space $\mathcal{Y}$ to compute the area of a quantile region. In this section, we describe how we make these discretizations in practice. For both tasks, we define a grid in space $\mathcal{Y}$ or space $\mathcal{Z}$, and set the grid's boundaries in each dimension to be the 99\% and 1\% empirical quantiles of the training set's response variables, respectively. We also increase the upper and lower boundaries by 1 for a grid used to discretize a quantile region, i.e., $R_\mathcal{Y}(x)$ or $R_\mathcal{Z}(x)$, and by 0.2 for a grid used to measure the area of a region in $\mathcal{Y}$. Table \ref{tab:cells_in_grid} shows the number of points in the grid, depending on its purpose.
In all cases, the grid is equally spaced, having the same number of points in each dimension. 

\begin{table}[!htb]
    \centering

    \setstretch{1.4}
  \scalebox{0.9}{

    \begin{tabular}{ccc}

    \toprule[1.1pt]
    
    \textbf{Space dimension} & \textbf{Area calculation} & $\boldsymbol{R_\mathcal{Y}(x)}$ \textbf{or} $\boldsymbol{R_\mathcal{Z}(x)}$  \\

    \midrule
    2 & $ 3025 $  & $ 1e^4 $ \\ 
    3 & $ 103823 $ & $ 42875 $ \\
    4 & $ 234256 $ & $ 104976 $ \\
    \bottomrule[1.1pt]
    
    \end{tabular}%
    }
        \captionsetup{format=hang} \caption{Number of cells in a grid used for area calculation, or for discretizing the quantile regions $R_\mathcal{Y}(x)$ and $R_\mathcal{Z}(x)$.}
            \label{tab:cells_in_grid}

\end{table}
\subsection{Dealing With High-Dimensional Responses}\label{sec:high_dim_y}

In many cases, it is required to quantify the uncertainty of a high-dimensional response that has more than three dimensions. In this section, we present several ways to interpret the quantile regions produced by our \texttt{ST-DQR} in the high-dimensional setting.

\paragraph{Dimension-wise prediction interval.} The simplest way to interpret a high-dimensional quantile region is to construct an interval for each dimension. For instance, given a quantile region $R(X)\subseteq\mathcal{Y}$, the prediction interval for the $j$-th variable of the response can be defined as: $C^j(X)=[\min (R(X))^j, \max (R(X))^j]$. The Cartesian product of these intervals forms the complement of $R(X)$ to a rectangle. Furthermore, this rectangle is more conservative as it covers more area than the region $R(X)$.
\paragraph{Querying a point coverage.} The quantile region can be used to query whether certain points fall inside it or not. By doing so, the user can check if a scenario is likely
and take a decision according to the response.
\paragraph{Visualizing 3d slices.} Assuming there are $d$ responses, one can visualize the quantile region of each triplet $(i_1, i_2, i_3)\in\{1,...,d\}^3$ while the other coordinates are held fixed at some value. This way, the user can learn the relation between each response triplet.

\subsection{Conditional Variational Auto Encoder}\label{sec:cvae_and_kl_lambda}

A \emph{conditional variational auto encoder} (CVAE), proposed by \cite{cvae}, is a model that aims to capture the implicit conditional distribution of $Y \mid X$ by reconstructing the input data from its encoded representation. Such a model is composed of an encoder $\mathcal{E}$ that encodes $Y \mid X$ into a latent variable $Z_x$, and a decoder $\mathcal{D}$ that reconstructs $(Z_x; x)$ back to the original variable $Y \mid X$. 
Formally, CVAE is a pair $(\mathcal{E}(\cdot;x), \mathcal{D}(\cdot;x))$ that ideally satisfies:
\begin{equation}
\begin{split}
& Z_x = \mathcal{E}(Y;X=x) \sim \mathcal{N}(0,1)^r,\\
& \mathcal{D}(\mathcal{E}(Y;X=x);X=x) \stackrel{d}{=} Y \mid X=x, 
\end{split}
\end{equation}
where $\mathcal{N}(0,1)$ is a standard normal distribution. In practice, we estimate $(E, \mathcal{D})$ with a neural network.
The loss used for training the CVAE model includes two components: a reconstruction loss (measuring the mean squared error) that penalizes for inaccurate reconstruction, and an additional term that encourages the encoded vector to be normally distributed. In this work, we used a loss that measures the KL-divergence \citep{kl2, kl1}, between the encoded representations and the multivariate normal distribution, multiplied by a coefficient $\lambda$. Below, we examine the effects of the KL-divergence penalty coefficient and the dimension of the encoded space on the performance of the CVAE model. To maximize performance, we suggest choosing these hyper-parameters via cross-validation.

\subsubsection{KL-divergence Loss Coefficient}

In this section, we investigate the effects of the KL-divergence loss penalty coefficient, denoted by $\lambda$, on the performance of the resulting CVAE model. Figure \ref{fig:cvae_lambda} presents the encoded and reconstructed samples, computed by CVAE models with different values of $\lambda$ on the non-linear synthetic data set, with $p=1$; see more details about the synthetic data in Section \ref{sec:syn_datasets_details}. We report the results only for the conditional distribution of $Y \mid X=2$. The figure shows that as $\lambda$ increases, the distribution of $Z_x$ is closer to standard normal, but the reconstruction becomes worse. On the other hand, as $\lambda$ decreases, the CVAE recovers better the input data, but the distribution of $Z_x$ is far from being standard normal. Relying on that figure, we set $\lambda$ to be equal to $0.01$ for all of the experiments presented in this paper.
\renewcommand\imagewidth{0.9}
\begin{figure}[htbp]
  \centering

    \includegraphics[width=\imagewidth\linewidth]{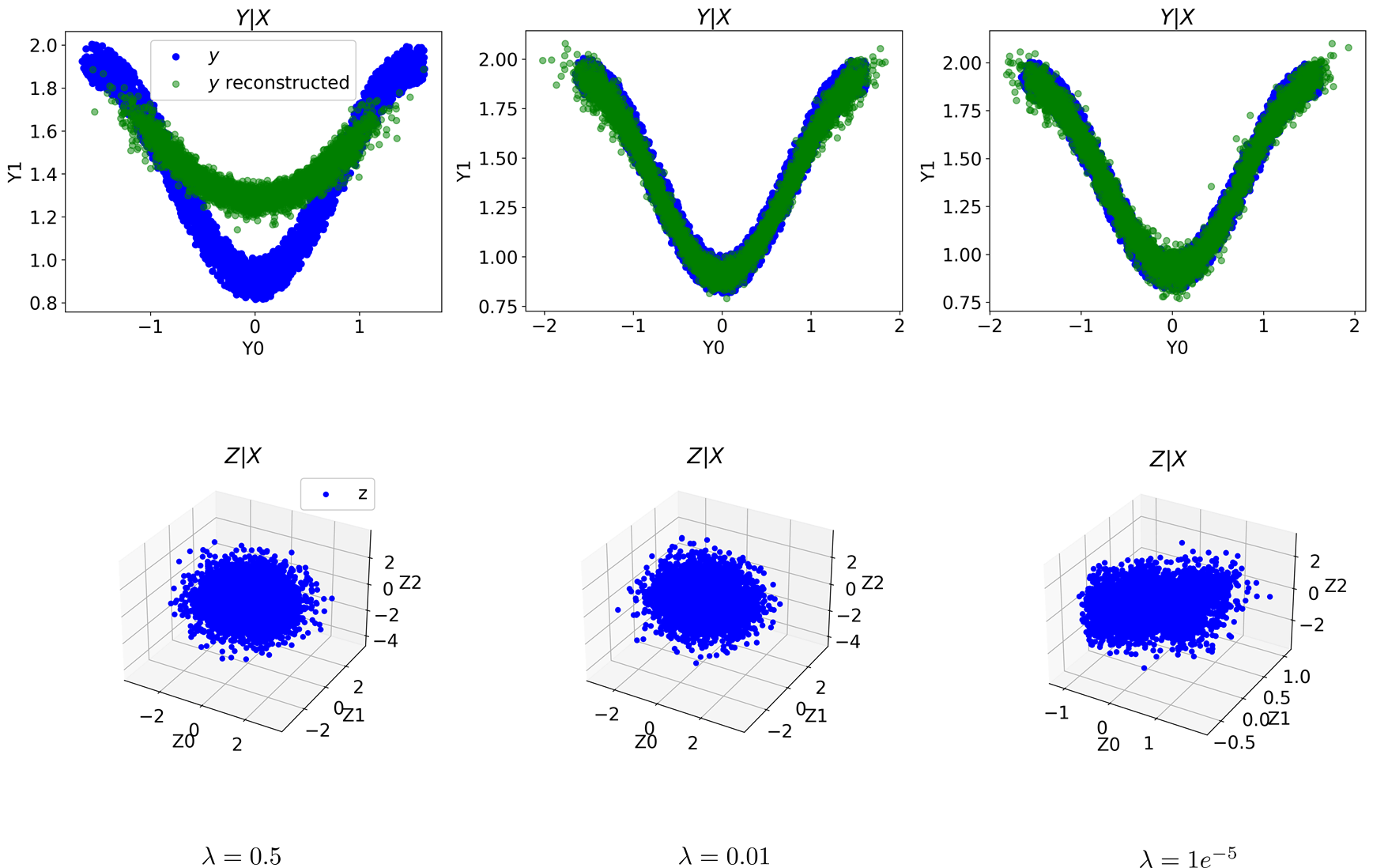}

    \captionsetup{format=hang} \caption{CVAE results for different $\lambda$s.}
    
\label{fig:cvae_lambda}%
\end{figure}%

\subsubsection{The Dimension of the Latent Space}\label{sec:vae_r}

In this section, we examine the second hyper-parameter of the CVAE, which is the dimension of the latent vector $Z$, denoted by $r$. 
Specifically, we study the effects of $r$ on the CVAE performance as well as the efficiency of the constructed quantile regions. The reconstruction errors for different values of $r$ are summarized in Table \ref{tab:real_mse_different_r}. The table shows that the CVAE reconstructs better the data when using a larger latent space dimension. This trend is also visualized in Figure \ref{fig:mse_vs_r}. This is anticipated, as the expressiveness of the CVAE depends on $r$, in the sense that as $r$ increases, the CVAE is more flexible. In practice, we recommend setting this hyper-parameter to be the dimension of the response, or choosing it using cross-validation.
\begin{table}[htbp]
\setstretch{1.5}
  \centering
\scalebox{0.8}{
\centering
\begin{tabular}{c|cccc}
    \toprule[1.1pt]
    \textbf{Data Set} & $\boldsymbol{r=1}$ & $\boldsymbol{r=2}$ & $\boldsymbol{r=3}$ & $\boldsymbol{r=4}$  \\
    \midrule
    
\textbf{bio}     &         0.034 (6e-3) &         0.021 (7e-3) &         0.021 (7e-3) &        0.021 (.001) \\
\textbf{house}   &         0.036 (6e-3) &        0.031 (.002) &        0.026 (.002) &        0.025 (.002) \\
\textbf{meps\_20} &        0.104 (.003) &         0.025 (6e-3) &         0.024 (7e-3) &         0.025 (4e-3) \\

    \bottomrule[1.1pt]
    \end{tabular}%
}
\captionsetup{format=hang} \caption{CVAE reconstruction loss and its standard error on real data sets achieved with different latent space dimensions.}
    \label{tab:real_mse_different_r}
\end{table}
\renewcommand\imagewidth{0.5}
\begin{figure}[htbp]
\setstretch{1.1}
  \centering

    {\centering{\includegraphics[width=\imagewidth\linewidth]{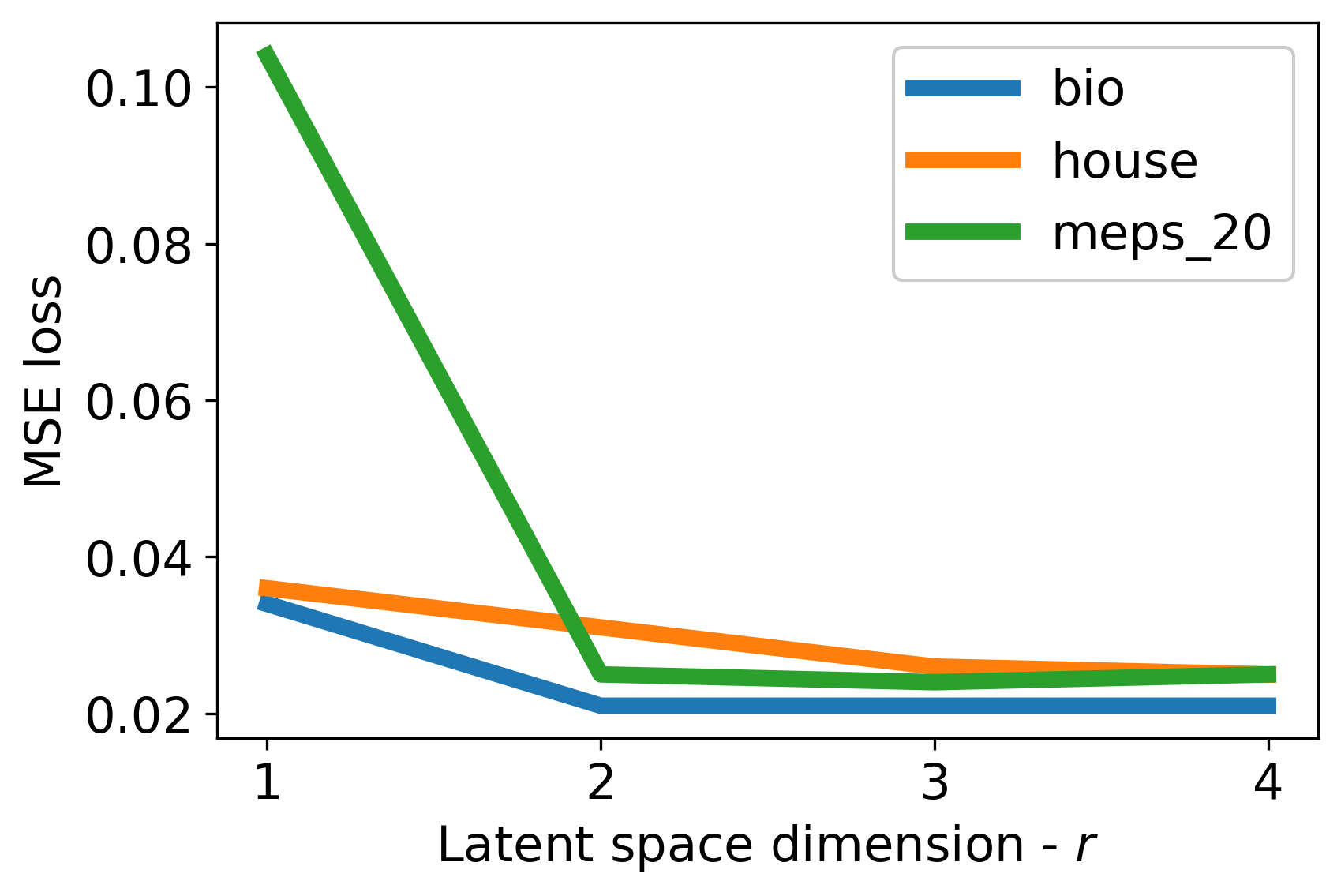}}}

    \captionsetup{format=hang} \caption{CVAE reconstruction (MSE) loss as a function of the latent space dimension for bio, house, and meps\_20 data sets.}
    
\label{fig:mse_vs_r}%
\end{figure}%

We now turn to test how $r$ affects the area of a quantile region constructed to cover $1-\alpha=90\%$ of the data. The levels of the estimated directional quantiles used for constructing the quantile region are summarized in Table \ref{tab:real_coverage_rates_info}. Table \ref{tab:real_area_different_r} displays the area of quantile regions constructed by our method with different values of $r$, scaled by the area achieved by the best method. Following that table, the best latent space dimension for bio and house data sets is $r=1$, whereas for meps\_20 data set it is $r=4$. This phenomenon can be explained in the following way. The coverage of a quantile region constructed in space $\mathcal{Z}$ depends on the dimension $r$; see Figure \ref{fig:dqr_cov}. As the dimension increases, the coverage rate of the \texttt{DQR} region decreases, and this coverage is preserved when transforming the region to space $\mathcal{Y}$; see Proposition~\ref{thm:cov_preservation}. On the other hand, as $r$ increases, the CVAE reconstructs better the data, which results in more accurate regions. Therefore, there is a trade-off between the reconstruction loss and the coverage attained in space $\mathcal{Z}$, where both are determined by $r$. Since the CVAE’s reconstruction on bio and house data sets does not significantly improve for larger values of $r$, the limitation of the \texttt{DQR} has a stronger impact on the produced region, and as a result, \texttt{ST-DQR} with $r$ set to $1$ constructs smaller quantile regions. However, meps\_20 data set requires a large value of $r$ to reconstruct well the data, as presented in Figure \ref{fig:mse_vs_r}, and therefore using $r=4$ results in the best quantile regions. To conclude, the best value of $r$ varies between the data sets, so we recommend finding it via cross-validation to maximize the performance of the method.

\begin{table}[htbp]

\setstretch{1.5}
  \centering
\scalebox{0.8}{
\centering
\begin{tabular}{c|cccc}
    \toprule[1.1pt]
    \textbf{Data Set} & $\boldsymbol{r=1}$ & $\boldsymbol{r=2}$ & $\boldsymbol{r=3}$ & $\boldsymbol{r=4}$ \\
    \midrule
    
    \textbf{bio}     & 1           &                1.013 &                 1.05 &                1.079 \\
    \textbf{house}   &1           &                1.215 &                1.344 &                1.338 \\
    \textbf{meps\_20} &  1.015       &                1.044 &                1.004 &                    1 \\
    
    \bottomrule[1.1pt]
    \end{tabular}%
}
\captionsetup{format=hang} \caption{Area of quantile regions constructed on real data sets using our method with different values of $r$, scaled by the area achieved by the best method. The standard errors are reported in Table \ref{tab:real_area_std_different_r}}
    \label{tab:real_area_different_r}

\end{table}

\subsection{Coverage Rate of a \texttt{DQR} Quantile Region}\label{sec:DQR_cov}
In this section, we compute the coverage rate of a quantile region constructed with \texttt{NPDQR} when the nominal coverage level is set to $1-\alpha \geq 0.5$. In this setting, the training samples given to \texttt{NPDQR} are $\{(X_i, Z_i) \}_{i=1}^n$, so the model is fitted in space $\mathcal{Z}$. We assume that the CVAE is ideal, i.e., $Z_i \sim \mathcal{N}(0,1)^r$, where $r \in \mathbb{N}^+$. 

For any direction $u \in \mathbb{S}^{r-1}$, the half-space defined by \texttt{NPDQR} satisfies:
\begin{equation*}
\mathbb{P}(Z^r \in H^+_u(x)) = 1 - \alpha.
\end{equation*}
Let $Z^r \sim \mathcal{N}(0,1)^r$, and $Z \sim \mathcal{N}(0,1)$. Note that
\begin{equation*}
\begin{split}
& 1-\alpha = \mathbb{P}(Z^r \in H^+_u(x)) = \mathbb{P}(u^T Z^r \geq C),
\end{split}
\end{equation*}
for $C = \Phi^{-1}(\alpha) \leq 0$.
Denote by $\chi^2_r$ a chi-squared random variable with $r$ degrees of freedom. The probability to lie inside the quantile region is:
\begin{equation*}
\begin{split}
\mathbb{P}(Z^r \in R(x)) &= \mathbb{P}(\forall u \in \mathbb{S}^{r-1}: Z^r \in H^+_u(x))\\
&=  \mathbb{P}(\forall u \in \mathbb{S}^{r-1}: u^T Z^r \geq C) \\
&=  \mathbb{P}\left(\max_{u \in \mathbb{S}^{r-1}} u^T Z^r \leq -C  \right) \\
&= \mathbb{P}\left(\frac{{Z^r}^T}{\norm{Z^r}_2} Z^r \leq -C  \right) \\
& = \mathbb{P}(\norm{Z^r}_2 \leq -C) \\ 
&= \mathbb{P}(\chi^2_r \leq C^2) \\ 
&=  \mathbb{P}(\chi^2_r \leq \Phi^{-1}(\alpha)^2 ) \\
& = F_{\chi^2_r}(\Phi^{-1}(\alpha)^2),
\end{split}
\end{equation*}
where $F_{\chi^2_r}$ is the CDF of chi-squared distribution with $r$ degrees of freedom. In Section \ref{sec:need_conf} we use this computation to report the exact coverage rate of \texttt{NPDQR} as a function of the nominal coverage level, and the dimension of the response $r$.

\subsection{Calibration - Shrinking The Quantile Region}\label{sec:cal_shrink_figs}

In this section, we demonstrate the necessity of handling the two calibration cases (Case~1 and Case~2) separately. Recall that in Case~1 $c_\textrm{init} \leq 1-\alpha$, therefore we grow the base region \eqref{eq:base_qr}, and in the complementary case, Case~2, we shrink it. If these two cases would not be treated separately, one would shrink a large region by following the technique presented in Case~1, which is mainly suited for growing small regions; see Section \ref{sec:intuition}. 
Figure \ref{fig:cal_too_small_gamma} illustrates the outcome of shrinking a large region using the procedure described in Case~1. The figure shows that the calibrated region is not continuous, and its boundaries are the same as those of the base region. That is, shrinking the base region according to Case~1 procedure does not yield the desired region. However, shrinking the base region according to the method described in Case~2 (see Section \ref{sec:intuition}) does result in a smaller, continuous quantile region, as displayed in Figure \ref{fig:cal_S_gamma}. Therefore, we suggest treating the two cases separately.
\renewcommand\imagewidth{0.85}
\begin{figure}[htbp]
\setstretch{1.1}
  \centering

    {\centering{\includegraphics[width=\imagewidth\linewidth]{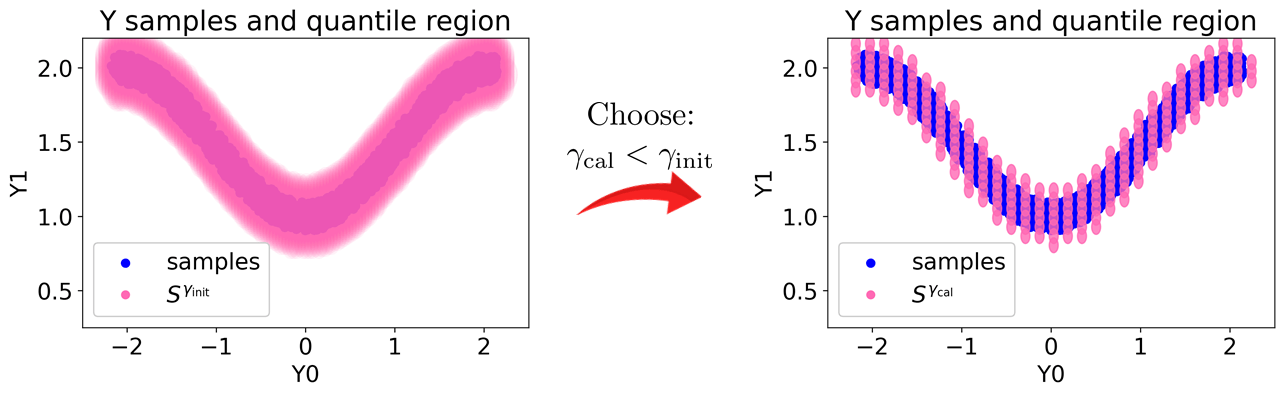}}}

    \captionsetup{format=hang} \caption{Visualization of a large quantile region (pink) calibrated according to Case~1. The resulted calibrated region is not continuous and is not a smaller version of the base region.}
    
\label{fig:cal_too_small_gamma}%
\end{figure}%
\begin{figure}[htbp]
\setstretch{1.1}
  \centering

    {\centering{\includegraphics[width=\imagewidth\linewidth]{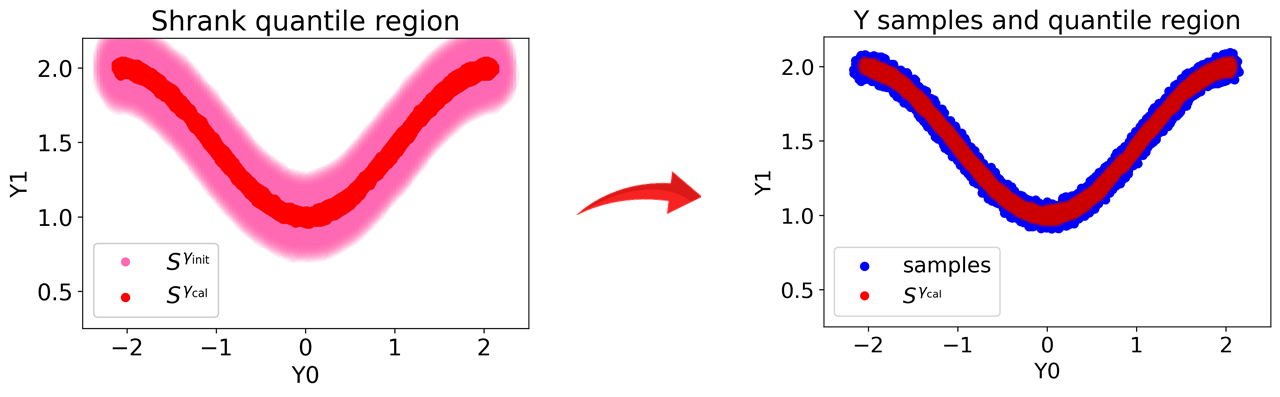}}}

    \captionsetup{format=hang} \caption{Visualization of the calibration of the large base region (pink) from Figure \ref{fig:cal_too_small_gamma}, where the calibration is done according to Case~2. The calibrated region (red) is both continuous, and smaller than the base one.}
    
\label{fig:cal_S_gamma}%
\end{figure}%

\subsection{Vector Quantile Regression}\label{sec:vqr}
Given a test point $X_{n+1}=x$, we estimate the vector quantile function $Q(\cdot ; x)$ as described by \cite{carlier2016vector}, using Gurobi solver \citep{gurobi}.\footnote{We used the implementation from https://github.com/alfredgalichon/VQR} This process results in a set of $m$ quantiles $\{Q(U_i; x)\}_{i=1}^m$, for $U_i \sim \textrm{Uniform}[0,1]^d$, where $\textrm{Uniform(a,b)}$ is a uniform distribution on the interval $(a,b)$. To construct the quantile region, we follow the procedure developed by \cite{chernozhukov2017monge}. We sample $m$ directions, denoted by $\{S_j\}_{j=1}^m$, from a spherical uniform distribution of dimension $d$. Next, we match every element in $\{S_j\}_{j=1}^m$ to an element in $\{U_i\}_{i=1}^m$ by solving the following assignment problem:
\begin{align*}
& \textrm{minimize} \sum_{i\in [m]} {\norm{U_i - f(U_i)}_2}, \\
& \textrm{subject to } \forall i\in [m]: f(U_i) \in \{S_j\}_{j=1}^m \wedge f \textrm{ is a permutation}.
\end{align*}
The discrete set of points inside the quantile region is given by:
\begin{equation*}
R_\mathcal{Y}(x) = \{Q(U_i; x) : i\in[m] \wedge \norm{f(U_i)}_2 \leq 1-\alpha\}.
\end{equation*}
\begin{table}[htbp]
\setstretch{1.5}
  \centering
\scalebox{0.8}{
\centering
    \begin{tabular}{ccc|cc}

    \toprule[1.1pt]
    
    \textbf{Setting} &  \textbf{$d$} & \textbf{$p$} & \textbf{Run time (hours)} &  \textbf{Memory consumption (GB)} \\

    \midrule
    \textbf{linear}              &   2 & 1 &  < 1 &    < 13\\
    \textbf{linear}              &   2 & 10 &  < 1 &    < 13\\
    \textbf{linear}              &   2 & 50 &  > 1 &    < 13\\

    \textbf{non-linear}          &   2 & 1 &  < 1 &    < 13\\

    \textbf{non-linear}          &   3 & 1 &  < 1 &    < 13\\
    \textbf{non-linear}          &   3 & 10 &  - &    > 13\\

    \textbf{non-linear}          &   4 & 1 &  - &    > 13\\
    \textbf{non-linear}          &   4 & 10 &  - &    > 13\\

    \bottomrule[1.1pt]
    
    \end{tabular}%
}
\captionsetup{format=hang} \caption{VQR run time and memory usage on the synthetic data.}
    \label{tab:vqr_time_memory}

\end{table}

\bibliography{bibliography}

\end{document}